\documentclass[10pt,twocolumn,letterpaper]{article}
\usepackage[pagenumbers]{cvpr}      

\usepackage{cuted} 

\usepackage{currfile} 

\usepackage{caption} 

\newif\ifwithsupp
\withsupptrue
\ifdefined\includesupp\withsupptrue\fi
\ifdefined\nosupp\withsuppfalse\fi

\usepackage{amsmath,amssymb,amsfonts,amsthm,mathtools}

\usepackage{booktabs,multirow,array,tabularx,makecell,colortbl}
\usepackage{algorithm}
\usepackage{algpseudocode}

\newcolumntype{P}[1]{>{\centering\arraybackslash}p{#1}}

\newcolumntype{L}[1]{>{\raggedright\arraybackslash}p{#1}}

\newcolumntype{C}[1]{>{\centering\arraybackslash}p{#1}}

\newcolumntype{R}[1]{>{\raggedleft\arraybackslash}p{#1}}

\newcolumntype{Y}{>{\centering\arraybackslash}X}

\definecolor{tableRankOne}{HTML}{A9D2DC}
\definecolor{tableRankTwo}{HTML}{CBE3E9}
\definecolor{tableRankThree}{HTML}{E5F1F4}
\newcommand{\bestcell}[1]{\cellcolor{tableRankOne}{\bfseries\boldmath #1}}
\newcommand{\secondcell}[1]{\cellcolor{tableRankTwo}#1}
\newcommand{\thirdcell}[1]{\cellcolor{tableRankThree}#1}
\newcommand{\ranklegend}{%
  Shading: {\setlength{\fboxsep}{1pt}%
  \colorbox{tableRankOne}{\textbf{first}}, %
  \colorbox{tableRankTwo}{second}, %
  \colorbox{tableRankThree}{third}}}
\newcommand{\tabsecondary}{\fontsize{7.5}{9}\selectfont\mdseries\color{black!78}}
\newcommand{\tabspread}[1]{{\tabsecondary $\pm$\,#1}}
\newcommand{\tabstat}[2]{#1\,\tabspread{#2}}
\newcommand{\tabstack}[2]{\makecell{#1\\\tabspread{#2}}}

\newtheorem{theorem}{Theorem}[section]

\newtheorem{corollary}[theorem]{Corollary}
\newtheorem{proposition}[theorem]{Proposition}

\theoremstyle{definition}

\theoremstyle{remark}

\definecolor{cvprblue}{rgb}{0.21,0.49,0.74}
\usepackage[pagebackref,breaklinks,colorlinks,allcolors=cvprblue]{hyperref}

\def\paperID{3444}
\def\confName{CVPR}
\def\confYear{2027}

\title{Agent-Centric Observation Adaptation for Robust Visual Control \\under Dynamic Perturbations}
\author{
Zhengru Fang$^{1}$\thanks{Equal contribution. The majority of this work was completed at City University of Hong Kong.}, Yu Guo$^{2}$\footnotemark[1], Huanhuan Lou$^{2}$, Fei Liu$^{3}$, Yuang Zhang$^{4}$\\
Yihang Tao$^{2}$, Senkang Hu$^{2}$, Wenbo Ding$^{4}$, Yuguang Fang$^{2}$\\
\small $^1$The Hong Kong University of Science and Technology\quad $^2$City University of Hong Kong\\
\small $^3$University of Zurich\quad $^4$Tsinghua University
}

\makeatletter
\patchcmd{\@maketitle}{\vspace*{12pt}}{\vspace*{0pt}}{}{}
\patchcmd{\@maketitle}{\vspace*{24pt}}{\vspace*{12pt}}{}{}
\makeatother

\begin{document}

\maketitle

\setlength{\stripsep}{6pt}
\begin{strip}
  \centering
  \includegraphics[width=\linewidth]{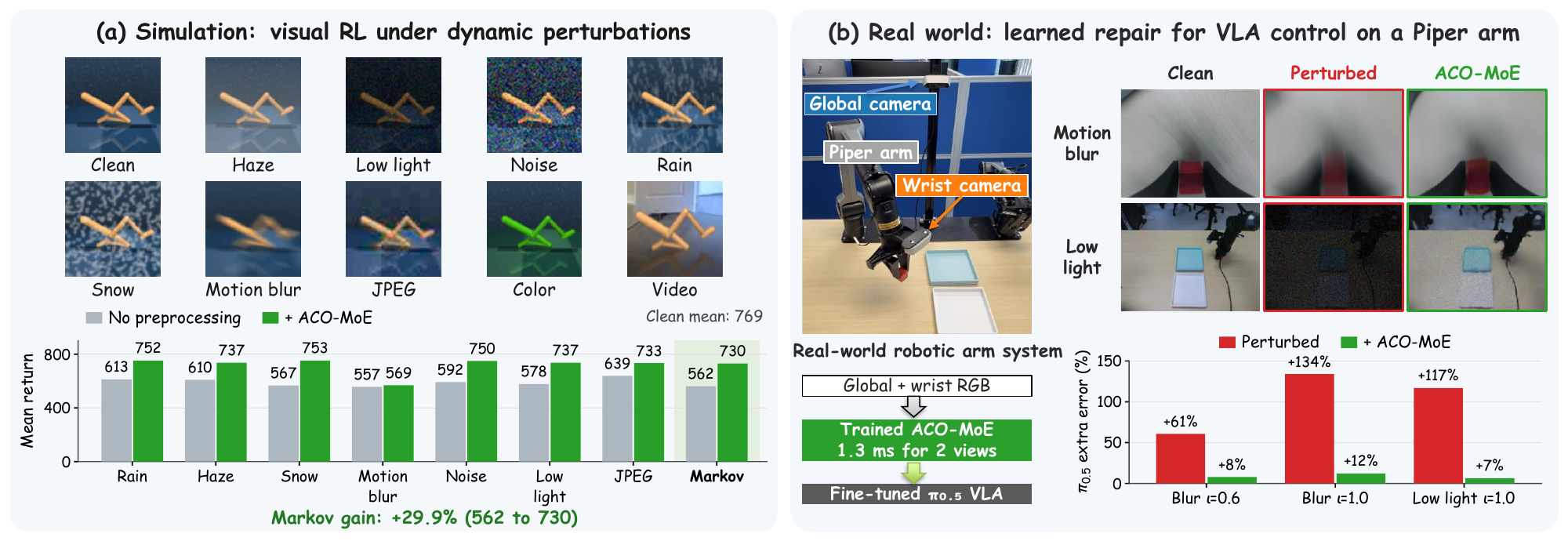}
  \captionsetup{hypcap=false}
  \captionof{figure}{\textbf{ACO-MoE restores control from simulation to real robots.} (a)~We test seven switching degradations in VDCS and color and video distractions in DMC-GB. ACO-MoE raises Markov-switching return from 562 to 730 against a clean reference of 769, recovering 81.2\% of the lost return. (b)~On an AgileX Piper arm, ACO-MoE suppresses action-error increases under synthetic degradations and raises closed-loop success from at most 10\% to between 75\% and 95\% under injected motion blur and real dim lighting with added noise.}
  \captionsetup{hypcap=true}
  \label{fig:teaser}
\end{strip}

\begin{abstract}
Visual control maps visual streams to actions in reinforcement learning and vision-language-action (VLA) policies. Changing weather, noise, blur, and compression can obscure different task cues within one episode. Existing control benchmarks mainly test scene distractions, while generic restorers optimize whole-image fidelity. We introduce the Visual Degraded Control Suite (VDCS) to evaluate seven degradations under Markov switching. We propose \emph{Agent-Centric Observations with Mixture-of-Experts} (ACO-MoE), which routes each frame to a repair expert, restores RGB content, and composes task-relevant foregrounds. On eight VDCS tasks, corruption lowers DreamerV3's mean return from 769 to 562. ACO-MoE raises it to 730, recovering 81.2\% of the lost return. It also improves returns on DMC-GB and RoboSuite and retains 70\% to 97\% of clean-input mean return under six unseen perturbation families. We train ACO-MoE on robot demonstrations and deploy it with a 3B-parameter VLA policy on a Piper arm for cube placement in blue and white boxes and cup placement on a coaster. On the blue-box and cup tasks, motion blur and physical low light with added noise reduce success from 100\% to at most 10\%. ACO-MoE raises success to 75\% to 95\% over 20 trials per task and condition.

\end{abstract}

\section{Introduction}
\label{sec:introduction}

Visual control selects actions from a stream of visual observations in applications ranging from video games to robotic locomotion and manipulation~\cite{mnih2015human,hafner2023mastering,zhou2024robodreamerlearningcompositional,yang2023learninginteractivereal}. Visual reinforcement learning trains policies and world models through interaction~\cite{xu2023drmmasteringvisual,micheli2022transformersaresample,zhang2023stormefficientstochastic,robine2023transformerbasedworld}, while vision-language-action (VLA) policies additionally use language instructions to specify robot tasks~\cite{kim2024openvla,intelligence2025pi05}. These controllers extract task cues, including object positions and relative motion, from incoming video.

In practice, however, visual inputs are often corrupted by rain, haze, sensor noise, motion blur, low light, compression, or changing backgrounds. These perturbations obscure task cues and shift observations beyond the training distribution, degrading control as Fig.~\ref{fig:teaser}a illustrates. Their effects can extend to later decisions because actions shape subsequent observations. Moreover, world models trained to reconstruct their inputs, such as DreamerV3~\cite{hafner2023mastering}, may retain corruption patterns in their latents~\cite{pan2024isodream,sun2024hrssm}. When corruption types change within an episode, the controller must handle several distinct observation shifts during one task.

Existing remedies address parts of this problem. Visual reinforcement learning methods improve generalization through augmentation~\cite{hansen2021stabilizing,laskin2020reinforcement,kostrikov2020image,song2024simgrl} or invariant representations~\cite{bertoin2022lookwhereyou,zhang2020invariant,zheng2023tacotemporallatent,zhang2024focusonwhat,han2026dspregdomainsensitiveparameterregularization}. They learn to tolerate corrupted observations without explicitly repairing damaged task cues. Image restoration offers a complementary route through denoising~\cite{dabov2007image}, deraining~\cite{xiao2022image}, and general or unified restoration~\cite{chen2022simple,li2022allone,potlapalli2023promptir,ren2024moediffir}. However, generic restorers usually optimize whole-image fidelity without considering which details determine the action. Since distinct corruptions can also require different repairs, an adapter for control should prioritize task cues and adapt its repair to each frame. These requirements motivate repair experts selected by visual content and supervised with task-relevant foregrounds.

Evaluating such an adapter also requires benchmarks in which corruptions change over time. Yet DMC-GB~\cite{hansen2021generalization} applies a fixed background-distraction regime per episode, and vision corruption benchmarks~\cite{hendrycks2019benchmarkingneuralnetwork,croce2020robustbenchastandardized,kar2022imagenet3dcc} apply one corruption per image. Neither setting tests switching among physical degradations during closed-loop interaction, motivating a benchmark with temporally changing corruptions.

To address these challenges, we introduce the Visual Degraded Control Suite (VDCS) and propose \emph{Agent-Centric Observations with Mixture-of-Experts} (ACO-MoE). VDCS switches among seven physical degradations, while ACO-MoE selects a repair expert for each frame and combines RGB repair with foreground prediction. We train ACO-MoE on paired frames from simulation and robot demonstrations, as Fig.~\ref{fig:teaser} summarizes. Our contributions are as follows.

\begin{itemize}
    \item[\textbf{(1)}] We introduce VDCS, extending DMControl~\cite{tassa2018deepmind} with seven online degradation operators, Markov-switching corruption types, and temporally correlated severities. It enables closed-loop evaluation under visual perturbations that persist and change within an episode. Relative to fixed corruptions, switching lowers every evaluated baseline in Supplementary Table~9, including a 16\% drop for the test-time adapted FTR, while frame-wise ACO-MoE maintains its return.
    \item[\textbf{(2)}] We train ACO-MoE to jointly learn frame-wise routing, RGB repair, and foreground composition from paired frames and offline pseudo masks. Experiments on VDCS, DMC-GB, and RoboSuite show gains across model-free and model-based backbones, supported by component ablations and matched-controller comparisons. It recovers 81.2\% of the return DreamerV3 loses under switching corruptions and retains 70\% to 97\% of the clean-input mean return under six unseen perturbation families.
    \item[\textbf{(3)}] We build a real-robot system with ACO-MoE and a 3B-parameter VLA policy on an AgileX Piper arm for cube placement in the blue and white boxes and cup placement on a coaster. ACO-MoE restores both camera views before the policy predicts actions. Under injected motion blur and real dim lighting with added noise, it raises closed-loop success on the blue-box and cup tasks from at most 10\% to between 75\% and 95\%, a gain of 70 to 90 percentage points.
\end{itemize}
\section{Related Work}
\label{sec:related_work}

\noindent\textbf{Robust visual control under perturbations.} World models such as DreamerV3~\cite{hafner2023mastering} and TD-MPC2~\cite{hansen2024tdmpc2}, and language-conditioned robot policies such as OpenVLA~\cite{kim2024openvla} and $\pi_0$~\cite{black2024pi0}, act on camera images and must handle visual degradation. RAD~\cite{laskin2020reinforcement} and DrQ~\cite{kostrikov2020image} improve learning from pixels through augmentation. DrQ-v2~\cite{yarats2021mastering} extends this approach to harder continuous-control tasks, while SVEA~\cite{hansen2021stabilizing} stabilizes learning under strong augmentation. CURL~\cite{laskin2020curl} learns contrastive representations, and SGQN~\cite{bertoin2022lookwhereyou} guides feature learning with saliency. However, these methods mainly change how controllers use inputs without restoring degraded task cues. Input adapters provide a complementary approach. FTR~\cite{wang2025ftr} selects task-relevant objects and adapts its selection mechanism to the target environment. VAI~\cite{wang2021vai} learns an unsupervised attention adapter that reconstructs foreground masks from images augmented with visual distractors. Both target visual distractors without explicitly repairing degraded foreground RGB. ACO-MoE instead addresses switching physical degradations through routed repair experts and task-oriented foreground supervision, and supplies restored observations to world models and VLA policies.

\noindent\textbf{Visual corruption benchmarks.} ImageNet-C~\cite{hendrycks2019benchmarkingneuralnetwork} and 3D Common Corruptions~\cite{kar2022imagenet3dcc} test recognition under graded image corruptions, and RobustBench~\cite{croce2020robustbenchastandardized} standardizes robustness comparisons. These recognition benchmarks score images independently and do not measure how observation errors change actions and affect subsequent views or eventual task completion. For closed-loop control, DMC-GB~\cite{hansen2021generalization} changes colors and replaces backgrounds with videos. The Distracting Control Suite~\cite{stone2021distracting} varies backgrounds, colors, and camera pose, including within an episode, while DMC-VB~\cite{ortiz2024dmc} adds static and dynamic distractors. Colosseum~\cite{pumacay2024colosseum} perturbs object properties, lighting, and camera pose in manipulation. These settings emphasize scene and viewpoint changes. However, weather, sensor noise, motion blur, and compression can corrupt the image itself and obscure the agent. VDCS evaluates seven such degradations in closed loop, with Markov-switching types and drifting severities.

\noindent\textbf{Image restoration and Mixture-of-Experts.} Restormer~\cite{zamir2022restormer} and NAFNet~\cite{chen2022simple} provide restoration architectures, while AirNet~\cite{li2022allone}, PromptIR~\cite{potlapalli2023promptir}, and AdaIR~\cite{cui2025adair} handle multiple degradation types within one model. Conditional computation offers a way to specialize processing across inputs. Sparsely gated MoE~\cite{shazeer2017outrageously} selects experts per input, V-MoE~\cite{riquelme2021scaling} applies routing to vision, and MoE-DiffIR~\cite{ren2024moediffir} uses degradation-aware routing to combine prompts for compressed-image restoration. These methods evaluate recovered pixels, but high image fidelity alone does not establish that repaired observations support task completion, and expert specialization does not identify which scene content matters for an action. ACO-MoE pairs routed RGB repair with foreground supervision to recover task-relevant observations for control.

\section{Preliminaries}
\label{sec:preliminaries}

\paragraph{Visual control from pixels.}
\label{sec:prelim_vrl}

To formalize this control setting, we use a partially observable Markov decision process (POMDP) $(\mathcal{S},\mathcal{A},\mathcal{T},R,\Omega,\gamma)$. The agent does not observe the state $s_t$. It receives a pixel observation $x_t\in\Omega$ and learns a policy $\pi(a_t\mid x_{\le t})$ that maximizes the expected return $\mathbb{E}_{\pi}[\sum_{t}\gamma^t r_t]$ with $r_t=R(s_t,a_t)$.

\paragraph{Observation model (clean vs.\ corrupted).}
To distinguish scene content from corruption, let $o_t^{\mathrm{raw}}=\mathcal{R}(s_t)\in[0,255]^{H\times W\times 3}$ denote the clean 8-bit RGB frame rendered by the simulator renderer $\mathcal{R}$. VDCS samples a corruption mode $k_t\in\mathcal{K}$ at each time step, where $\mathcal{K}$ indexes the $|\mathcal{K}|$ corruption types. Given $k_t$, the observed 8-bit frame is generated by
\begin{equation}
x_t^{\mathrm{raw}}=\mathcal{D}_{k_t}\!\left(o_t^{\mathrm{raw}};\,\iota_t,\xi_t\right),
\label{eq:obs_raw}
\end{equation}
where $\mathcal{D}_k$ is the $k$-th degradation operator, $\iota_t\in[0,1]$ is its mode-dependent severity, and $\xi_t$ collects the random draws of the operator, such as particle positions and noise.
Backbones consume normalized pixels:
\begin{equation}
x_t \;=\; 2(x_t^{\mathrm{raw}}/255)-1 \;\in\; [-1,1]^{H\times W\times 3}.
\label{eq:obs_norm}
\end{equation}
Section~\ref{sec:vdcs} specifies the dynamics of $(k_t,\iota_t)$. The index $k_t$ belongs to the benchmark generator and is never given to ACO-MoE. We decompose the clean observation as $O_t=(F_t,B_t)$, where $F_t$ contains the pixels of the agent and task objects and $B_t$ contains the background. For a foreground-trained controller, ACO-MoE aims to recover $F_t$ and replace $B_t$ with a constant. The supplementary analysis uses uppercase letters for random variables.

\begin{figure*}[t]
  \centering
  \includegraphics[width=\linewidth]{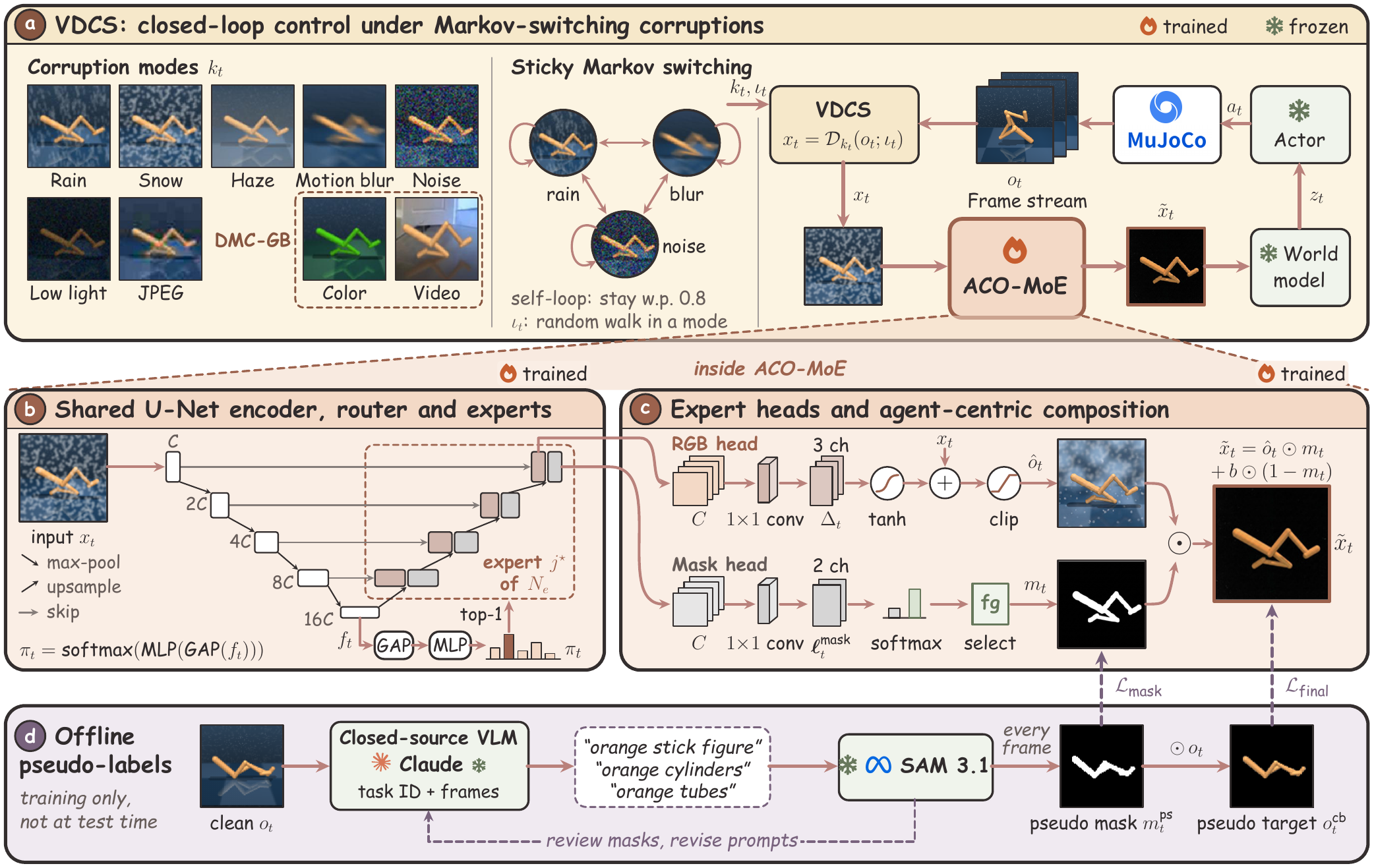}
  \caption{\textbf{VDCS and ACO-MoE overview.} (a)~VDCS corrupts each rendered frame with a mode $k_t$ that follows a sticky Markov chain, and ACO-MoE adapts the corrupted input $x_t$ for the downstream controller. The dashed cone zooms into ACO-MoE. (b)~A shared U-Net encoder and a router on its bottleneck select expert $j^\star$, whose RGB and mask U-Net decoders take the bottleneck and skips. (c)~The heads of $j^\star$ predict $\hat{o}_t$ and $m_t$, which compose $\tilde{x}_t$. (d)~For training on \texttt{walker\_run}, a closed-source vision-language model (VLM) uses the task ID and clean frames to prompt SAM~3.1 and revise prompts after mask review. We use the clean frame for $\mathcal{L}_{\mathrm{rgb}}$, the pseudo mask for $\mathcal{L}_{\mathrm{mask}}$, and the masked clean target for $\mathcal{L}_{\mathrm{final}}$. Panels (a) to (d) correspond to Secs.~\ref{sec:vdcs}, \ref{sec:dual_stream}, \ref{sec:agent_centric_method}, and~\ref{sec:router_train}.}
  \label{fig:framework}
\end{figure*}

\paragraph{Downstream controllers.}
Within this observation model, we evaluate three controller families with unchanged architectures and objectives. DreamerV3~\cite{hafner2023mastering} trains its latent with an observation reconstruction loss. TD-MPC2~\cite{hansen2024tdmpc2} uses a latent planner trained by prediction in the embedding space. On the real robot, the VLA policy $\pi_{0.5}$~\cite{intelligence2025pi05} maps two camera images, the joint state, and a language instruction to 50-step action chunks. The supplement gives the DreamerV3 objective.

\section{Visual Degraded Control Suite (VDCS)}
\label{sec:vdcs}

Weather, illumination, sensor noise, blur, and compression can persist and change during deployment. VDCS models this behavior on DMControl~\cite{tassa2018deepmind}. At each policy step, it applies one image operator to the freshly rendered RGB frame before passing it to the controller, as Fig.~\ref{fig:framework}a shows. Dynamics and the reward function remain unchanged. The mode $k_t$ follows a sticky Markov chain, and severity $\iota_t$ drifts within each persistent segment.

\paragraph{Operators.}
Table~\ref{tab:vdcs_ops} specifies seven operators $\mathcal{D}_k$ and base severities $\bar\iota_k$. Rain and snow blend blurred particle layers into the frame. Haze blends a vertical opacity ramp with constant airlight. Motion blur convolves with a rotated line kernel. Gaussian noise adds independent pixel noise, low light reduces intensity and adds noise, and JPEG encodes then decodes each frame. All act on 8-bit RGB with severity $\iota\in[0,1]$. Particle positions, streak angles, noise, and blur direction are re-sampled each frame. Lengths and radii use a $64{\times}64$ reference and scale with $s{=}\min(H,W)/64$, while particle counts preserve the covered-area fraction. Supplementary Sec.~C gives the formulas, sampling rules, and scaling to robot frames.

\begin{table}[t]
\centering
\caption{VDCS degradation operators and base severities $\bar\iota$.}
\label{tab:vdcs_ops}
\footnotesize
\setlength{\tabcolsep}{3pt}
\begin{tabular}{@{}l>{\raggedright\arraybackslash}p{0.69\linewidth}c@{}}
\toprule
\textbf{Mode} & \textbf{Operator (64\,px reference)} & $\bar\iota$ \\
\midrule
Rain & $\lfloor 500\iota\rfloor$ streaks, length 3 to 9\,px, $80^\circ{\pm}10^\circ$, $\alpha{=}0.3$ & 0.60 \\
Snow & $\lfloor 1000\iota\rfloor$ flakes, intensity 200 to 254, $\alpha{=}0.5$ & 0.60 \\
Haze & $\alpha(y){=}0.7\,\iota\,(y/H)^{1/2}$, airlight 200 & 0.60 \\
Motion blur & line kernel, length $5{+}20\iota$\,px, $\theta{\sim}\mathcal{U}[0,2\pi)$ & 0.35 \\
Gauss.\ noise & additive, $\sigma{=}25\iota$ & 0.50 \\
Low light & intensity ${\times}(1{-}0.8\iota)$, noise $\sigma{=}15\iota$ & 0.70 \\
JPEG & encode--decode, quality $90{-}80\iota$ & 0.70 \\
\bottomrule
\end{tabular}
\end{table}

\paragraph{Markov dynamics.}
VDCS initializes the mode uniformly over seven degradations and advances the chain once per policy step, after action repetition. It keeps $k_{t+1}{=}k_t$ with probability $p_s{=}0.8$, giving a mean dwell time of five policy steps, and otherwise switches uniformly among the other six modes. On entering mode $k$, the severity is drawn from $\mathcal{U}[\iota^{k}_{\min},\iota^{k}_{\max}]$ with $\iota^{k}_{\min}{=}\max(0.1,0.9\bar\iota_k)$ and $\iota^{k}_{\max}{=}\min(1,1.1\bar\iota_k)$. While the mode persists, the severity follows the clipped random walk $\iota_{t+1}{=}\mathrm{clip}(\iota_t{+}\eta_t,\iota^{k}_{\min},\iota^{k}_{\max})$ with $\eta_t{\sim}\mathcal{N}(0,0.02^2)$.

\paragraph{Protocol.}
VDCS corrupts controller inputs at evaluation. A fixed seed per episode gives every method the same sequence of modes and severities. No method receives $k_t$. The same operators generate offline ACO-MoE training pairs in Secs.~\ref{sec:router_train} and~\ref{sec:exp_real}.

\section{ACO-MoE}
\label{sec:method}

ACO-MoE maps corrupted inputs $x_t$ to observations $\tilde{x}_t$ suited to the downstream controller. Two requirements shape its design. First, different corruptions obscure different visual cues, so repair should depend on the current frame. Second, task-relevant objects should remain visible as backgrounds vary. Routed RGB repair addresses the first requirement, and foreground composition addresses the second for foreground-trained controllers. The supplementary analysis bounds the corruption information retained by reconstruction and characterizes the ideal clean foreground. We jointly train a shared encoder, a compact router, and $N_e$ repair experts, with $N_e$ controlling model capacity. Experts are not assigned to predefined corruption categories, and the router is not trained with corruption labels.

\paragraph{Matching the controller input.}
Following FTR's policy-reuse setting~\cite{wang2025ftr}, we train DreamerV3 and TD-MPC2 on observations processed by a mask predictor trained with the data in Fig.~\ref{fig:framework}d. On uncorrupted inputs, this DreamerV3 scores 769, while DreamerV3 trained and evaluated on full scenes scores 767, so the foreground interface does not lower the mean clean return. Masking also provides the largest gain in the ablations of Sec.~\ref{sec:exp_ablation}. The robot VLA is fine-tuned on full-scene demonstrations from the real environment, so ACO-MoE supplies repaired full frames in Sec.~\ref{sec:exp_real}.

\subsection{Shared Encoder, Router, and Experts}
\label{sec:dual_stream}

Fig.~\ref{fig:framework}b shows the shared four-level U-Net encoder. It takes $x_t$ from Eq.~\eqref{eq:obs_norm} and applies two blocks of a $3{\times}3$ convolution, BatchNorm, and SiLU per stage, with $2{\times}2$ max-pooling between stages. Widths are $C\!\cdot\![1,2,4,8]$, with a $16C$ bottleneck $f_t$. The router computes
\begin{equation}
\pi_t=\mathrm{softmax}\!\left(\mathrm{MLP}(\mathrm{GAP}(f_t))\right)
\in\Delta^{N_e-1},
\label{eq:router}
\end{equation}
where $\Delta^{N_e-1}$ denotes the probability simplex, GAP denotes global average pooling, and the MLP has widths $16C{\to}256{\to}N_e$. Global pooling supports frame-level expert selection. Each expert has RGB and mask decoders with four levels of bilinear upsampling, skip concatenation, and two convolution blocks. Their $1{\times}1$ heads predict a residual $\Delta_t^{(j)}\in\mathbb{R}^{H\times W\times3}$ and two-class logits $\boldsymbol{\ell}_t^{\mathrm{mask},(j)}\in\mathbb{R}^{H\times W\times2}$. At inference, $j^\star=\arg\max_j\pi_{t,j}$ selects one expert. In simulation, $C{=}9$, $N_e{=}9$, and $84{\times}84$ inputs give 4.90M parameters. On the robot, $C{=}32$, $N_e{=}7$, and inputs are $168{\times}224$. Top-1 routing executes 8.0M of the 26.8M parameters.

\subsection{Expert Heads and Agent-Centric Composition}
\label{sec:agent_centric_method}

Fig.~\ref{fig:framework}c converts expert predictions to the controller input. Each expert $j$ adds a bounded RGB residual,
\begin{equation}
\hat{o}^{(j)}_t = \mathrm{clip}_{[-1,1]}\!\left(x_t+\tanh(\Delta_t^{(j)})\right),
\label{eq:rgb_repair}
\end{equation}
where clipping acts element-wise. This parameterization lets an expert learn the correction to $x_t$. During soft-routing training, we mix the repaired images and mask logits,
\begin{equation}
\hat{o}_t=\sum_{j=1}^{N_e}\pi_{t,j}\hat{o}_t^{(j)}, \qquad
\boldsymbol{\ell}^{\mathrm{mask}}_t=\sum_{j=1}^{N_e}\pi_{t,j}\boldsymbol{\ell}^{\mathrm{mask},(j)}_t.
\label{eq:soft_routing}
\end{equation}
At inference, these are the outputs of $j^\star$. The mask head produces foreground probabilities $m_t=\mathrm{softmax}(\boldsymbol{\ell}_t^{\mathrm{mask}})_{[\mathrm{fg}]}\in[0,1]^{H\times W}$, where $[\mathrm{fg}]$ selects the foreground class. For a foreground-trained controller, we compose
\begin{equation}
  \tilde{x}_t = \hat{o}_t \odot m_t + b \odot (1-m_t),
  \label{eq:agent_centric_def}
\end{equation}
where $b=-1$ is the normalized clean-background value and $\odot$ broadcasts across RGB channels. Small values of $m_t$ suppress background texture and its corruptions.

\subsection{Offline Pseudo-Labels and Training}
\label{sec:router_train}

\paragraph{Offline pseudo-labels.}
Fig.~\ref{fig:framework}d shows the supervision pipeline. For each DMControl or RoboSuite task, the closed-source VLM Claude Opus~5.5 receives the task ID and three clean frames and proposes short foreground prompts. For \texttt{walker\_run}, these include \emph{orange stick figure}, \emph{orange cylinders}, and \emph{orange tubes}. It reviews SAM~3.1~\cite{carion2026sam3} masks for up to three rounds using the clean frame, mask overlay, and masked frame, and revises prompts as needed. SAM~3.1 then labels all clean training frames. The VLM audits the two smallest masks and the largest mask for up to two rounds, after which the prompts are fixed for labeling. The final prompts yield a binary union mask $m_t^{\mathrm{ps}}\in\{0,1\}^{H\times W}$ and target $o_t^{\mathrm{cb}}=o_t\odot m_t^{\mathrm{ps}}+b\odot(1-m_t^{\mathrm{ps}})$. Here $o_t$ is the normalized clean frame. These targets train ACO-MoE, whose mask predictor processes observations for controller training and predicts masks from corrupted frames at deployment. The supplement gives the model identifier, prompts, and review procedure.

\paragraph{Training objective.}
Each tuple $(x_t,o_t,o_t^{\mathrm{cb}},m_t^{\mathrm{ps}})$ pairs a VDCS- or DMC-GB-corrupted frame with its clean frame, target, and mask. We jointly train the encoder, router, and decoders with
\begin{equation}
  \mathcal{L} = \lambda_{\mathrm{rgb}}\mathcal{L}_{\mathrm{rgb}} + \lambda_{\mathrm{mask}}\mathcal{L}_{\mathrm{mask}} + \lambda_{\mathrm{final}}\mathcal{L}_{\mathrm{final}},
  \label{eq:total_loss}
\end{equation}
where $\mathcal{L}_{\mathrm{rgb}}=\|\hat{o}_t-o_t\|_1$ supervises RGB repair toward the clean frame. The mask loss $\mathcal{L}_{\mathrm{mask}}=\mathrm{CE}(\boldsymbol{\ell}^{\mathrm{mask}}_t,m_t^{\mathrm{ps}})$ applies pixel-wise cross-entropy to the pseudo mask. Finally, $\mathcal{L}_{\mathrm{final}}=\|\tilde{x}_t-o_t^{\mathrm{cb}}\|_1$ supervises the composed observation supplied to the controller. This last term couples RGB repair and foreground prediction through their combined output. The three $\lambda$ scalars weight the terms, and gradients reach the router through Eq.~\eqref{eq:soft_routing}.

\paragraph{Training on robot demonstrations.}
To restore the robot camera views, we train ACO-MoE with RGB reconstruction and a frequency-domain loss. The robot network uses RGB decoders and returns full frames. Soft training mixes repaired expert images as in Eq.~\eqref{eq:soft_routing}. We set $\lambda_{\mathrm{mask}}{=}\lambda_{\mathrm{final}}{=}0$ and use $\mathcal{L}=\|\hat o_t-o_t\|_1+0.05\,\|\mathcal{F}(\hat o_t)-\mathcal{F}(o_t)\|_1$, where $\mathcal{F}$ denotes the 2D discrete Fourier transform. Two label-free regularizers encourage balanced and confident routing. The first is a load-balance term $N_e\sum_j\bar\pi_j^2-1$, where $\bar\pi_j$ is the batch-averaged routing weight of expert $j$, and the second is the normalized per-frame routing entropy. They are weighted by 0.01 and 0.005, respectively. The softmax temperature of the router is annealed from 1 to 0.5, and the last 15\% of updates use hard top-1 routing to match the deployed path. A task-weighted version with $C{=}16$ keeps the full-frame output but weights the per-pixel $\ell_1$ error by $1{+}4\,m_t^{\mathrm{ps}}$, where $m_t^{\mathrm{ps}}$ is the VLM-prompted SAM~3.1 mask of the gripper and task objects. This concentrates fidelity on task content while preserving the full scene. Remaining hyperparameters are in the supplement.
\begin{table*}[t]
\centering
\caption{\textbf{VDCS Markov-switching returns.} Recovered is the percentage of the same controller's corruption-induced loss that ACO-MoE restores, pooled in Average. Return averages: mean $\pm$ std across tasks. \ranklegend{} on corrupted inputs.}
\label{tab:vdcs_markov_temporal_results}
\small
\setlength{\tabcolsep}{1.6pt}
\begin{tabularx}{\textwidth}{@{}l*{6}{>{\hsize=.95\hsize\centering\arraybackslash}X}>{\hsize=.95\hsize\centering\arraybackslash}X>{\hsize=.95\hsize\centering\arraybackslash}X>{\hsize=1.05\hsize\centering\arraybackslash}X>{\hsize=1.35\hsize\centering\arraybackslash}X@{}}
\toprule
\multirow{2}{*}{\textbf{Task}} & \multicolumn{6}{c}{\textbf{Robust-control baselines}} & \multicolumn{4}{c}{\textbf{Matched DreamerV3}~\cite{hafner2023mastering}} \\
\cmidrule(lr){2-7}\cmidrule(l){8-11}
& \makecell{SVEA\\\cite{hansen2021stabilizing}} & \makecell{SODA\\\cite{hansen2021generalization}} & \makecell{SGQN\\\cite{bertoin2022lookwhereyou}} & \makecell{SimGRL\\\cite{song2024simgrl}} & \makecell{Q$^2$\\\cite{liao2023q2}} & \makecell{FTR\\\cite{wang2025ftr}} & \makecell{No\\preproc.} & \makecell{Clean\\reference} & \textbf{ACO-MoE} & \makecell{Recovered\\(\%)} \\
\midrule
CP-S & \tabstat{649}{96} & \tabstat{615}{149} & \tabstat{354}{61} & \tabstat{613}{176} & \thirdcell{\tabstat{683}{230}} & \tabstat{342}{53} & \secondcell{802} & 878 & \bestcell{\tabstat{862}{75}} & 78.9 \\
F-S & \bestcell{\tabstat{961}{32}} & \tabstat{307}{120} & \secondcell{\tabstat{948}{87}} & \thirdcell{\tabstat{778}{69}} & \tabstat{379}{63} & \tabstat{545}{54} & 483 & 543 & \tabstat{489}{102} & 10.0 \\
F-TH & \tabstat{60}{129} & \tabstat{59}{38} & \thirdcell{\tabstat{287}{89}} & \tabstat{31}{20} & \tabstat{23}{15} & \tabstat{105}{37} & \secondcell{558} & 929 & \bestcell{\tabstat{878}{65}} & 86.3 \\
H-S & \tabstat{69}{89} & \tabstat{475}{73} & \tabstat{5}{3} & \tabstat{603}{80} & \tabstat{310}{178} & \thirdcell{\tabstat{659}{75}} & \secondcell{822} & 948 & \bestcell{\tabstat{902}{50}} & 63.5 \\
H-H & \tabstat{41}{27} & \tabstat{44}{15} & \tabstat{6}{4} & \tabstat{3}{2} & \thirdcell{\tabstat{165}{50}} & \tabstat{109}{15} & \secondcell{216} & 359 & \bestcell{\tabstat{333}{46}} & 81.8 \\
C-R & \tabstat{221}{43} & \tabstat{158}{62} & \tabstat{166}{35} & \thirdcell{\tabstat{237}{85}} & \secondcell{\tabstat{309}{67}} & \tabstat{4}{2} & 13 & 749 & \bestcell{\tabstat{699}{95}} & 93.2 \\
W-W & \tabstat{837}{42} & \tabstat{171}{58} & \tabstat{792}{87} & \thirdcell{\tabstat{845}{97}} & \tabstat{443}{85} & \tabstat{532}{38} & \secondcell{915} & 974 & \bestcell{\tabstat{946}{87}} & 52.5 \\
W-R & \thirdcell{\tabstat{279}{59}} & \tabstat{128}{22} & \tabstat{229}{74} & \tabstat{188}{64} & \tabstat{57}{37} & \tabstat{232}{31} & \secondcell{687} & 772 & \bestcell{\tabstat{731}{56}} & 51.8 \\
\midrule
\rowcolor{gray!7}
\textbf{Average} & \tabstack{389.6}{335.7} & \tabstack{244.6}{211.0} & \tabstack{348.4}{333.1} & \thirdcell{\tabstack{412.3}{314.9}} & \tabstack{296.1}{202.9} & \tabstack{316.1}{226.5} & \secondcell{\tabstack{562.0}{294.0}} & \tabstack{769.0}{203.4} & \bestcell{\tabstack{730.0}{203.8}} & \textbf{81.2} \\
\bottomrule
\end{tabularx}
\end{table*}

\section{Experiments}
\label{sec:experiments}

\subsection{Experimental Setup}
\label{sec:exp_setup}
\paragraph{Benchmarks.}
\textbf{VDCS} applies the Markov-switching degradations of Section~\ref{sec:vdcs} to DMControl with $p_s{=}0.8$. \textbf{DMC-GB}~\cite{hansen2021generalization} applies the \texttt{color\_hard} and \texttt{video\_hard} background distractions at the episode level and serves as a complementary test under a different perturbation regime.
\paragraph{Tasks.} The eight DMControl tasks in Table~\ref{tab:vdcs_markov_temporal_results} span balance, whole-body locomotion, and fine-grained manipulation. Two RoboSuite tasks, \texttt{Door} and \texttt{Lift}, test transfer to TD-MPC2~\cite{hansen2024tdmpc2} in Supplementary Sec.~H.3, and Sec.~\ref{sec:exp_real} evaluates manipulation on a real arm. Abbreviations: CP-S, cartpole swingup, F-S, finger spin, F-TH, finger turn hard, H-S, hopper stand, H-H, hopper hop, C-R, cheetah run, W-W, walker walk, and W-R, walker run.

\paragraph{Baselines.}
\emph{Data-augmentation methods} include DrQ-v2~\cite{yarats2021mastering}, SVEA~\cite{hansen2021stabilizing}, SODA~\cite{hansen2021generalization}, SGQN~\cite{bertoin2022lookwhereyou}, SimGRL~\cite{song2024simgrl}, and Q$^2$~\cite{liao2023q2}. Among \emph{task-aware methods}, FTR~\cite{wang2025ftr} is the most relevant. FTR uses SAM-based segmentation and test-time fine-tuning to adapt the foreground. ACO-MoE learns offline from paired frames and SAM~3.1 foreground targets and requires no test-time adaptation. DrQ-v2 results appear in Supplementary Table~16 because its returns stay below 70 on every task. On DMC-GB, we additionally include PAD~\cite{hansen2021pad}, DrQ~\cite{kostrikov2020image}, SRM~\cite{huang2022spectrum}, and SMG~\cite{zhang2024focusonwhat}.

\paragraph{Training settings.}
We train the simulation repair network for 15{,}000 updates on paired frames as in Sec.~\ref{sec:router_train}. Our controllers train for 1M environment steps on DMControl and 200K on RoboSuite using observations processed by the learned mask predictor. The DreamerV3 comparisons share one policy per task, with controller parameters fixed across clean, corrupted, and repaired inputs. We report means and standard deviations over 5 seeds and 10 episodes per seed.

\subsection{Main Results on VDCS}
\label{sec:exp_vdcs}

Table~\ref{tab:vdcs_markov_temporal_results} compares ACO-MoE with robust-control baselines and reports a matched DreamerV3 controller on clean, corrupted, and repaired inputs. With this controller, ACO-MoE raises average return from 562 to 730, the highest among the evaluated methods. Recovered divides the adapter gain by the clean-to-corrupted loss, giving $1344/1656{=}81.2\%$ pooled recovery. Cheetah run has the largest recovery, 93.2\%, with return rising from 13 to 699 against a clean reference of 749. On finger spin, ACO-MoE recovers 6 of the 60 lost points, or 10.0\%, and retains 90.1\% of the clean return of 543.

\paragraph{Fixed versus switching corruptions.}
Supplementary Table~9 evaluates each corruption fixed for a whole episode. Switching lowers the average return of DreamerV3 without preprocessing from 594 to 562, of SimGRL from 423 to 412, and of FTR from 376 to 316. ACO-MoE scores 719 under fixed corruptions and 730 under switching. Both settings use the same trained adapter and controller. ACO-MoE processes each frame independently, so switching requires no adaptation updates.

\begin{table}[t]
  \centering
  \caption{ACO-MoE gives the highest DMC-GB averages on both background regimes. Values average five tasks. \ranklegend{}.}
\label{tab:dmcgb_summary}
  \small
  \begin{tabularx}{\columnwidth}{@{}lY>{\columncolor{white}[\tabcolsep][0pt]\centering\arraybackslash}X@{}}
    \toprule
    \textbf{Method}
      & \textbf{Video Hard}
      & \textbf{Color Hard} \\
    \midrule
SGQN~\cite{bertoin2022lookwhereyou} & \thirdcell{\tabstat{590}{275}} & \thirdcell{\tabstat{593}{223}} \\
SMG~\cite{zhang2024focusonwhat} & \secondcell{\tabstat{749}{261}} & \secondcell{\tabstat{697}{232}} \\
\textbf{ACO-MoE} & \bestcell{\tabstat{797}{189}} & \bestcell{\tabstat{806}{194}} \\
\bottomrule
  \end{tabularx}
\end{table}

\begin{figure*}[t]
  \centering
  \begin{minipage}[t]{\dimexpr(\textwidth-\columnsep)/2\relax}
    \vspace{0pt}
    \centering
    \includegraphics[width=\linewidth]{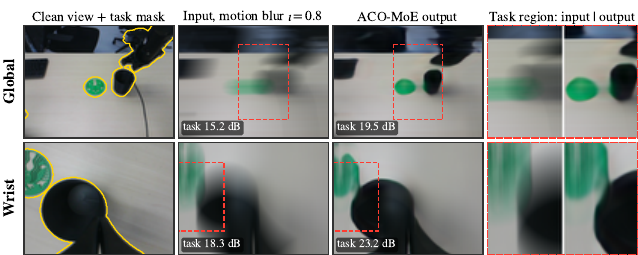}
    \caption{\textbf{Task-weighted restoration in closed loop.} A cup trial with motion blur at $\iota{=}0.8$ on both cameras. Yellow: SAM~3.1 task mask, used only to weight training. Dashed boxes: zoomed region. Labels: PSNR inside the mask.}
    \label{fig:real_mask}
  \end{minipage}\hfill
  \begin{minipage}[t]{\dimexpr(\textwidth-\columnsep)/2\relax}
    \vspace{0pt}
    \centering
    \captionof{table}{\textbf{ACO-MoE improves closed-loop success.} Columns report operator-judged successes over trials for the same $\pi_{0.5}$~\cite{intelligence2025pi05} policy with and without ACO-MoE. The $p$-values use the one-sided Fisher exact test.}
    \label{tab:real_closed_loop_main}
    \small
    \renewcommand{\arraystretch}{1.50}
    \setlength{\tabcolsep}{1pt}
    \begin{tabular*}{\linewidth}{@{\extracolsep{\fill}}lcccccc@{}}
      \toprule
      & \multicolumn{3}{c}{\textbf{Cube to blue box}} & \multicolumn{3}{c}{\textbf{Cup to coaster}} \\
      \cmidrule(lr){2-4}\cmidrule(l){5-7}
      \textbf{Perturbation} & w/o & ours & $p$ & w/o & ours & $p$ \\
      \midrule
      Clean & 20/20 & 20/20 & 1.00 & 20/20 & 20/20 & 1.00 \\
      Motion blur 0.6 & 1/20 & \textbf{19/20} & $3{\cdot}10^{-9}$ & 2/20 & \textbf{16/20} & $8{\cdot}10^{-6}$ \\
      Dim room + noise 0.5 & 0/20 & \textbf{18/20} & $2{\cdot}10^{-9}$ & 0/20 & \textbf{15/20} & $4{\cdot}10^{-7}$ \\
      \bottomrule
    \end{tabular*}
  \end{minipage}
\end{figure*}

\begin{figure}[t]
  \centering
  \includegraphics[width=\linewidth]{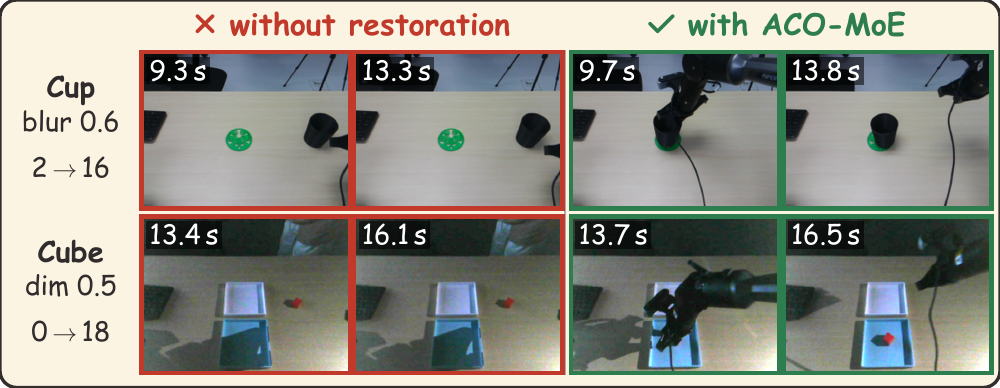}
  \caption{\textbf{Closed-loop rollouts on the real arm.} Global-camera frames of one trial per row under motion blur ($\iota{=}0.6$) and a physically dimmed room with noise ($\iota{=}0.5$). Left: successes out of 20 without and with ACO-MoE. Dim-room frames are brightened for display.}
  \label{fig:real_strip}
\end{figure}

\subsection{Ablations}
\label{sec:exp_ablation}

Supplementary Fig.~2 removes one component at a time on eight VDCS tasks under Markov switching. A parameter-matched single U-Net loses 6\% to 16\% on every task, so routing improves repair at matched capacity. Dropping the mask loss, denoted RGB-Only, loses 27\% to 84\% on seven tasks, more than dropping the RGB loss, denoted Mask-Only. Supplementary Table~7 shows the same ordering on Video Hard, where the single expert, Mask-Only, and RGB-Only reduce mean return by 11.1\%, 10.5\%, and 57.4\%, respectively.
Thus foreground composition supplies the larger contribution, while RGB
repair remains complementary. The freeze-then-adapt and no-restoration variants reduce Video-Hard mean return by 46.6\% and 62.9\%, respectively.

\subsection{Generalization to Background Distractions}
\label{sec:exp_analysis}


\paragraph{Generalization on DMC-GB.}
\label{sec:exp_dmcgb}
Beyond physical degradation, DMC-GB tests background replacement while preserving the foreground, so composition directly removes the distractor. Training pairs from both VDCS and DMC-GB let one ACO-MoE network serve both regimes. Over five DMC-GB tasks~\cite{hansen2021generalization}, ACO-MoE scores \textbf{797} on Video Hard and \textbf{806} on Color Hard in Table~\ref{tab:dmcgb_summary}, exceeding SMG by 6.4\% and 15.6\%, respectively. FTR reports 762 on Video Hard, which ACO-MoE exceeds by 4.6\%. Supplementary Tables~5 and 6 give the per-task comparisons.

\subsection{Matched Comparisons and Transfer}
\label{sec:further_analysis}

\paragraph{Matched controller and generic restorers.}
We fine-tune PromptIR~\cite{potlapalli2023promptir}, Restormer~\cite{zamir2022restormer}, AdaIR~\cite{cui2025adair}, and NAFNet~\cite{chen2022simple} from released weights on ACO-MoE's VDCS training pairs and evaluate each before the same DreamerV3 as in Sec.~\ref{sec:exp_vdcs}. Supplementary Table~10 uses \texttt{walker\_stand} in place of \texttt{walker\_run}. AdaIR leads these restorers at 546.6, while ACO-MoE reaches 759.9, a 39.0\% gain. The restorers output full frames, while ACO-MoE supplies the foreground observation expected by the controller. The complementary RGB-Only ablation in Sec.~\ref{sec:exp_ablation} loses 57.4\% on Video Hard, supporting the role of foreground supervision.

\paragraph{Transfer to unseen corruptions and other controllers.}
Under six unseen corruption families, ACO-MoE achieves mean returns of 535.8 to 744.2 in Supplementary Table~11. With DrQ-v2, it retains 90.1\% of clean return. With TD-MPC2 on RoboSuite, it exceeds FTR by 337\% on Door and 163\% on Lift under VDCS, and by 32\% and 71\% under Video Hard, covering both model-free and model-based backbones in Supplementary Fig.~11 and Sec.~H.3.

\paragraph{Runtime cost.}
At $84{\times}84$ and batch size one on an RTX~5090, ACO-MoE takes $20.54{\pm}0.93$ ms for preprocessing and $22.90{\pm}0.93$ ms per policy step, versus $76.24{\pm}1.25$ ms for FTR in Supplementary Table~3.

\subsection{Real-Robot Deployment with a VLA Policy}
\label{sec:exp_real}

\paragraph{Setup.}
Beyond simulation, we use a 6-DoF Piper arm with a fixed global camera and a wrist camera, both at $640{\times}480$ and 20\,Hz. The three tasks are cube placement in the blue box, cube placement in the white box, and cup placement on a coaster. We fine-tune $\pi_{0.5}$~\cite{intelligence2025pi05} on 193 teleoperated demonstrations collected in the real environment, totaling 65,958 frames. We train ACO-MoE for 30{,}000 updates on paired frames from these demonstrations using the objective in Sec.~\ref{sec:router_train}. We resize frames to the policy's $168{\times}224$ content region, apply VDCS operators at $\iota{\sim}\mathcal{U}[0.1,1]$, and retain 10\% clean inputs. Offline evaluation uses frames from 45 held-out episodes comprising 14,338 frames from separate recording sessions.

\paragraph{Protocol.}
We first compare action predictions from clean, corrupted, and repaired held-out camera frames. We hold the fine-tuned $\pi_{0.5}$ policy and the flow-matching noise fixed across the three inputs. For each image pair, the policy predicts a 50-step action chunk. We measure joint error against the demonstration and report its increase relative to the clean-input prediction. We also measure the mean absolute joint deviation from the clean-input prediction. For reference, two noise draws on the same clean frame differ by $0.65^\circ$. Second, we measure closed-loop task success using repaired live streams.

\paragraph{Effect of perturbations.}
At the VDCS base severities, corruption shifts the predicted joints by $0.63^\circ$ over 256 frames, close to the policy sampling spread. ACO-MoE lowers this deviation to $0.24^\circ$. The severity sweep in Supplementary Fig.~16 shows that rain, snow, haze, noise, and JPEG change the joint error by at most 8\% even at $\iota{=}1$. The policy thus tolerates these corruptions in the action comparison. In contrast, motion blur and severe low light substantially affect control. Blur raises the joint error by 15\% at $\iota{=}0.2$ and by 134\% at $\iota{=}1$, and low light at $\iota{=}1$ raises it by 117\%.

\paragraph{Effect of restoration.}
For these disruptive perturbations, ACO-MoE removes between 76\% and 92\% of the blur-induced increase at every severity and 94\% of the increase under $\iota{=}1$ low light, while it changes the actions on clean frames by only $0.015^\circ$. With BatchNorm folding, fp16, and top-1 routing, restoring both camera images takes 1.3\,ms on one RTX~5000 Ada GPU, roughly 1\% of one 116\,ms $\pi_{0.5}$ inference.

\paragraph{Closed-loop success.}
The action analysis motivates closed-loop tests under blur and low light. Table~\ref{tab:real_closed_loop_main} reports 20 operator-judged trials per condition on each of the blue-box and cup tasks, covering clean input, injected blur, and physical low light with added Gaussian noise. For low light, we physically dim the room and add noise to both live camera streams. Without restoration, $\pi_{0.5}$ succeeds in at most 2 of 20 trials under blur at $\iota{=}0.6$ and in none in the dim room with noise. With ACO-MoE, success rises to between 15 and 19 of 20 trials per task. Pooled over both tasks, it rises from 3/40 to 35/40 under blur and from 0/40 to 33/40 under physical low light with noise, while clean success stays at 20/20. Fig.~\ref{fig:real_strip} connects these outcomes to rollout sequences. Under synthetic low light at $\iota{=}0.9$ and Markov switching, the unrestored policy already succeeds in at least 8 of 10 trials, as Supplementary Table~14 reports.

\paragraph{Task-weighted repair.}
The robot must recover enough detail about the gripper and cup to complete a grasp under stronger blur. We therefore weight the RGB training loss on the task regions in both camera views. At $\iota{=}0.8$, ACO-MoE completes 19 of 20 cup trials, versus none of 10 without repair. Fig.~\ref{fig:real_mask} shows one trial, and Supplementary Table~14 gives the counts. Over 249 logged steps, ACO-MoE raises the task-mask PSNR from 15.1 to 19.9\,dB on the global camera and from 18.9 to 26.9\,dB on the wrist camera. At the same capacity, the weighting adds 1.1\,dB inside the mask of the global view and 0.7\,dB outside it.

\paragraph{Expert selection and runtime.}
The learned router selects a consistent expert for each synthetic perturbation in 11{,}687 of 11{,}688 camera-level decisions and changes experts at each corruption switch. In the physically dim room with injected noise, it selects the low-light expert on all 5{,}828 decisions. Supplementary Sec.~J.5 details the routing logs. On the robot computer, the complete image pipeline, including resizing, perturbation, repair, and padding, takes 5.5 to 9.7\,ms within the 50\,ms action period.
\section{Conclusion and Limitations}
\label{sec:conclusion}

VDCS measures how seven switching degradations affect visual control. Its image operators and shared corruption sequences enable reproducible comparisons. We train ACO-MoE to route and repair each frame and compose task-relevant observations. It recovers 81.2\% of DreamerV3's lost return on VDCS, improves control on DMC-GB and RoboSuite, and retains 70\% to 97\% of clean mean return under six unseen degradations. Component ablations support the contribution of foreground composition and the benefit of routing at matched capacity. On the Piper arm, blue-box and cup success rises from 3/40 to 35/40 under motion blur and from 0/40 to 33/40 in physical low light with noise. Repair adds about 1\% to policy inference time. Together, these results support learned observation repair as a way to sustain control under changing visual corruptions.

\textbf{Limitations.} The repair network learns from synthetic corruptions and offline foreground labels from a \mbox{closed-source} VLM and SAM~3.1. Future work can extend real-robot evaluation to broader lighting changes, motion blur, and lens contamination, measuring closed-loop task success across tasks and camera views.

\section*{LLM Usage Statement}
Claude Opus 5.5 proposes and reviews SAM 3.1 prompts for offline foreground pseudo-labels in Sec.~\ref{sec:router_train} and Supplementary Sec.~K. It also reviewed manuscript wording, grammar, and organization. OpenAI Codex applied author-approved edits to text, figure layouts, and LaTeX source. The authors verified all statements and results.

{\small
\setlength{\bibsep}{0pt plus 1pt}
\ifwithsupp\setlength{\bibsep}{-1.5pt}\fi
\bibliographystyle{ieeenat_fullname}
\bibliography{paper_references}
}

\ifwithsupp
  \clearpage
  \setcounter{figure}{0}\setcounter{table}{0}\setcounter{equation}{0}
  \renewcommand{\theHfigure}{S\arabic{figure}}
  \renewcommand{\theHtable}{S\arabic{table}}
  \renewcommand{\theHequation}{S\arabic{equation}}
  \maketitlesupplementary

\appendix

\raggedbottom
\setcounter{topnumber}{3}
\setcounter{bottomnumber}{2}
\setcounter{totalnumber}{5}
\setcounter{dbltopnumber}{3}
\renewcommand{\topfraction}{0.95}
\renewcommand{\bottomfraction}{0.85}
\renewcommand{\textfraction}{0.05}
\renewcommand{\dbltopfraction}{0.95}
\renewcommand{\floatpagefraction}{0.75}
\renewcommand{\dblfloatpagefraction}{0.75}
\makeatletter
\setlength{\@fptop}{0pt}
\setlength{\@fpsep}{14pt plus 2pt minus 2pt}
\setlength{\@fpbot}{0pt plus 1fil}
\setlength{\@dblfptop}{0pt}
\setlength{\@dblfpsep}{14pt plus 2pt minus 2pt}
\setlength{\@dblfpbot}{0pt plus 1fil}
\makeatother

\setlength{\intextsep}{7pt plus 1pt minus 1pt}
\setlength{\floatsep}{10pt plus 2pt minus 2pt}
\setlength{\textfloatsep}{12pt plus 2pt minus 2pt}
\setlength{\dbltextfloatsep}{12pt plus 2pt minus 2pt}

\section*{Supplementary Contents}
\begin{tabular}{@{}p{0.05\linewidth}p{0.90\linewidth}@{}}
\textbf{A.} & Additional Related Work \\
\addlinespace[0.12em]
\textbf{B.} & Qualitative Results and Ablation \\
\addlinespace[0.12em]
\textbf{C.} & VDCS Corruption Operators \\
\addlinespace[0.12em]
\textbf{D.} & Theoretical Analysis and Proofs \\
\addlinespace[0.12em]
\textbf{E.} & ACO-MoE Architecture Details \\
\addlinespace[0.12em]
\textbf{F.} & Offline Dataset Generation \\
\addlinespace[0.12em]
\textbf{G.} & Training Settings and Computational Cost \\
\addlinespace[0.12em]
\textbf{H.} & Additional Experimental Results and Ablations \\
\addlinespace[0.12em]
\textbf{I.} & Downstream Controller Objectives \\
\addlinespace[0.12em]
\textbf{J.} & Real-Robot Deployment with ACO-MoE \\
\addlinespace[0.12em]
\textbf{K.} & VLM Prompts for SAM~3.1 Pseudo-Labels \\
\addlinespace[0.12em]
\textbf{L.} & Additional VDCS Baseline \\
\end{tabular}

\section{Additional Related Work}
\label{app:full_related_work}

\noindent\textbf{Robust visual control under perturbations.} Data augmentation~\cite{yarats2021mastering,hansen2021stabilizing,laskin2020reinforcement,kostrikov2020image,song2024simgrl} uses random shifts and color jitter, while invariant-representation methods~\cite{laskin2020curl,zheng2023tacotemporallatent,yuan2022pretrainedimage,bertoin2022lookwhereyou,zhang2024focusonwhat,xu2023drmmasteringvisual,nakamoto2024steeringyourgeneralists,han2026dspregdomainsensitiveparameterregularization,zhang2020invariant} suppress corruption-specific features through contrastive or regularization objectives. Both approaches train the controller to use corrupted inputs directly. Foreground-aware adapters~\cite{wang2023generalizable,wang2025ftr,kim2024make,wang2021vai} introduce object-centric priors. FTR~\cite{wang2025ftr} uses SAM-supervised masking with 200K test-time adaptation steps. Likewise, world models~\cite{ha2018world,hafner2019learning,micheli2022transformersaresample,zhang2023stormefficientstochastic,robine2023transformerbasedworld,hafner2023mastering,hansen2024tdmpc2,zhou2024dinowmworld,zhou2024robodreamerlearningcompositional,yang2023learninginteractivereal,yu2026latent} can encode corruption cues~\cite{pan2024isodream,sun2024hrssm}. ACO-MoE instead uses SAM~3.1~\cite{carion2026sam3} only offline to label its pretraining frames, restores observations \emph{before} the backbone consumes them, and needs neither SAM nor test-time adaptation at deployment.

\noindent\textbf{Visual corruption benchmarks.} ImageNet-C~\cite{hendrycks2019benchmarkingneuralnetwork}, RobustBench~\cite{croce2020robustbenchastandardized}, and 3D Common Corruptions~\cite{kar2022imagenet3dcc} test fixed per-image corruptions, while test-time adaptation methods~\cite{boudiaf2022lame,liu2021tttpp} address shift without explicit labels. For control, the Distracting Control Suite~\cite{stone2021distracting}, DMC-GB~\cite{hansen2021generalization}, and broader generalization benchmarks~\cite{yuan2023rl,ortiz2024dmc,de2024sliding,pumacay2024colosseum} vary backgrounds, colors, and viewpoints, and some include dynamic variations within an episode. VDCS instead switches among heterogeneous physical image degradations within an episode.

\noindent\textbf{Image restoration and Mixture-of-Experts.} Restoration architectures evolved from U-Nets~\cite{ronneberger2015u} and transformers~\cite{zamir2022restormer} to all-in-one restorers~\cite{li2022allone,potlapalli2023promptir,xia2023diffirefficientdiffusion,conde2024instructirhighquality,guo2024mambairasimple,luo2024daclip,zhang2025perceiveirlearning,wang2023gridformerresidualdense,ren2024moediffir} that handle multiple corruption types in one network. Evaluations use pixel-fidelity metrics under i.i.d.\ clip-level corruptions. Sparse Mixture-of-Experts~\cite{jacobs1991adaptive,jordan1994hierarchical,shazeer2017outrageously,puigcerver2023fromsparseto,jiang2024mixtralofexperts,lin2024moellavamixture} routes inputs to specialized sub-networks; in vision it has been applied to restoration~\cite{ren2024moediffir} and classification~\cite{riquelme2021scaling}. Building on these tools, we pair restoration with foreground masking and frame-level routing to handle Markov-switching corruption without oracle labels.
\section{Qualitative Results and Ablation}
\label{app:supp_visualizations}
Fig.~\ref{fig:imagination_robustness} compares world-model predictions from the same motion-blurred context. ACO-MoE preserves the foreground across the four displayed steps.

\begin{figure}[H]
  \centering
  \includegraphics[width=\linewidth]{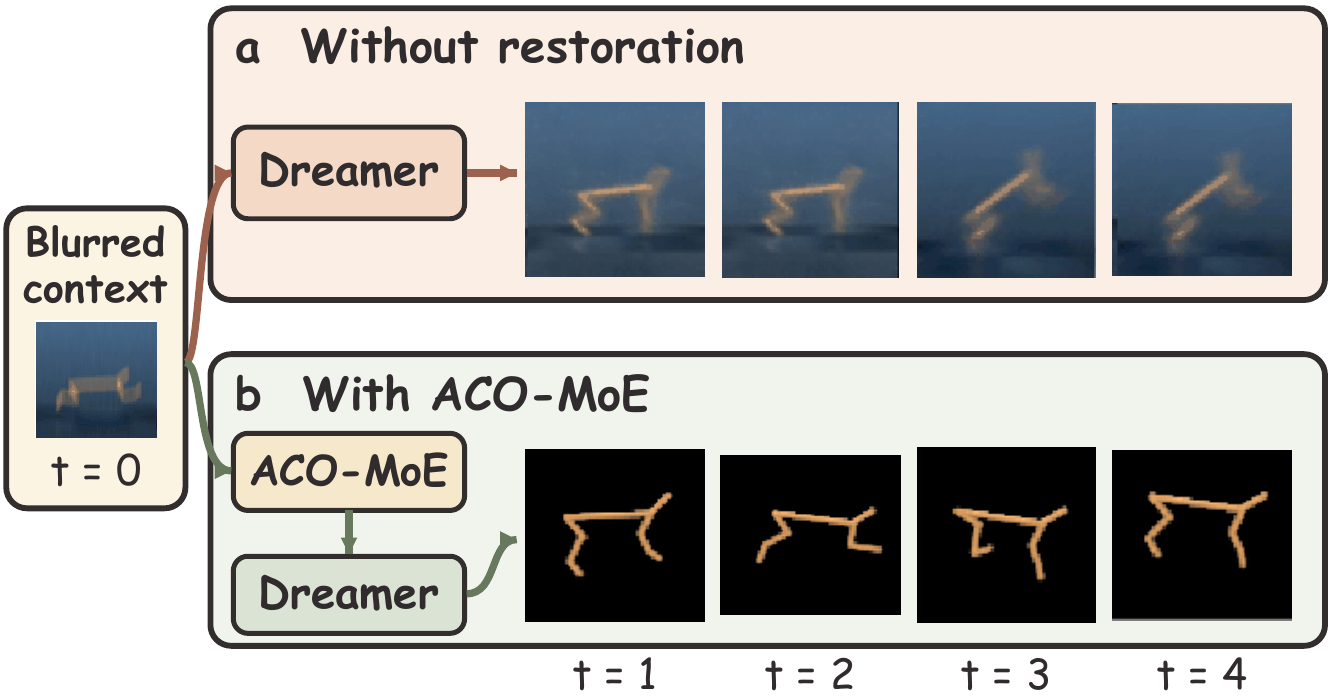}
  \caption{World-model imagination from a motion-blurred context at $t{=}0$. The upper row shows baseline Dreamer predictions, and the lower row uses ACO-MoE restoration. Columns follow the same horizon, $t{=}1$ to $4$. Original image samples are rearranged without changing their pixels.}
  \label{fig:imagination_robustness}
\end{figure}

\paragraph{Component ablation across VDCS tasks.}
Figure~\ref{fig:ablation_bars} tests each component on eight VDCS tasks, using
\texttt{walker\_stand} in place of the main benchmark's \texttt{walker\_run}. A
capacity-matched single expert loses 6--16\% on every task. RGB-Only loses
27--84\% on seven tasks, far more than Mask-Only, while omitting repair never
matches the full model. Foreground composition is therefore central, with RGB
repair and input-adaptive routing providing complementary gains. Full per-task
results appear in Section~\ref{app:results}. These drops use the figure's full-model reference. Benchmark comparisons instead use the main VDCS table
and the matching values in Fig.~\ref{fig:protocol_ablation}.

\begin{figure*}[hpbt]
    \centering
    \includegraphics[width=1.0\textwidth,trim=13 26 10 6,clip]{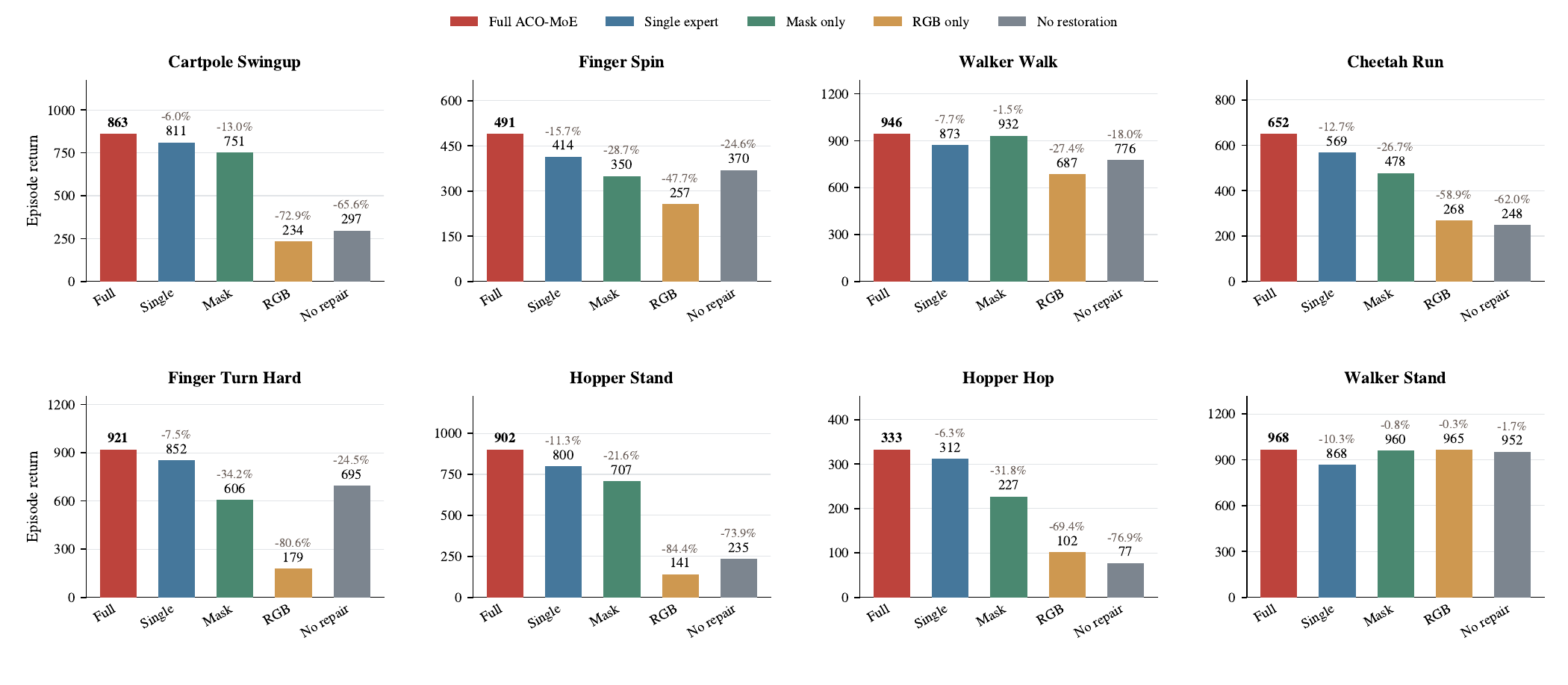}
    \caption{\textbf{Component ablation across 8 VDCS tasks.}
\textbf{Single Expert}: replace ACO-MoE with a capacity-matched U-Net.
\textbf{Mask-Only}: disable RGB repair loss.
\textbf{RGB-Only}: disable mask loss.
\textbf{w/o Repair}: remove RGB repair at inference while keeping foreground
composition. Numbers below each bar denote relative drops w.r.t.\ the full-model reference in this figure. The task set includes \texttt{walker\_stand} rather than \texttt{walker\_run}.}
    \label{fig:ablation_bars}
    \end{figure*}

\begin{figure*}[hpbt]
    \centering
    \includegraphics[width=\textwidth]{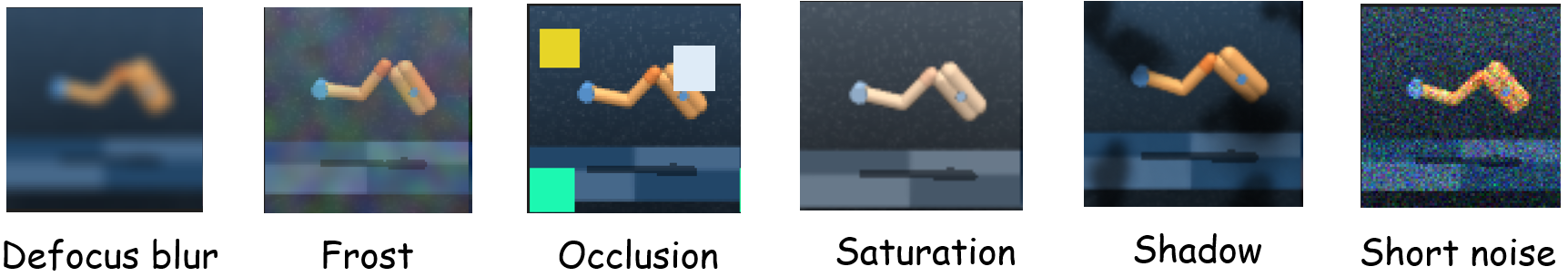}
    \caption{Examples of the six unseen OOD visual perturbations used for zero-shot evaluation: \texttt{defocus\_blur}, \texttt{frost}, \texttt{occlusion\_patch}, \texttt{saturation}, \texttt{shadow}, and \texttt{shot\_noise}. These perturbation families are excluded from ACO-MoE pretraining and are evaluated only at test time with the frozen observation adapter and frozen downstream controller.}
    \label{fig:unseen_deg}
\end{figure*}
\section{VDCS: Full Degradation Operator Definitions}
\label{app:vdcs_details}

Let $o_t^{\mathrm{raw}}=\mathcal{R}(s_t)\in[0,255]^{H\times W\times3}$ and
$x_t^{\mathrm{raw}}=\mathcal{D}_{k_t}(o_t^{\mathrm{raw}};\iota_t,\xi_t)$
as in Eq.~\eqref{eq:obs_raw}. At each policy step, VDCS takes the new rendered
frame and applies the current operator once, before the controller receives
it. All operators act in 8-bit space before Eq.~\eqref{eq:obs_norm}.
Their outputs are clipped to $[0,255]$ and stored as unsigned 8-bit pixels.
The symbol $\odot$ denotes channel-broadcast multiplication.

\subsection{Markov-Switching Settings and Severity Dynamics}

\paragraph{Mode switching.}
For $\mathcal{K}=\{1,\dots,|\mathcal{K}|\}$, $k_t$ follows a Markov chain
$\Pi$, where $\Pi_{ij}=\Pr(k_t{=}j\mid k_{t-1}{=}i)$. We use
$\Pi_{ii}=p_s$ and $\Pi_{ij}=(1-p_s)/(|\mathcal{K}|-1)$ for $j\ne i$
with $p_s=0.8$. The initial mode is uniform over the seven operators.
The chain advances once per policy step after action repetition. A fixed
seed per evaluation episode gives every method the same mode and severity
sequence, independently of the rendered scene.

\paragraph{Severity sampling.}
Each type has base severity $\bar{\iota}_k$ and jitter band $\delta=0.1$.
At $t=0$ or after a mode switch, we sample
\begin{equation}
\iota_t \sim \mathcal{U}\!\bigl[\iota_{k_t,\min},\ \iota_{k_t,\max}\bigr],
\end{equation}

\paragraph{Temporal correlation within a segment.}
Within an unchanged mode, severity follows
\begin{equation}
\label{eq:intensity_sample}
\begin{aligned}
\iota_{t+1}
  &=\mathrm{clip}\!\left(\iota_t+\eta_t,\
      \iota_{k_t,\min},\ \iota_{k_t,\max}\right),\\
\eta_t&\sim\mathcal{N}(0,0.02^2),
\end{aligned}
\end{equation}
with mode-specific bounds
\begin{equation}
\begin{aligned}
\iota_{k,\min}&=\max\!\bigl(0.1,\ (1-\delta)\bar{\iota}_k\bigr),\\
\iota_{k,\max}&=\min\!\bigl(1,\ (1+\delta)\bar{\iota}_k\bigr),
\end{aligned}
\end{equation}
and $\mathrm{clip}(u,a,b)=\min\{\max\{u,a\},b\}$.

\begin{table}[hpbt]
\centering
\caption{VDCS physical degradations and base severities $\bar{\iota}_k$.}
\label{tab:vdcs_degradations}

\small
\setlength{\tabcolsep}{6pt}
\begin{tabular}{l l c}
\toprule
\textbf{Degradation} & \textbf{Physical meaning} & $\bar{\iota}_k$ \\
\midrule
Rain & Streak overlay (weather) & 0.6 \\
Haze & Atmospheric scattering / haze & 0.6 \\
Snow & Flake overlay (weather) & 0.6 \\
Motion blur & Directional blur kernel & 0.35 \\
Gaussian noise & Additive sensor noise & 0.5 \\
Low-light & Exposure reduction + noise & 0.7 \\
JPEG compression & Lossy block artifacts & 0.7 \\
\bottomrule
\end{tabular}
\end{table}

\subsection{VDCS Degradation Operators}

Operator randomness is included in $\xi_t$. Particle positions and attributes,
pixel noise, and blur directions are drawn anew for each frame, including
frames within the same persistent mode. Haze and JPEG are deterministic
given the input frame and current severity. Thus the Markov process controls
the degradation type and strength, while $\xi_t$ controls its realization.
The following spatial parameters use the $64{\times}64$ reference resolution.

\paragraph{Rain.}
We initialize a black single-channel layer and draw
$N_r=\lfloor500\iota_t\rfloor$ white line segments at uniformly sampled
pixel positions. Each segment has width one pixel, an integer length sampled
uniformly from $\{3,\ldots,9\}$, and an angle sampled uniformly from
$70^\circ$ to $90^\circ$. A $3{\times}3$ Gaussian filter smooths the layer.
Dividing it by 255 gives $M_t\in[0,1]^{H\times W\times1}$.
We blend this layer with the frame using
\begin{equation}
x_t^{\mathrm{raw}} = (1-\gamma M_t)\odot o_t^{\mathrm{raw}} + \gamma M_t \odot c_w,
\end{equation}
where $c_w=(255,255,255)$ and $\gamma=0.3$. Pixels outside the smoothed
streaks have $M_t=0$ and retain their original values.

\paragraph{Snow.}
We draw $N_s=\lfloor1000\iota_t\rfloor$ filled circles at uniformly sampled
pixel positions on a black layer. Each circle has radius one pixel at the
reference resolution and an integer brightness sampled uniformly from 200
through 254. A $3{\times}3$ Gaussian filter smooths the layer, which is then
divided by 255 to form $M_t$. The same blending equation applies with
$\gamma=0.5$. Increasing severity therefore increases particle density,
while the sampled positions and brightness change every frame.

\paragraph{Haze.}
We blend each row with constant airlight $c_f=(200,200,200)$.
The vertical opacity profile is $\rho(y)=(y/H)^{1/2}$, where $y=0$
denotes the top row. Broadcasting this profile across columns gives
$\rho_t$. With $\alpha=0.7\iota_t$, the output is
\begin{equation}
x_t^{\mathrm{raw}} = (1-\alpha \rho_t)\odot o_t^{\mathrm{raw}} + \alpha \rho_t \odot c_f.
\end{equation}

\paragraph{Motion blur.}
We sample a direction $\theta\sim\mathcal{U}[0,2\pi)$ for each frame.
The nominal line length is
\begin{equation}
\ell = \ell_{\min} + \iota_t(\ell_{\max}-\ell_{\min}),
\qquad \ell_{\min}=5,\ \ell_{\max}=25.
\end{equation}
For simulation, we truncate $\ell$ to an integer and increase it by one
if it is even, giving an odd kernel width $m$. We place a horizontal line
of weights $1/m$ in the center row of an $m{\times}m$ zero matrix and rotate
it by $\theta$ using bilinear interpolation. The resulting kernel
$\mathbf{k}_{\theta,\ell}$ is applied to each RGB channel by convolution,
\begin{equation}
x_t^{\mathrm{raw}} = \mathbf{k}_{\theta,\ell} * o_t^{\mathrm{raw}}.
\end{equation}

\paragraph{Gaussian noise.}
We independently sample noise for every pixel and RGB channel,
$n_t\sim\mathcal{N}(0,\sigma^2 I)$, with standard deviation in 8-bit pixel units,
\begin{equation}
\sigma=\iota_t\sigma_{\max}, \qquad \sigma_{\max}=25,
\end{equation}
\begin{equation}
x_t^{\mathrm{raw}} = \mathrm{clip}_{[0,255]}\!\left(o_t^{\mathrm{raw}} + n_t\right).
\end{equation}

\paragraph{Low light.}
The synthetic low-light operator first scales image brightness and then
adds noise. Its brightness factor is
\begin{equation}
\nu = 1-\iota_t(1-\nu_{\min}), \qquad \nu_{\min}=0.2
\end{equation}
and the noise is sampled independently per pixel and channel as
$n_t\sim\mathcal{N}(0,(\iota_t\sigma_{\ell})^2 I)$ with $\sigma_{\ell}=15$. Thus
\begin{equation}
x_t^{\mathrm{raw}} = \mathrm{clip}_{[0,255]}\!\left(\nu\,o_t^{\mathrm{raw}} + n_t\right).
\end{equation}

\paragraph{JPEG compression.}
We encode each frame as a JPEG image and immediately decode it.
For this operation $\mathcal{J}_q$, the nominal encoding quality is
\begin{equation}
q = q_{\max}-\iota_t(q_{\max}-q_{\min}),
\qquad q_{\min}=10,\ q_{\max}=90,
\end{equation}
\begin{equation}
x_t^{\mathrm{raw}} = \mathcal{J}_q\!\left(o_t^{\mathrm{raw}}\right).
\end{equation}
Simulation truncates $q$ to an integer. Lower quality produces stronger
compression artifacts. Each frame is compressed independently.

\paragraph{Scaling to robot camera frames.}
For a frame of size $H{\times}W$, define the length scale
$s=\min(H,W)/64$ and area scale $a=HW/64^2$.
The robot implementation multiplies rain lengths by $s$ and uses
$\lfloor500\iota_t a/s\rfloor$ streaks, keeping line width one pixel
and the $3{\times}3$ smoothing filter. Snow uses
$\lfloor1000\iota_t a/s^2\rfloor$ circles, radii
$\max(1,\mathrm{round}(su))$ for $u\sim\mathcal{U}[0.7,1.3]$,
and integer brightness from 200 through 255.
Let $\mathrm{odd}(v)$ round $v$ to the nearest integer, clamp it to at
least one, and add one if it is even. Snow smoothing uses a square kernel
of width $\mathrm{odd}(1+2s)$, and motion blur uses width
$\mathrm{odd}((5+20\iota_t)s)$. The rotated robot blur kernel is normalized
to sum to one and applied with reflected borders. Haze, noise, and brightness
parameters stay in the same pixel units. Robot JPEG processing rounds the
quality and converts RGB to BGR for OpenCV encoding, then converts the
decoded image back to RGB. Section~\ref{sec:supp_closed_loop} separately
describes the physical low-light trials, for which we dim the room itself.

\section{Theoretical Analysis and Proofs}
\label{app:theory}

\subsection{Formal Setup and Main Results}
\label{app:theory_main}
\paragraph{Setup.}
Uppercase letters denote random variables, and lowercase letters their realizations. The state $S_t$ renders a clean observation $O_t=(F_t,B_t)$, with foreground $F_t$ and background $B_t$. The corrupted observation $X_t$ follows Eq.~\eqref{eq:obs_raw} on the finite 8-bit alphabet $\mathcal{X}=\{0,\ldots,255\}^{H\times W\times3}$. The analysis assumes exogeneity $K_t\perp S_t$ and balanced modes $p(K_t=k)=1/|\mathcal{K}|$. The VDCS sampler selects operators without consulting the task state. In particular, VDCS motion blur uses a global kernel with sampled direction and severity, rather than a kernel determined by the agent's velocity. Let $Z_t=f_\phi(X_t)$ and let $\hat X_t=Q(g_\psi(Z_t))$, where $Q$ rounds each channel to an 8-bit level. The distortion $d:\mathcal{X}\times\mathcal{X}\to[0,1]$ upper-bounds $\mathbf{1}\{X_t\ne\hat X_t\}$ and satisfies $\mathbb{E}[d(X_t,\hat X_t)]\le\epsilon$. All logarithms are base-2, $H$ denotes entropy, $I$ denotes mutual information, and $h(p)=-p\log p-(1-p)\log(1-p)$ is binary entropy. The discrete policy target $Y_t$ is a deterministic function of $S_t$, such as a quantized optimal action label. We suppress $t$ in the proofs.

A mode predictor is a function $\varphi:\mathcal{X}\to\mathcal{K}$. Mode-identifiability requires a Bayes error below random guessing for every state $s$:
\begin{equation}
P_e(s)\triangleq\inf_\varphi\Pr\!\bigl(\varphi(X_t)\ne K_t\mid S_t=s\bigr)
<1-\tfrac{1}{|\mathcal{K}|}.
\label{eq:mode_ident}
\end{equation}

\begin{proposition}[Representation Contamination]
\label{prop:contamination}
Under~\eqref{eq:mode_ident}, $K_t\perp S_t$, balanced corruption modes, and the 8-bit reconstruction setup above, if $\epsilon\in(0,\tfrac12]$, then
\begin{equation}
I(Z_t;K_t\mid S_t)\ge I(X_t;K_t\mid S_t)-C(\epsilon),
\label{eq:contamination}
\end{equation}
where $C(\epsilon)=\epsilon\log|\mathcal{K}|+h(\epsilon)$.
\end{proposition}
As the distortion budget shrinks, the lower bound approaches the corruption information in the input.

\paragraph{Information-bottleneck objective.}
To retain task information while compressing the observation, the information bottleneck solves
\begin{equation}
\min_{p(Z_t\mid X_t)} I(Z_t;X_t)
\quad\text{s.t.}\quad I(Z_t;Y_t)\ge\beta.
\label{eq:ib_standard}
\end{equation}
For deterministic $Z_t=f_\phi(X_t)$, the chain rule gives
\begin{equation}
I(Z_t;X_t)=I(Z_t;Y_t)+I(Z_t;X_t\mid Y_t).
\label{eq:ib_decompose}
\end{equation}
The residual bounds corruption dependence because
\begin{equation}
I(Z_t;X_t\mid Y_t)=H(Z_t\mid Y_t)\ge I(Z_t;K_t\mid Y_t).
\label{eq:ib_contains}
\end{equation}
Since $K_t\perp S_t$ and $Y_t$ is a function of $S_t$, $K_t\perp Y_t$. Thus the right side measures corruption dependence without task confounding.

\begin{corollary}[Foreground as an Approximate IB Anchor]
\label{cor:foreground}
Let $F_t$ and $Y_t$ be deterministic functions of $S_t$, let $K_t\perp S_t$, and define $\eta=H(Y_t\mid F_t)\ge0$. Setting $Z_t=F_t$ achieves exact nuisance invariance, $I(Z_t;K_t\mid Y_t)=0$, and approximate task sufficiency, $I(Z_t;Y_t)=H(Y_t)-\eta$. For the ideal foreground representation, this meets the task-information threshold $\beta=H(Y_t)-\eta$.
\end{corollary}
Foreground-only and full-background DreamerV3 yield comparable clean returns of 769 and 767 in Table~\ref{tab:clean_recovery}, providing an empirical comparison of the two observation interfaces.

\paragraph{From anchor to ACO-MoE.}
Physical degradations affect foreground pixels, so the ideal $F_t$ is not directly observable from $X_t$. ACO-MoE approximates it through RGB repair and foreground composition. Input-adaptive experts repair corrupted foreground appearance toward clean targets without corruption labels. Composition then replaces background variation with a constant value. These steps train the output toward the foreground anchor while retaining task content.

\subsection{Proofs of Theoretical Analysis}
\label{app:proofs}

\subsubsection{Notation and Reconstruction Model}
\label{app:notation}
We fix the time index and suppress the subscript $t$. The task state $S$ determines the clean observation $O=(F,B)$ and the discrete policy target $Y$. The corruption mode $K$ is independent of $S$. Severity and other operator randomness are included in the conditional distribution of $X$ given $K$ and $S$. Thus the argument also applies when the mode persists across adjacent frames, as long as the stated independence holds at the time being analyzed.

The encoder produces $Z=f_\phi(X)$. We quantize the decoder output as $\hat X=Q(g_\psi(Z))$, so that $X$ and $\hat X$ take values in the same finite image alphabet. Quantization also makes frame equality well defined. The distortion assumption connects the reconstruction objective to this event through
\[
\mathbf{1}\{X\ne\hat X\}\le d(X,\hat X),
\qquad \mathbb{E}[d(X,\hat X)]\le\epsilon.
\]
All entropies use base-2 logarithms. In particular, $H(U)=-\sum_u p(u)\log p(u)$, $I(U;V)=H(U)-H(U\mid V)$, and $I(U;V\mid W)=H(U\mid W)-H(U\mid V,W)$. We use $h(q)$ for the entropy of a Bernoulli variable with probability $q$.

\subsubsection{Proof of Proposition~\ref{prop:contamination}}
\label{app:proof_prop_contamination}
\begin{proof}
\par\smallskip\noindent\textbf{Step 1: Relate the latent to its reconstruction.}
Conditioned on a state $S=s$, the generative process and deterministic encoder and decoder form the Markov chain
\[
K\longrightarrow X\longrightarrow Z\longrightarrow\hat X.
\]
The data-processing inequality therefore gives $I(Z;K\mid S=s)\ge I(\hat X;K\mid S=s)$. Averaging over states yields
\begin{equation}
I(Z;K\mid S)\ge I(\hat X;K\mid S).
\label{eq:dpi}
\end{equation}
It remains to bound the corruption information lost when $X$ is replaced by $\hat X$.

\par\smallskip\noindent\textbf{Step 2: Convert distortion to frame mismatch.}
Let $E=\mathbf{1}\{X\ne\hat X\}$ and $q=\Pr(E=1)$. Taking the expectation of the distortion inequality gives
\[
q=\mathbb{E}[E]\le\mathbb{E}[d(X,\hat X)]\le\epsilon\le\tfrac12.
\]
This is the mismatch probability averaged over states and corruption randomness. The subsequent entropy bound uses this same probability.

\par\smallskip\noindent\textbf{Step 3: Bound the information lost on mismatched frames.}
Expanding the two mutual-information terms and then conditioning on the reconstruction gives
\begin{equation}
\begin{aligned}
&I(X;K\mid S)-I(\hat X;K\mid S)\\
&\quad=H(K\mid\hat X,S)-H(K\mid X,S)\\
&\quad\le H(K\mid\hat X,S)-H(K\mid X,\hat X,S)\\
&\quad=I(K;X\mid\hat X,S).
\end{aligned}
\label{eq:key_diff}
\end{equation}
The inequality follows because conditioning on $\hat X$ cannot increase $H(K\mid X,S)$. Since $E$ is determined by $X$ and $\hat X$, adding it alongside $X$ does not change the conditional information. The chain rule gives
\begin{equation}
\begin{aligned}
I(K;X\mid\hat X,S)
&=I(K;X,E\mid\hat X,S)\\
&=I(K;E\mid\hat X,S)\\
&\quad+I(K;X\mid E,\hat X,S).
\end{aligned}
\label{eq:cond_expand}
\end{equation}
For the first term,
\[
I(K;E\mid\hat X,S)\le H(E\mid\hat X,S)\le H(E)=h(q).
\]
For the second term, $E=0$ implies $X=\hat X$. Once $\hat X$ is given, $X$ therefore conveys no additional information. On the event $E=1$, the information is at most the entropy of the finite corruption-mode alphabet. Consequently,
\[
\begin{aligned}
I(K;X\mid E,\hat X,S)
&=(1-q)I(K;X\mid E=0,\hat X,S)\\
&\quad+qI(K;X\mid E=1,\hat X,S)\\
&\le q\log|\mathcal{K}|.
\end{aligned}
\]
Combining these bounds with Eq.~\eqref{eq:cond_expand} gives
\[
I(K;X\mid\hat X,S)\le h(q)+q\log|\mathcal{K}|.
\]

\par\smallskip\noindent\textbf{Step 4: Apply the distortion budget.}
Binary entropy is nondecreasing on $[0,\tfrac12]$. Since $q\le\epsilon\le\tfrac12$, we obtain
\[
h(q)+q\log|\mathcal{K}|\le h(\epsilon)+\epsilon\log|\mathcal{K}|.
\]
Using Eqs.~\eqref{eq:dpi} and~\eqref{eq:key_diff},
\[
\begin{aligned}
I(Z;K\mid S)
&\ge I(X;K\mid S)-C(\epsilon),\\
C(\epsilon)&=h(\epsilon)+\epsilon\log|\mathcal{K}|.
\end{aligned}
\]
This is Eq.~\eqref{eq:contamination}.

\par\smallskip\noindent\textbf{Step 5: Establish positive corruption information.}
Exogeneity and balanced modes give $H(K\mid S=s)=\log|\mathcal{K}|$. Applying Fano's inequality to the conditional Bayes predictor of $K$ yields
\[
\begin{aligned}
I(X;K\mid S=s)
&\ge\log|\mathcal{K}|-h(P_e(s))\\
&\quad-P_e(s)\log(|\mathcal{K}|-1).
\end{aligned}
\]
The right side equals zero at the random-guessing error $1-1/|\mathcal{K}|$ and is strictly positive below it. The mode-identifiability assumption in Eq.~\eqref{eq:mode_ident} therefore implies $I(X;K\mid S=s)>0$ for each state. Averaging over $S$ gives $I(X;K\mid S)>0$.
\end{proof}

\paragraph{Interpretation.}
As $\epsilon\to0$, the slack $C(\epsilon)$ approaches zero. In the exact-reconstruction limit, the latent retains all corruption-mode information present in the observation. For a positive mismatch budget, Eq.~\eqref{eq:contamination} quantifies the corresponding reduction of the lower bound. This connects the reconstruction target to the nuisance information carried by the latent.

\subsubsection{Proof of Corollary~\ref{cor:foreground}}
\label{app:proof_cor_foreground}
\begin{proof}
Write $F=h_1(S)$ and $Y=h_2(S)$ for deterministic functions of the state. Since $K\perp S$, it is also independent of the pair $(F,Y)$:
\[
p(k,f,y)=p(k)p(f,y).
\]

\par\smallskip\noindent\textbf{Step 1: Nuisance invariance.}
For values with positive probability, marginalizing or conditioning the factorization gives
\[
\begin{aligned}
p(k\mid y)
&=\frac{\sum_f p(k,f,y)}{p(y)}=p(k),\\
p(k\mid f,y)
&=\frac{p(k,f,y)}{p(f,y)}=p(k).
\end{aligned}
\]
Hence $p(k\mid f,y)=p(k\mid y)$, which is conditional independence of $K$ and $F$ given $Y$. Thus $I(F;K\mid Y)=0$, and choosing $Z=F$ establishes exact nuisance invariance.

\par\smallskip\noindent\textbf{Step 2: Task information.}
By the definition $\eta=H(Y\mid F)$,
\[
\begin{aligned}
I(Z;Y)&=I(F;Y)\\
&=H(Y)-H(Y\mid F)\\
&=H(Y)-\eta.
\end{aligned}
\]
The foreground representation therefore satisfies $I(Z;Y)\ge\beta$ with $\beta=H(Y)-\eta$. Nuisance invariance follows from exogeneity, while the amount of retained task information is determined by $\eta$.
\end{proof}

\section{ACO-MoE Architecture Details}
\label{app:network}

Table~\ref{tab:aco_moe_unet} gives the complete ACO-MoE architecture.

\subsection{Dual-Stream MoE U-Net}
\label{app:network_moe}

ACO-MoE is a dual-stream MoE U-Net with a shared four-stage \emph{DoubleConv}
encoder (BatchNorm and SiLU), bottleneck GAP--MLP router, and per-expert RGB/mask
decoders. $N_e$ is a capacity hyperparameter: expert indices receive no
corruption labels or routing targets. Each expert predicts an RGB residual and
foreground/background logits; softmax foreground composition uses a black
background ($-1$).

\begin{table*}[hpbt]
\centering
\caption{ACO-MoE U-Net~\cite{ronneberger2015u} architecture (default base channel $C{=}9$). Spatial sizes shown are for $H,W$ divisible by 16. DoubleConv denotes
$3{\times}3$ Conv$\rightarrow$BN$\rightarrow$SiLU$\rightarrow$ $3{\times}3$ Conv$\rightarrow$BN$\rightarrow$SiLU.}
\label{tab:aco_moe_unet}

\small
\setlength{\tabcolsep}{4pt}
\begin{tabularx}{\linewidth}{@{}l l X l@{}}
\toprule
\textbf{Block} & \textbf{Stage} & \textbf{Configuration} & \textbf{Output size} \\
\midrule
\multirow{6}{*}{Shared encoder}
& Input
& RGB in $[-1,1]$
& $H\times W\times 3$ \\
& Enc-1
& DoubleConv $3\to C$
& $H\times W\times C$ \\
& Enc-2
& MaxPool $2{\times}2$ + DoubleConv $C\to 2C$
& $\frac{H}{2}\times \frac{W}{2}\times 2C$ \\
& Enc-3
& MaxPool $2{\times}2$ + DoubleConv $2C\to 4C$
& $\frac{H}{4}\times \frac{W}{4}\times 4C$ \\
& Enc-4
& MaxPool $2{\times}2$ + DoubleConv $4C\to 8C$
& $\frac{H}{8}\times \frac{W}{8}\times 8C$ \\
& Bottleneck
& MaxPool $2{\times}2$ + DoubleConv $8C\to 16C$
& $\frac{H}{16}\times \frac{W}{16}\times 16C$ \\
\midrule
\multirow{3}{*}{Router}
& GAP
& AdaptiveAvgPool to $1{\times}1$
& $1\times 1\times 16C$ \\
& FC-1
& MLP $16C\to 256$ with SiLU
& $256$ \\
& FC-2
& MLP $256\to N_e$ (routing logits over latent experts)
& $N_e$ \\
\midrule
\multirow{4}{*}{Decoder trunk (per expert)}
& Up-1
& Upsample $\times 2$ + concat (skip Enc-4) + DoubleConv $(16C{+}8C)\to 8C$
& $\frac{H}{8}\times \frac{W}{8}\times 8C$ \\
& Up-2
& Upsample $\times 2$ + concat (skip Enc-3) + DoubleConv $(8C{+}4C)\to 4C$
& $\frac{H}{4}\times \frac{W}{4}\times 4C$ \\
& Up-3
& Upsample $\times 2$ + concat (skip Enc-2) + DoubleConv $(4C{+}2C)\to 2C$
& $\frac{H}{2}\times \frac{W}{2}\times 2C$ \\
& Up-4
& Upsample $\times 2$ + concat (skip Enc-1) + DoubleConv $(2C{+}C)\to C$
& $H\times W\times C$ \\
\midrule
\multirow{2}{*}{RGB head (per expert)}
& Out
& $1{\times}1$ Conv $C\to 3$ (residual)
& $H\times W\times 3$ \\
& Residual add
& $\hat{o}^{(j)}_t=\mathrm{clip}\big(x_t + \tanh(\Delta_t^{(j)}),\,[-1,1]\big)$
& $H\times W\times 3$ \\
\midrule
\multirow{2}{*}{Mask head (per expert)}
& Out
& $1{\times}1$ Conv $C\to 2$ (logits) + softmax
& $H\times W\times 2$ \\
& Foreground prob.
& $m=\mathrm{softmax}(\cdot)[\text{fg}]$
& $H\times W\times 1$ \\
\bottomrule
\end{tabularx}
\end{table*}

\paragraph{Expert aggregation.}
Routing logits are converted to probabilities with softmax. Training uses soft routing weights to mix expert outputs; inference uses hard top-1 routing. The final agent-centric observation is $\hat{o}_t\odot m_t + (-1)(1-m_t)$, where $m_t$ is the foreground probability.

\subsection{Per-Expert Decoder and Router Specifications}
\label{app:network_specs}

Table~\ref{tab:aco_moe_unet} summarizes the architecture. The following details specify its decoders and router.

Each expert uses the shared U-Net encoder and its own RGB/mask decoder branches. With the default base channel $C{=}9$, the encoder stages have widths $[C,2C,4C,8C]$ and the shared bottleneck has width $16C$ (i.e., $[9,18,36,72,144]$ channels). Each DoubleConv applies Conv$\to$BatchNorm$\to$SiLU$\to$Conv$\to$BatchNorm$\to$SiLU, with $2{\times}2$ max-pooling between stages, to an $84\times84\times3$ input normalized to $[-1,1]$.
The RGB decoder of expert $j$ uses four upsampling blocks with skip connections and predicts a residual $\Delta_t^{(j)}$, yielding $\hat{o}_t^{(j)}=\mathrm{clip}(x_t+\tanh(\Delta_t^{(j)}),\,[-1,1])$ at $84\times84\times3$. Soft routing mixes these repaired outputs as in main-paper Eq.~\eqref{eq:soft_routing}.
The mask decoder mirrors the upsampling structure and predicts 2-channel logits, which are converted by softmax into an $84\times84\times1$ foreground map. The router operates directly on the shared $16C$-channel bottleneck: global average pooling produces a $16C$ vector, followed by a two-layer MLP $16C\to256\to N_e$ and softmax routing probabilities. It has no separate strided-convolution encoder. In our default implementation, $C=9$ and $N_e=9$ are fixed capacity choices; $N_e$ is not a corruption-mode label space.

\section{Offline Dataset Generation}
\label{app:data}

\subsection{Setup for Dataset Generation}
\label{app:setup4data}

ACO-MoE is pretrained offline on paired data before RL training. The dataset combines VDCS physical degradations with DMC-GB background distractions. Each tuple $(x_t,o_t,o_t^{\mathrm{cb}},m_t^{\mathrm{ps}})$ contains a corrupted image, a clean image, a clean-background pseudo target, and a pseudo foreground mask for the three losses in Eq.~\ref{eq:total_loss}. The corruption-mode index $k_t$ is retained only as HDF5 metadata and is not passed to the training loader or router.

We construct a unified dataset over DMControl tasks by balancing samples across tasks and perturbation-generation settings. Frozen, closed-source Claude Opus~5.5 and frozen SAM~3.1~\cite{carion2026sam3} generate all foreground pseudo-labels offline. Given task IDs and clean frames, the VLM proposes prompts for the agent and task objects, reviews the SAM masks, and revises the prompts. SAM~3.1 then applies the final prompts to every clean training frame to produce binary pseudo masks $m_t^{\mathrm{ps}}$. Neither model is updated.

The clean-background pseudo target is $o_t^{\mathrm{cb}}=o_t\odot m_t^{\mathrm{ps}}+b\odot(1-m_t^{\mathrm{ps}})$, with $b=-1$ for normalized RGB. The pseudo mask supervises $\mathcal{L}_{\mathrm{mask}}$, and the pseudo target supervises $\mathcal{L}_{\mathrm{final}}$. The target of $\mathcal{L}_{\mathrm{rgb}}$ is the clean frame. These labels train the mask predictor in Fig.~\ref{fig:framework}d. Its predicted masks process the observations used to train DreamerV3 on DMControl and TD-MPC2 on RoboSuite. Each trained controller is then shared across the matched evaluations. RoboSuite uses the same VLM-guided SAM~3.1 labeling procedure as DMControl. At test time, ACO-MoE predicts its mask from corrupted RGB, and neither the VLM nor SAM~3.1 is used.

Figure~\ref{fig:realworld_data} shows the offline pipeline using \texttt{walker\_run} frame~15 and its SAM~3.1 mask. The final VLM prompts are \emph{orange stick figure}, \emph{orange cylinders}, and \emph{orange tubes}. Sec.~\ref{app:vlm_prompts} gives the VLM, the instruction, the review procedure, and the prompts of every task.

For VDCS physical degradations, we apply stochastic intensity jitter
within a $\pm10\%$ band around a base severity to encourage robustness to
severity variation. For DMC-GB background distractions, we render dynamic
backgrounds (e.g., natural video clips) or palette-based color shifts to induce
appearance changes that are independent of the agent dynamics.
We use a $0.9/0.1$ train/validation split. In our main experiments,
images are resized to $84{\times}84$, and we generate 5{,}000 samples per task,
with the task list and degradation set matching the settings described in the main text. Algorithm~\ref{alg:dataset_generation} summarizes the complete data generation pipeline.

\begin{algorithm}[!t]
\caption{Unified DMC Dataset Generation}
\label{alg:dataset_generation}
\begin{algorithmic}[1]
\renewcommand{\algorithmicrequire}{\textbf{Input:}}
\Require tasks $\mathcal{T}$, perturbation generators $\mathcal{K}$, samples $N$, size $(H,W)$, train ratio $\rho$, background path $\mathcal{B}$, seed $s$
\Require frozen, closed-source VLM and frozen SAM~3.1
\For{each task $\tau \in \mathcal{T}$}
    \State infer action\_repeat from benchmark map or policy logs
    \If{policy logdir is provided} \State load Dreamer policy for $\tau$ \EndIf
    \State collect clean training frames for task $\tau$
    \State VLM generates task-oriented prompts from $\tau$ and the frames
    \State SAM~3.1 predicts masks; VLM reviews masks and revises prompts
    \State retain the final task prompts $\mathcal{P}_{\tau}$
    \For{each degradation $k \in \mathcal{K}$}
        \State initialize clean env and/or DMC-GB env for $(\tau, k)$
        \For{$i=1$ to $N$}
            \State render clean image $o_t$
            \State $m_t^{\mathrm{ps}} \leftarrow \operatorname{SAM}_{3.1}(o_t,\mathcal{P}_{\tau})$
            \State $o_t^{\mathrm{cb}} \leftarrow o_t \odot m_t^{\mathrm{ps}} + b\odot(1-m_t^{\mathrm{ps}})$
            \If{$k$ is VDCS} \State apply degradation $\mathcal{D}_k$ with jittered intensity \EndIf
            \If{$k$ is DMC-GB} \State render corrupted frame from DMC-GB env and sync physics \EndIf
            \State write H5 tuple $(x_t,o_t,o_t^{\mathrm{cb}},m_t^{\mathrm{ps}})$; store $k$ only as audit metadata, excluded from the training loader
            \State assign to train/val split with probability $\rho$
            \State step env with random action or Dreamer policy
        \EndFor
    \EndFor
\EndFor
\end{algorithmic}
\end{algorithm}

\subsection{Offline Labels and Robot Training}
\label{app:real_world_data}

The pipeline in Fig.~\ref{fig:realworld_data} supplies foreground supervision for ACO-MoE. On the robot, ACO-MoE restores full camera frames to match the VLA policy input. We generate paired frames from teleoperated demonstrations and train with $\lambda_{\mathrm{mask}}=\lambda_{\mathrm{final}}=0$. For task-weighted training in Sec.~\ref{sec:supp_closed_loop}, masks generated by SAM~3.1 with VLM prompts supply per-pixel loss weights. At deployment, the network takes the camera RGB frames as input.

\begin{figure*}[t]
    \centering
    \includegraphics[width=\textwidth]{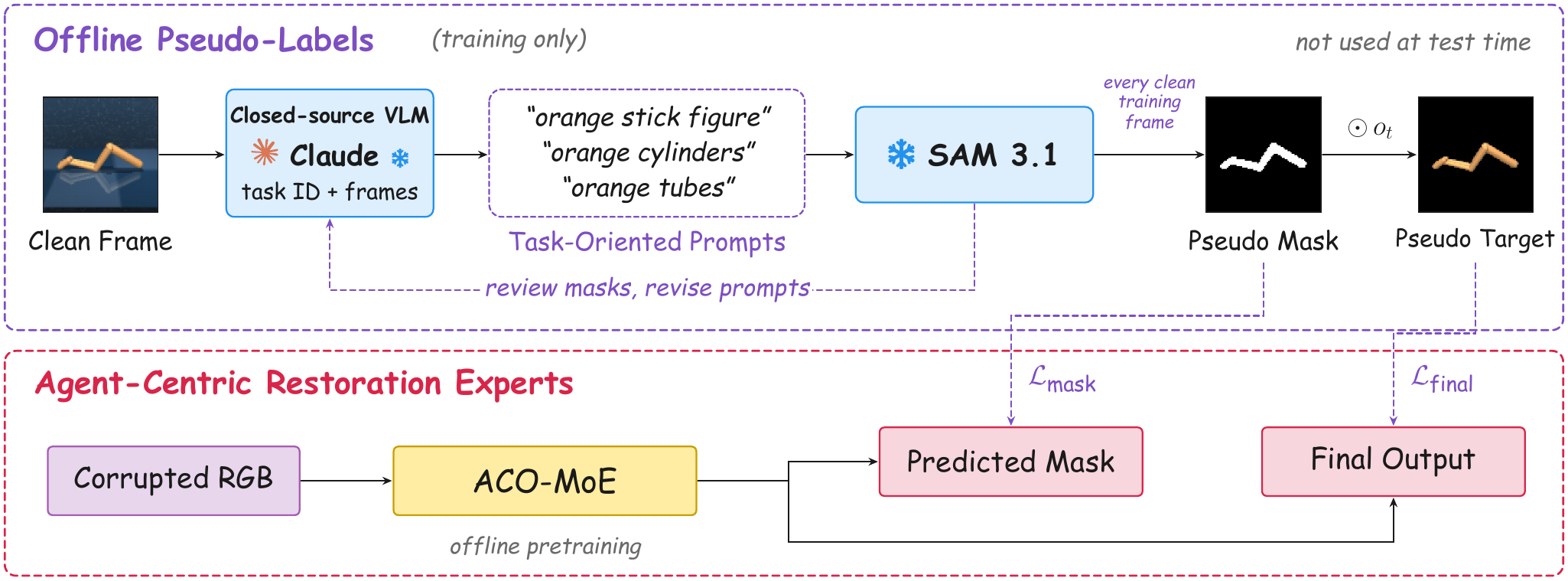}
    \caption{Offline pseudo-label generation and supervision for foreground composition. The frozen, closed-source VLM Claude Opus~5.5 produces task-oriented prompts and reviews SAM~3.1 masks to revise the prompts. Frozen SAM~3.1 labels clean training frames; the resulting pseudo mask and masked clean-frame target supervise $\mathcal{L}_{\mathrm{mask}}$ and $\mathcal{L}_{\mathrm{final}}$. The example uses \texttt{walker\_run} frame~15. Neither model is used at test time.}
    \label{fig:realworld_data}
\end{figure*}

\section{Training Settings and Computational Cost}
\label{app:training}

\begin{table*}[hpbt]
\centering
\caption{\textbf{Latency decomposition (ms/step).}
RTX 5090, $84{\times}84$, batch size 1 (mean$\pm$std).
Preprocess modules: ACO-MoE restorer (4.90M params) vs.\ SAM 2 Tiny~\cite{ravi2024sam2} (38.95M params).
For FTR, preprocessing is amortized with segmentation invoked once every
$T_{\text{sel}}{=}20$ steps and tracking otherwise:
$\frac{1}{20}t_{\text{segment}}+\frac{19}{20}t_{\text{track}}$.}
\label{tab:latency_decomp}

\small
\setlength{\tabcolsep}{6pt}
\begin{tabular}{lccc}
\toprule
\textbf{System} & \textbf{RL policy} & \textbf{Preprocess} & \textbf{Total} \\
\midrule
ACO-MoE (ours) & \tabstat{2.36}{0.003} & \tabstat{20.54}{0.93} & \textbf{\tabstat{22.90}{0.93}} \\
FTR~\cite{wang2025ftr} & \tabstat{2.76}{0.34} & \tabstat{73.48}{1.21} & \textbf{\tabstat{76.24}{1.25}} \\
\bottomrule
\end{tabular}
\end{table*}

\subsection{MoE Pretraining}
\label{app:training_moe}

Each offline sample contains $(x_t,o_t,o_t^{\text{cb}},m_t^{\text{ps}})$ but no
degradation label. The pseudo mask and target are generated by the offline
VLM--SAM~3.1 procedure in Sec.~\ref{app:setup4data}; both generators remain frozen.
Inputs are normalized to $[-1,1]$ and synchronously flipped;
the GAP--MLP router soft-mixes repaired expert images and mask logits. All adapter components
are trained only with $\mathcal{L}_{\text{rgb}}$, $\mathcal{L}_{\text{mask}}$, and
$\mathcal{L}_{\text{final}}$ (Eq.~\ref{eq:total_loss}); video backgrounds
down-weight $\mathcal{L}_{\text{rgb}}$. Algorithm~\ref{alg:moe_training} gives
the full procedure.

\begin{algorithm}[!t]
\caption{Dual-Stream MoE Pretraining}
\label{alg:moe_training}
\begin{algorithmic}[1]
\renewcommand{\algorithmicrequire}{\textbf{Input:}}
\Require paired dataset $\mathcal{D}=\{(x_t,o_t,o^{\text{cb}}_t,m_t^{\text{ps}})\}$ without degradation labels, number of experts $N_e$, loss weights $\lambda_{\text{rgb}},\lambda_{\text{mask}},\lambda_{\text{final}}$
\State Initialize a dual-stream MoE U-Net with a shared encoder, a router, and per-expert RGB/mask decoders
\State Initialize AdamW optimizer
\For{each training step}
    \State Sample a minibatch $(x_t,o_t,o^{\text{cb}}_t,m_t^{\text{ps}})\sim\mathcal{D}$
    \State Encode $x_t$ to bottleneck feature $u_t$ and compute router weights $\pi_t=\mathrm{softmax}(\mathrm{MLP}(\mathrm{GAP}(u_t)))$
    \State Repair each expert image $\hat{o}^{(j)}_t=\mathrm{clip}_{[-1,1]}(x_t+\tanh(\Delta_t^{(j)}))$
    \State Mix $\hat{o}_t=\sum_{j=1}^{N_e}\pi_{t,j}\hat{o}^{(j)}_t$ and $\boldsymbol{\ell}^{\text{mask}}_t=\sum_{j=1}^{N_e}\pi_{t,j}\boldsymbol{\ell}^{\text{mask},(j)}_t$
    \State Compute $m_t=\mathrm{softmax}(\boldsymbol{\ell}^{\text{mask}}_t)_{[\text{fg}]}$ and $\hat{o}^{\text{agent}}_t=\hat{o}_t\odot m_t+b\odot(1-m_t)$
    \State $\mathcal{L}_{\text{rgb}} \leftarrow \|\hat{o}_t - o_t\|_1$
    \State $\mathcal{L}_{\text{mask}} \leftarrow \mathrm{CE}(\boldsymbol{\ell}^{\text{mask}}_t,\, m_t^{\text{ps}})$
    \State $\mathcal{L}_{\text{final}} \leftarrow \|\hat{o}^{\text{agent}}_t - o^{\text{cb}}_t\|_1$
    \State $\mathcal{L} \leftarrow \lambda_{\text{rgb}}\mathcal{L}_{\text{rgb}}
    + \lambda_{\text{mask}}\mathcal{L}_{\text{mask}}
    + \lambda_{\text{final}}\mathcal{L}_{\text{final}}$
    \State Optionally rescale $\mathcal{L}_{\text{rgb}}$ for video-distraction modes
    \State Update all parameters by backpropagation; periodically validate and save checkpoints
\EndFor
\end{algorithmic}
\end{algorithm}

\subsection{Frozen Inference Settings}
\label{app:training_inference}

DreamerV3 and TD-MPC2 train on observations processed by the learned mask
predictor, which is supervised with the data from Fig.~\ref{fig:framework}d.
The resulting controllers remain unchanged across matched evaluations.
At evaluation, the frozen MoE uses hard top-1 routing and passes its
agent-centric output to the frozen model-free or model-based backbone.
Algorithm~\ref{alg:moe_inference} gives the per-step procedure.

\begin{algorithm}[!t]
\caption{Evaluation with Frozen ACO-MoE}
\label{alg:moe_inference}
\begin{algorithmic}[1]
\renewcommand{\algorithmicrequire}{\textbf{Input:}}
\Require env $\mathcal{E}$ with VDCS/DMC-GB, pretrained MoE $G$, RL agent $\pi$
\For{each episode}
    \State reset env, initialize agent state
    \For{each step $t$}
        \State observe corrupted frame $x_t$
        \State compute $(\hat{o}_t, m_t) = G(x_t)$ with top-1 routing
        \State compose agent-centric frame $o^{\text{agent}}_t = \hat{o}_t \odot m_t + (-1)(1-m_t)$
        \State action $a_t \leftarrow \pi(o^{\text{agent}}_t)$
        \State step env and accumulate return
    \EndFor
\EndFor
\end{algorithmic}
\end{algorithm}

\subsection{Hyperparameter Configuration}
\label{app:training_hyperparams}

MoE pretraining uses AdamW with learning rate $10^{-4}$, weight decay $10^{-4}$, and batch size 1024, for 15,000 iterations over the unified paired dataset. The total loss combines RGB restoration ($\ell_1$), mask segmentation (2-class cross-entropy against the pseudo masks $m_t^{\text{ps}}$), and agent-centric reconstruction ($\ell_1$ against $o^{\text{cb}}_t$), with weights $\lambda_{\text{rgb}}{:}\lambda_{\text{mask}}{:}\lambda_{\text{final}}=1{:}1{:}1$. The shared encoder, router MLP, and all $N_e$ per-expert decoders are optimized jointly under this single objective, with no auxiliary degradation-classification loss or expert-label supervision. For DreamerV3, the RSSM uses 1024 deterministic units with $32\times 32$ categorical stochastic units, imagination horizon $T=15$, discount $\gamma=0.997$, batch size 16 sequences of 64 timesteps, and 1M total environment steps. For frozen-policy RL evaluation, DMControl (DreamerV3) uses \texttt{steps=1e6}, \texttt{action\_repeat=2}, input size $84\times 84$, and episode length $1000$. RoboSuite (TD-MPC2) uses \texttt{image\_size=168}, \texttt{frame\_stack=3}, \texttt{action\_repeat=2}, and episode length $200$.
Both RL backbones require approximately 24 hours for each task/seed on a single NVIDIA GeForce RTX 5090.

\subsection{Computational Overhead}
\label{app:training_cost}

Table~\ref{tab:latency_decomp} reports warmed-up RTX 5090 latency at batch size
1. ACO-MoE totals $22.90\pm0.93$ ms/step; FTR totals $76.24\pm1.25$ ms/step
because SAM2-tiny preprocessing dominates its amortized cost.

\subsection{Existing Assets and Licenses}
\label{app:assets_licenses}

We use existing public assets only for non-commercial academic research and evaluation. The simulation benchmarks, RL backbones, and image-restoration baselines are credited through citations to their original papers and repositories. DMControl is released under Apache-2.0. DreamerV3, TD-MPC2, Restormer, AdaIR, and NAFNet are released under open-source licenses stated in their repositories. PromptIR is used under its academic/non-commercial research terms. We do not redistribute third-party datasets, models, or code beyond their stated terms of use.

\section{Additional Experimental Results and Ablations}
\label{app:results}
\noindent\textbf{Task codes.} CP-S for cartpole swingup, F-S for finger spin, F-TH for finger turn hard, H-S for hopper stand, H-H for hopper hop, C-R for cheetah run, W-S for walker stand, W-W for walker walk, and W-R for walker run.

\newpage
\paragraph{Reporting references.}
The main VDCS table, Table~\ref{tab:clean_recovery}, and
Fig.~\ref{fig:protocol_ablation} share the primary VDCS reference results.
Tables~\ref{tab:video_hard_results} and~\ref{tab:color_hard_results} are the
DMC-GB benchmark references. Component ablations in
Fig.~\ref{fig:ablation_bars} and Tables~\ref{tab:video_hard_fixed}
and~\ref{tab:color_hard_fixed} report changes relative to their own displayed
full-model scores, which differ from these benchmark references.
\subsection{Clean Baseline Recovery Analysis}
\label{app:clean_recovery}
Table~\ref{tab:clean_recovery} compares the clean full-background baseline with
DreamerV3 trained on observations processed by the learned mask predictor and evaluated with ACO-MoE under VDCS. The two clean baselines differ by 0.3\%
(767 versus 769). ACO-MoE retains 94.5\% of the clean foreground return on
average across tasks. \texttt{cartpole\_swingup} reaches 98.2\%,
\texttt{walker\_walk} reaches 97.1\%, and \texttt{finger\_spin} reaches 90.1\%.
\onecolumn
\begin{figure}[H]
\centering
\includegraphics[width=0.94\linewidth]{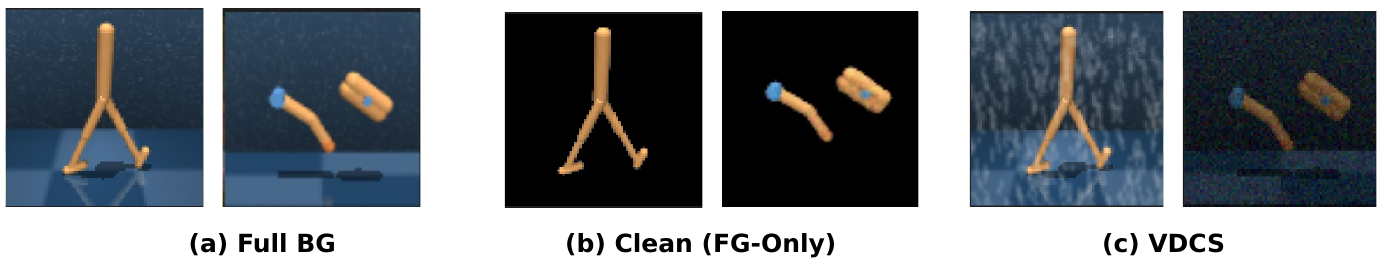}
\caption{Three observation settings: (a)~\textbf{Full BG}: standard DMC rendering; (b)~\textbf{Clean (FG-only)}: an uncorrupted frame processed by the learned mask predictor; (c)~\textbf{VDCS}: corrupted RGB before ACO-MoE restoration.}
\label{fig:bg_clean_vdcs}
\end{figure}

\begin{table}[H]
\centering
\caption{Clean baselines and ACO-MoE under VDCS. FG-only uses DreamerV3 trained on observations processed by the mask predictor learned from the data in Fig.~\ref{fig:framework}d. ACO-MoE uses the same controller. Retained averages task-wise ratios to clean return, yielding 94.5\%. The ratio of mean returns is 730/769, or 94.9\%. Average reports the mean and standard deviation across tasks. The clean means differ by 0.3\%.}
\label{tab:clean_recovery}

\small
\renewcommand{\arraystretch}{1.0}
\setlength{\extrarowheight}{0pt}
\setlength{\tabcolsep}{4pt}
\begin{tabular*}{\textwidth}{@{\extracolsep{\fill}}lcccc@{}}
\toprule
\multicolumn{1}{c}{\textbf{Task}} &
\multicolumn{1}{c}{DreamerV3~\cite{hafner2023mastering}} &
\multicolumn{1}{c}{DreamerV3~\cite{hafner2023mastering}} &
\multicolumn{1}{c}{\textbf{ACO-MoE}} &
\multicolumn{1}{c}{Retained} \\
\multicolumn{1}{c}{} &
\multicolumn{1}{c}{\textbf{(Full BG)}} &
\multicolumn{1}{c}{\textbf{(FG-only)}} &
\multicolumn{1}{c}{\textbf{(VDCS)}} &
\multicolumn{1}{c}{(\%)} \\
\midrule
CP-S & 852 & 878 & 862 & 98.2 \\
F-S & 558 & 543 & 489 & 90.1 \\
F-TH & 793 & 929 & 878 & 94.5 \\
H-S & 952 & 948 & 902 & 95.1 \\
H-H & 388 & 359 & 333 & 92.8 \\
C-R & 848 & 749 & 699 & 93.3 \\
W-W & 963 & 974 & 946 & 97.1 \\
W-R & 782 & 772 & 731 & 94.7 \\
\midrule
\textbf{Average} & \tabstat{767.0}{185.3} & \tabstat{769.0}{203.4} & \tabstat{730.0}{203.8} & \textbf{94.5} \\
\bottomrule
\end{tabular*}
\end{table}
\subsection{DMC-GB Generalization: Per-Task Results}
\label{app:results_dmcgb}

Tables~\ref{tab:video_hard_results} and~\ref{tab:color_hard_results} give
per-task DMC-GB results with five seeds and ten episodes per seed, without any
test-time controller or adapter update. Gains on \texttt{cheetah\_run} are +71\% for video-hard and +163\% for color-hard;
\texttt{walker\_walk} gains +2\% and +27\%, respectively.

\begin{table}[H]
\centering
\caption{Performance comparison on Video Hard perturbations. Results show mean $\pm$ std over evaluation episodes. \ranklegend{}.}
\label{tab:video_hard_results}

\small
\setlength{\extrarowheight}{1pt}
\setlength{\tabcolsep}{0.4pt}
\begin{tabularx}{\textwidth}{@{}l*{10}{>{\hsize=.98\hsize\centering\arraybackslash}X}>{\hsize=1.2\hsize\centering\arraybackslash}X>{\centering\arraybackslash}X@{}}
\toprule
\multicolumn{1}{c}{\textbf{Task}} & \multicolumn{1}{c}{PAD} & \multicolumn{1}{c}{Q$^2$} & \multicolumn{1}{c}{SimGRL} & \multicolumn{1}{c}{FTR} & \multicolumn{1}{c}{DrQ} & \multicolumn{1}{c}{SODA} & \multicolumn{1}{c}{SVEA} & \multicolumn{1}{c}{SRM} & \multicolumn{1}{c}{SGQN} & \multicolumn{1}{c}{SMG} & \textbf{ACO-MoE} & \multicolumn{1}{c}{$\bigtriangleup$} \\
\multicolumn{1}{c}{} & \multicolumn{1}{c}{\cite{hansen2021pad}} & \multicolumn{1}{c}{\cite{liao2023q2}} & \multicolumn{1}{c}{\cite{song2024simgrl}} & \multicolumn{1}{c}{\cite{wang2025ftr}} & \multicolumn{1}{c}{\cite{kostrikov2020image}} & \multicolumn{1}{c}{\cite{hansen2021generalization}} & \multicolumn{1}{c}{\cite{hansen2021stabilizing}} & \multicolumn{1}{c}{\cite{huang2022spectrum}} & \multicolumn{1}{c}{\cite{bertoin2022lookwhereyou}} & \multicolumn{1}{c}{\cite{zhang2024focusonwhat}} & \textbf{(Ours)} & \\
\midrule
CP-S & \tabstack{79}{9} & \secondcell{\tabstack{807}{84}} & \tabstack{579}{306} & \tabstack{646}{169} & \tabstack{168}{35} & \tabstack{346}{59} & \tabstack{510}{177} & \tabstack{254}{69} & \tabstack{599}{112} & \thirdcell{\tabstack{764}{32}} & \bestcell{\tabstack{845}{7}} & \makecell{+38\\{\tabsecondary 5\%}} \\
\midrule
F-S & \tabstack{1}{1} & \tabstack{696}{122} & \tabstack{352}{274} & \secondcell{\tabstack{904}{71}} & \tabstack{54}{44} & \tabstack{310}{72} & \tabstack{353}{71} & \tabstack{131}{89} & \thirdcell{\tabstack{710}{159}} & \bestcell{\tabstack{910}{61}} & \tabstack{490}{19} & \makecell{$-420$\\{\tabsecondary $-46\%$}} \\
\midrule
W-S & \tabstack{802}{41} & \tabstack{476}{45} & \thirdcell{\tabstack{932}{17}} & \tabstack{912}{51} & \tabstack{278}{79} & \tabstack{406}{68} & \tabstack{814}{57} & \tabstack{558}{139} & \tabstack{870}{78} & \secondcell{\tabstack{955}{9}} & \bestcell{\tabstack{971}{11}} & \makecell{+16\\{\tabsecondary 2\%}} \\
\midrule
W-W & \tabstack{26}{11} & \tabstack{325}{62} & \tabstack{734}{42} & \secondcell{\tabstack{900}{45}} & \tabstack{110}{33} & \tabstack{175}{31} & \tabstack{348}{80} & \tabstack{165}{99} & \tabstack{634}{136} & \thirdcell{\tabstack{814}{51}} & \bestcell{\tabstack{920}{52}} & \makecell{+20\\{\tabsecondary 2\%}} \\
\midrule
C-R & \tabstack{7}{7} & \tabstack{222}{50} & \tabstack{221}{111} & \secondcell{\tabstack{446}{65}} & \tabstack{38}{26} & \tabstack{118}{40} & \tabstack{105}{13} & \tabstack{87}{24} & \tabstack{135}{44} & \thirdcell{\tabstack{303}{46}} & \bestcell{\tabstack{761}{43}} & \makecell{+315\\{\tabsecondary 71\%}} \\
\midrule
\textbf{Average} & \tabstack{183.0}{347.4} & \tabstack{505.2}{245.5} & \tabstack{563.6}{286.0} & \secondcell{\tabstack{761.6}{209.2}} & \tabstack{129.6}{97.5} & \tabstack{271.0}{120.4} & \tabstack{426.0}{260.8} & \tabstack{239.0}{188.6} & \tabstack{589.6}{274.7} & \thirdcell{\tabstack{749.2}{260.6}} & \bestcell{\tabstack{797.4}{189.2}} & \makecell{+35.8\\{\tabsecondary 5\%}} \\
\bottomrule
\end{tabularx}
\end{table}

\begin{table}[H]
\centering
\caption{Performance comparison on Color Hard perturbations. Results show mean $\pm$ std over evaluation episodes. \ranklegend{}.}
\label{tab:color_hard_results}

\small
\setlength{\extrarowheight}{3pt}
\setlength{\tabcolsep}{4pt}
\begin{tabularx}{\textwidth}{@{}l*{9}{>{\centering\arraybackslash}X}@{}}
\toprule
\multicolumn{1}{c}{\textbf{Task}} & \multicolumn{1}{c}{SAC} & \multicolumn{1}{c}{DrQ} & \multicolumn{1}{c}{SODA} & \multicolumn{1}{c}{SVEA} & \multicolumn{1}{c}{SRM} & \multicolumn{1}{c}{SGQN} & \multicolumn{1}{c}{SMG} & \multicolumn{1}{c}{\textbf{ACO-MoE}} & \multicolumn{1}{c}{$\bigtriangleup$} \\
\multicolumn{1}{c}{} & \multicolumn{1}{c}{\cite{haarnoja2018soft}} & \multicolumn{1}{c}{\cite{kostrikov2020image}} & \multicolumn{1}{c}{\cite{hansen2021generalization}} & \multicolumn{1}{c}{\cite{hansen2021stabilizing}} & \multicolumn{1}{c}{\cite{huang2022spectrum}} & \multicolumn{1}{c}{\cite{bertoin2022lookwhereyou}} & \multicolumn{1}{c}{\cite{zhang2024focusonwhat}} & \multicolumn{1}{c}{\textbf{(Ours)}} & \\
\midrule
CP-S & \tabstack{184}{26} & \tabstack{717}{133} & \tabstack{585}{66} & \secondcell{\tabstack{752}{86}} & \secondcell{\tabstack{752}{103}} & \tabstack{636}{110} & \tabstack{726}{62} & \bestcell{\tabstack{832}{33}} & \makecell{+80\\{\tabsecondary 11\%}} \\
\midrule
F-S & \tabstack{271}{23} & \tabstack{655}{214} & \tabstack{663}{106} & \bestcell{\tabstack{868}{74}} & \thirdcell{\tabstack{834}{90}} & \tabstack{700}{219} & \secondcell{\tabstack{841}{113}} & \tabstack{488}{11} & \makecell{$-380$\\{\tabsecondary $-44\%$}} \\
\midrule
W-S & \tabstack{526}{259} & \tabstack{769}{182} & \tabstack{719}{138} & \tabstack{799}{118} & \thirdcell{\tabstack{807}{128}} & \tabstack{788}{114} & \secondcell{\tabstack{878}{70}} & \bestcell{\tabstack{981}{11}} & \makecell{+103\\{\tabsecondary 12\%}} \\
\midrule
W-W & \tabstack{379}{37} & \tabstack{456}{192} & \tabstack{396}{78} & \tabstack{571}{134} & \tabstack{483}{123} & \thirdcell{\tabstack{632}{176}} & \secondcell{\tabstack{739}{31}} & \bestcell{\tabstack{942}{37}} & \makecell{+203\\{\tabsecondary 27\%}} \\
\midrule
C-R & \tabstack{208}{54} & \tabstack{147}{80} & \tabstack{199}{38} & \thirdcell{\tabstack{238}{69}} & \tabstack{203}{30} & \tabstack{210}{18} & \secondcell{\tabstack{299}{22}} & \bestcell{\tabstack{787}{25}} & \makecell{+488\\{\tabsecondary 163\%}} \\
\midrule
\textbf{Average} & \tabstack{313.6}{140.7} & \tabstack{548.8}{254.0} & \tabstack{512.4}{213.5} & \thirdcell{\tabstack{645.6}{253.0}} & \tabstack{615.8}{269.6} & \tabstack{593.2}{223.3} & \secondcell{\tabstack{696.6}{231.6}} & \bestcell{\tabstack{806.0}{194.5}} & \makecell{+109.4\\{\tabsecondary 16\%}} \\
\bottomrule
\end{tabularx}
\end{table}

\subsection{RoboSuite Results}
\label{app:robosuite}

To test transfer beyond DMControl, we pair the same adapter with TD-MPC2~\cite{hansen2024tdmpc2} on the contact-rich RoboSuite tasks \texttt{Door} and \texttt{Lift}. The mask predictor is trained with frames labeled by the same VLM-guided SAM~3.1 pipeline as DMControl and processes the observations used to train TD-MPC2. Fig.~\ref{fig:vdcs_video_comparison} shows gains over FTR of 337\% on \texttt{Door} and 163\% on \texttt{Lift} under VDCS, and gains of 32\% and 71\% under Video Hard. Fig.~\ref{fig:robosuite_convergence_curves} gives the corresponding learning curves.

\begin{figure}[H]
\centering
\includegraphics[width=\linewidth]{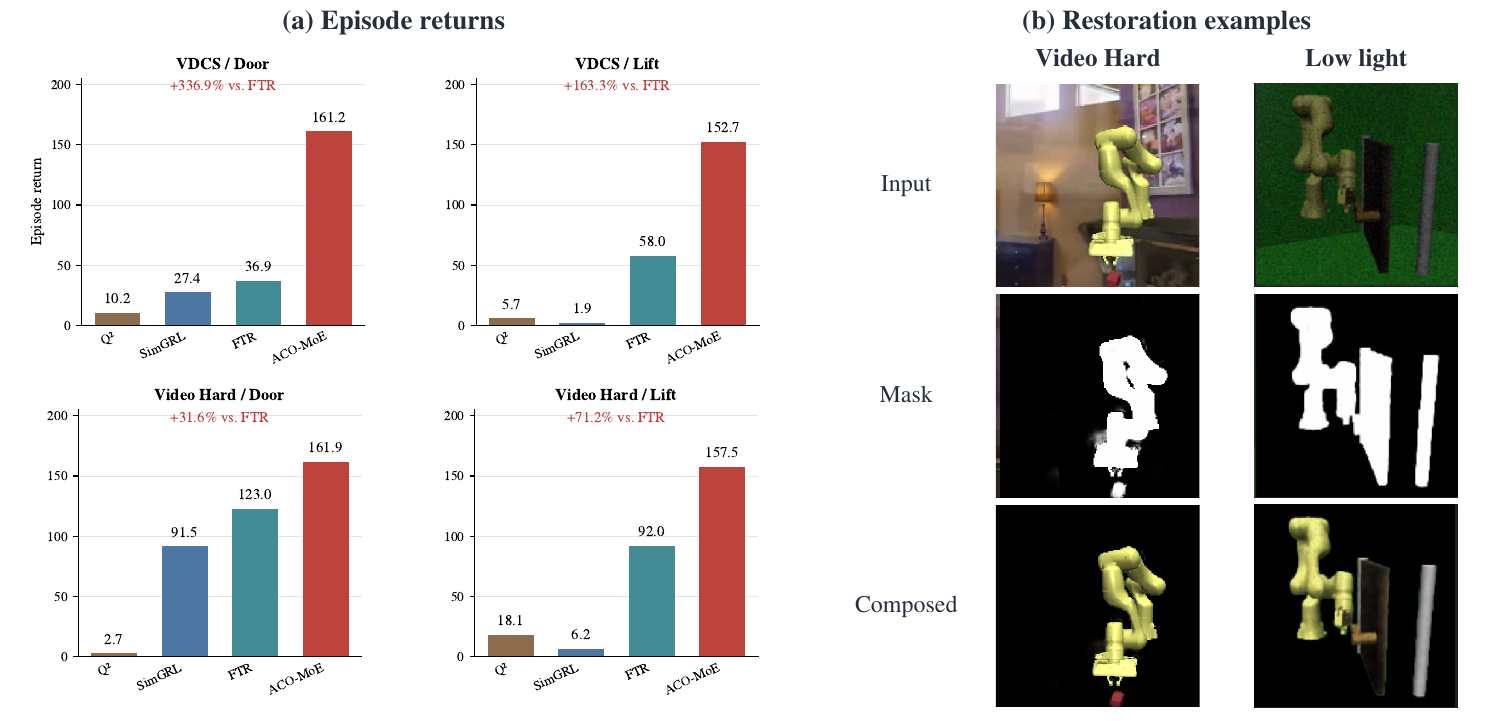}
\caption{\textbf{RoboSuite results with TD-MPC2}~\cite{hansen2024tdmpc2}. \textbf{Left:} Episode returns for Q$^2$~\cite{liao2023q2}, SimGRL~\cite{song2024simgrl}, FTR~\cite{wang2025ftr}, and ACO-MoE. Relative to FTR, ACO-MoE improves Door and Lift returns by 337\% and 163\% under VDCS, and by 32\% and 71\% under Video Hard. \textbf{Right:} Rows show the input, mask, and composed observation. Columns show Video Hard and low light.}
\label{fig:vdcs_video_comparison}
\end{figure}

\subsection{Learning Curves and Qualitative Visualizations}
\label{app:results_curves}

Figure~\ref{fig:convergence_curves} shows the learning curves of all methods
on the eight DMControl tasks under VDCS Markov switching over 1M environment steps.
Figure~\ref{fig:robosuite_convergence_curves} shows the RoboSuite learning
curves alongside qualitative restoration examples.
Figures~\ref{fig:vdcs_dmcgb_comprehensive} and
\ref{fig:robosuite_degra_demo} provide comprehensive visual examples
of each degradation type across tasks and environments.

\begin{figure}[H]
    \centering
    \includegraphics[width=\textwidth,trim=13 25 0 6,clip]{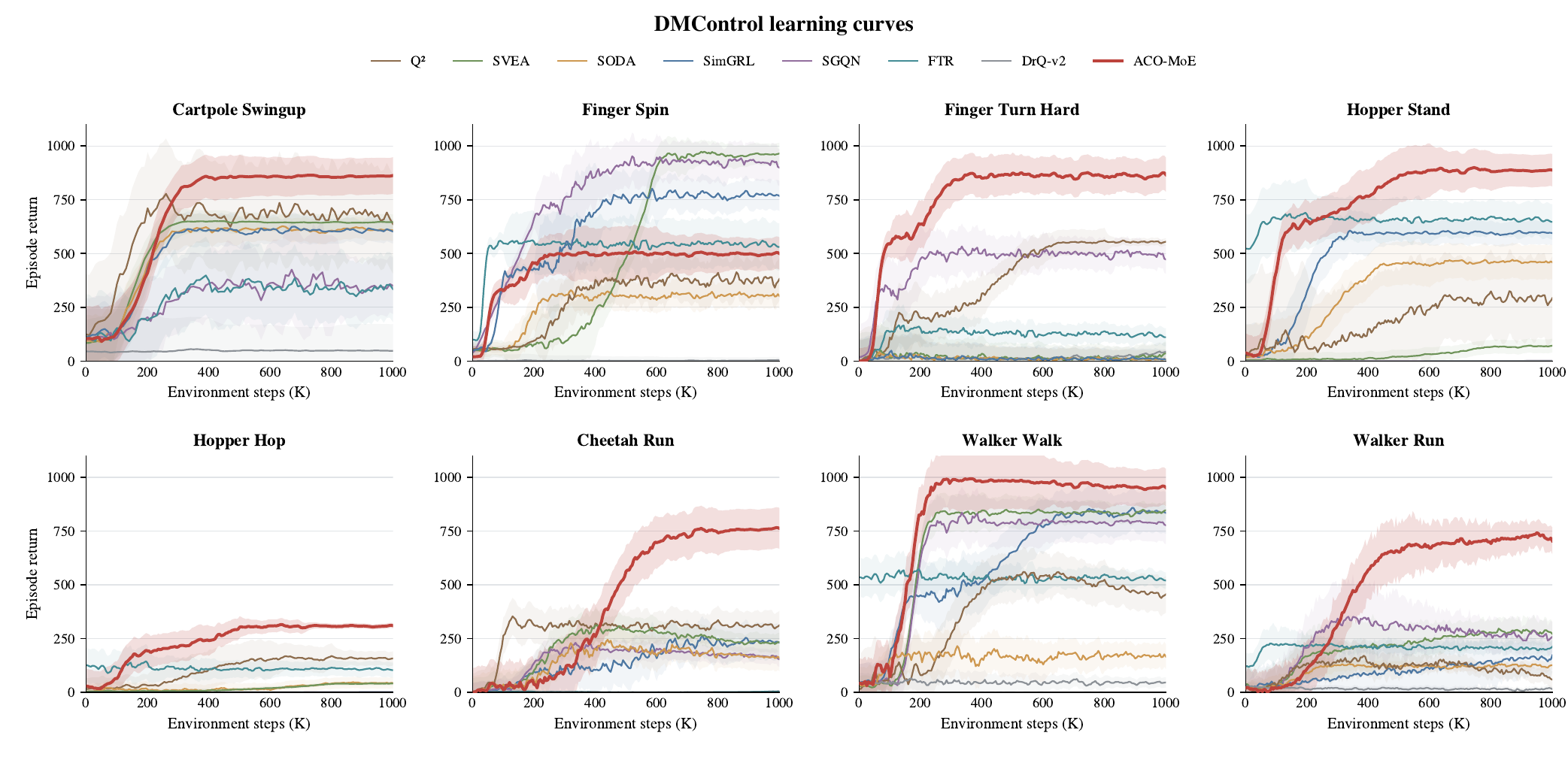}
    \caption{DMControl learning curves under VDCS Markov-switching perturbations. Episode return of every method over 1M environment steps on the eight tasks.}
    \label{fig:convergence_curves}
\end{figure}

\begin{figure}[H]
    \centering

    \begin{minipage}[t]{0.48\textwidth}
        \centering
        \includegraphics[width=\linewidth,trim=16 20 15 3,clip]{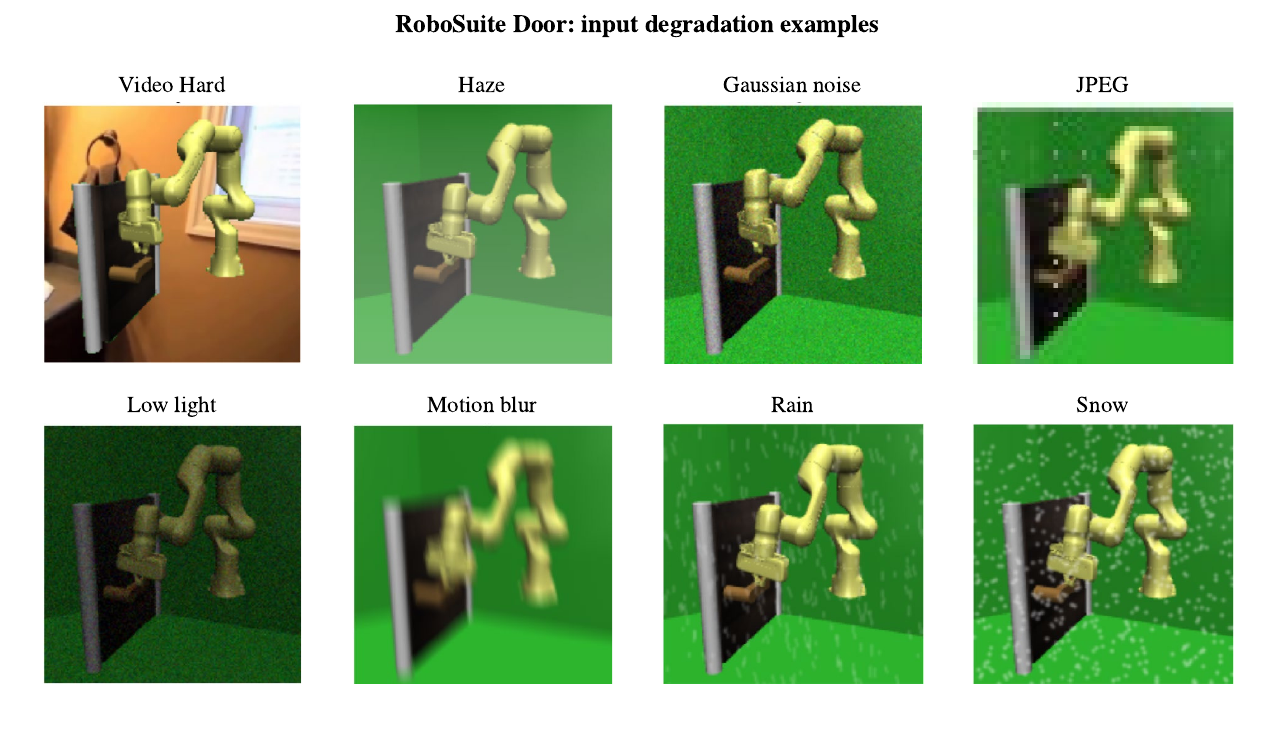}
        \vspace{0.4em}

        (a) \textbf{RoboSuite Door.}
    \end{minipage}
    \hfill
    \begin{minipage}[t]{0.48\textwidth}
        \centering
        \includegraphics[width=\linewidth,trim=16 20 16 3,clip]{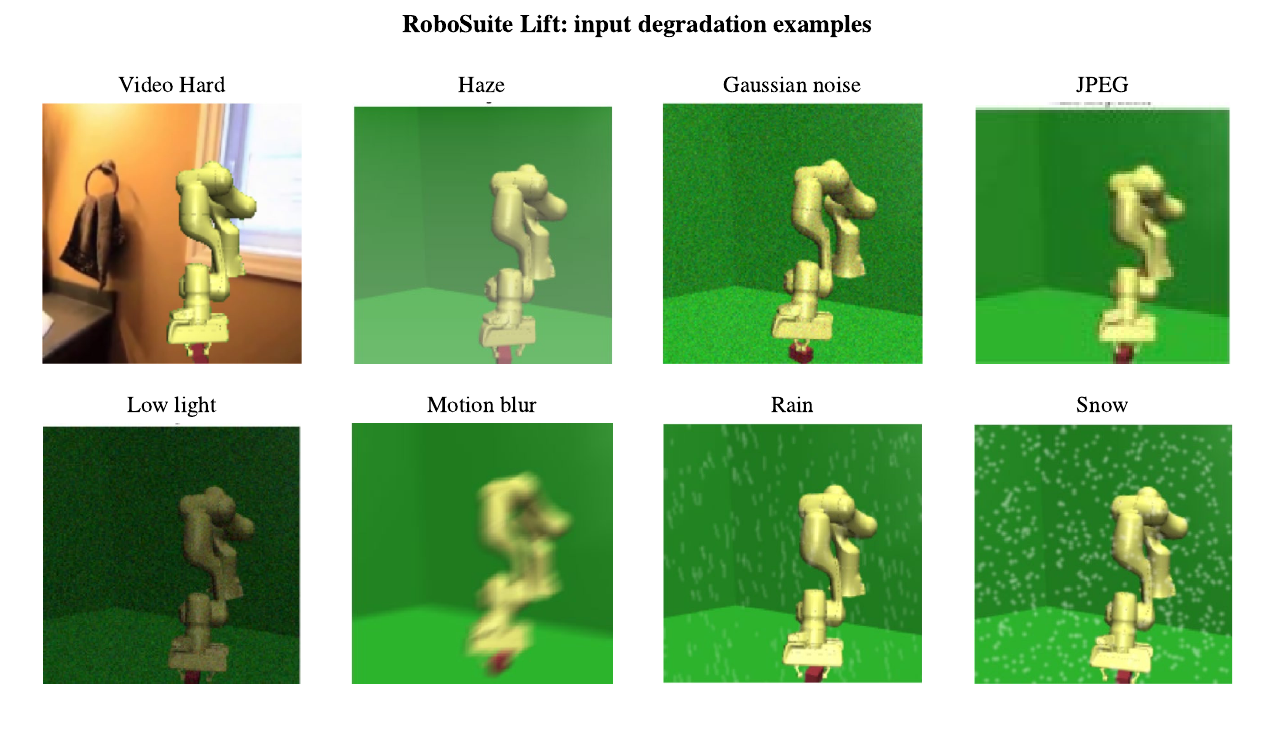}
        \vspace{0.4em}

        (b) \textbf{RoboSuite Lift.}
    \end{minipage}

    \caption{Visualization of RoboSuite manipulation observations under \textbf{VDCS} Markov-switching corruptions and \textbf{Video-Hard} background replacement. We apply 7 frame-wise visual corruptions (haze, Gaussian noise, JPEG compression, low light, motion blur, rain, snow) while replacing the original simulator background with dynamic video content (\textit{Video Hard}). Corruption may switch each timestep according to the Markov transition matrix, creating dynamic distracted observations without changing the underlying task dynamics.}
    \label{fig:robosuite_degra_demo}
\end{figure}

\begin{figure}[H]
    \centering
    \includegraphics[width=0.72\textwidth,trim=26 22 1 5,clip]{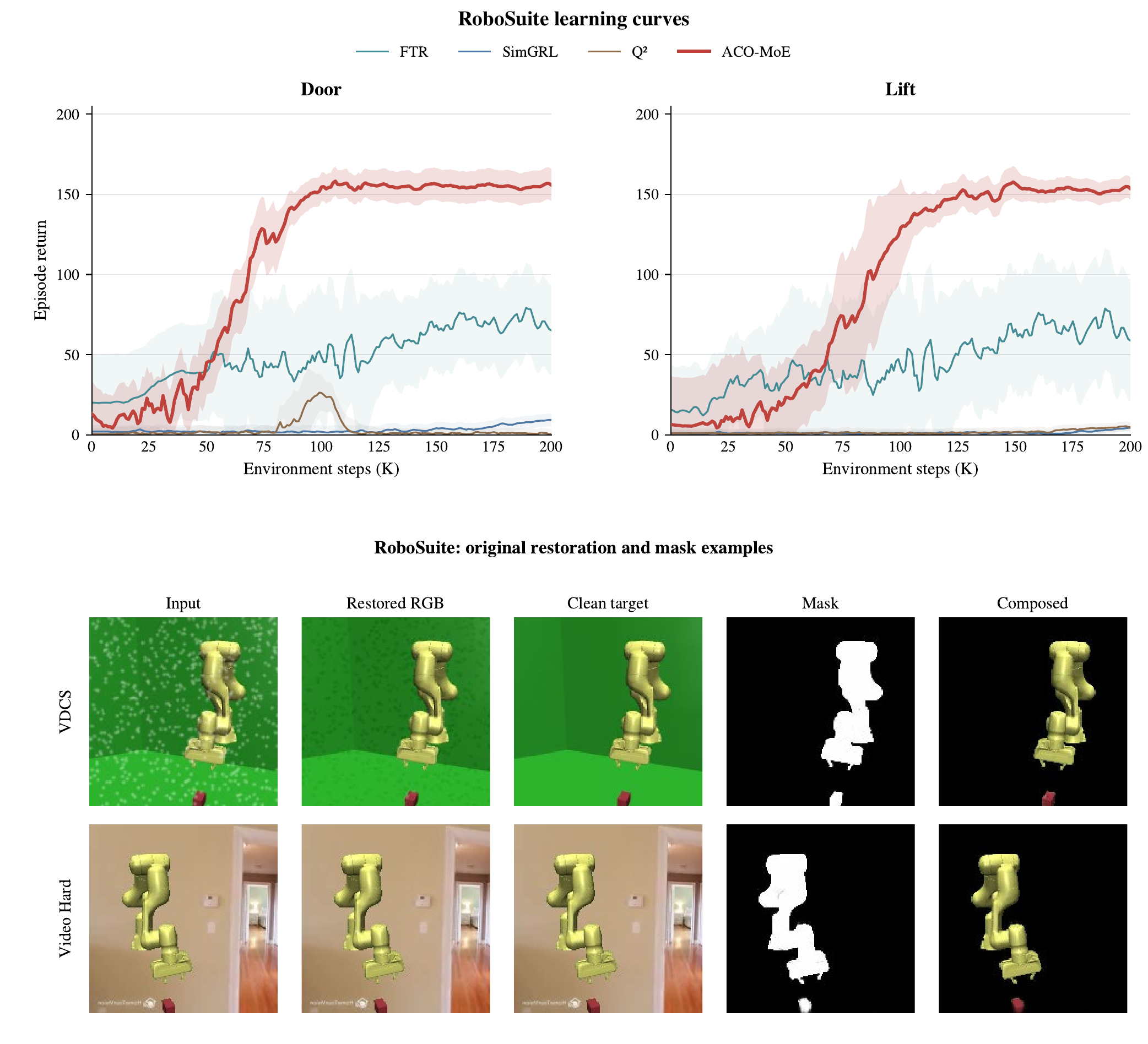}
    \caption{\textbf{RoboSuite trajectories and restoration examples.} \textbf{(a)} Episode return of FTR, SimGRL, Q$^2$, and ACO-MoE on \textit{Door} and \textit{Lift} over 200K environment steps. \textbf{(b) and (c)} Restoration under VDCS Markov switching and Video Hard, respectively. From left to right: corrupted input $x_t$, restored RGB $\hat{o}_t$, clean target $o_t$, predicted mask $m_t$, and agent-centric observation $\tilde{x}_t$.}
    \label{fig:robosuite_convergence_curves}
\end{figure}

\begin{figure}[H]
\centering
\includegraphics[width=0.48\textwidth,trim=14 20 14 3,clip]{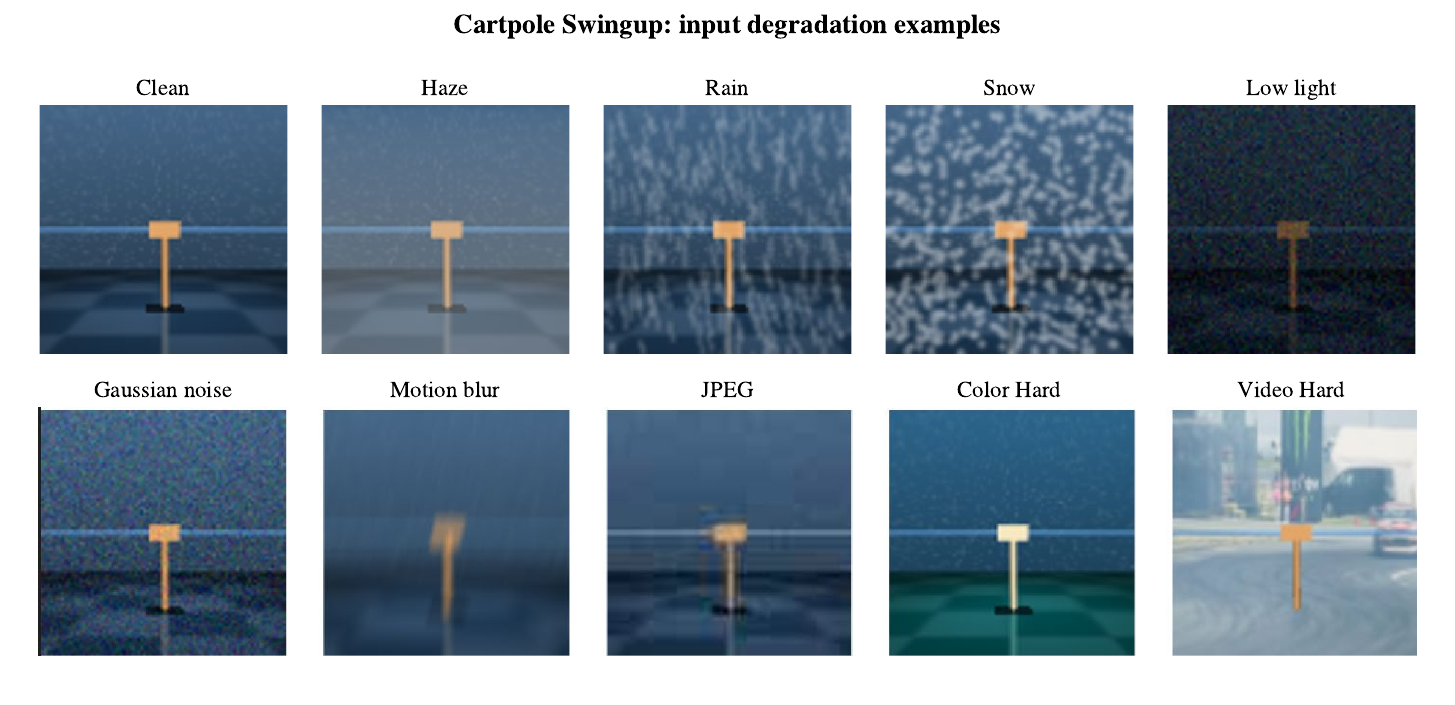}
\hfill
\includegraphics[width=0.48\textwidth,trim=14 20 15 3,clip]{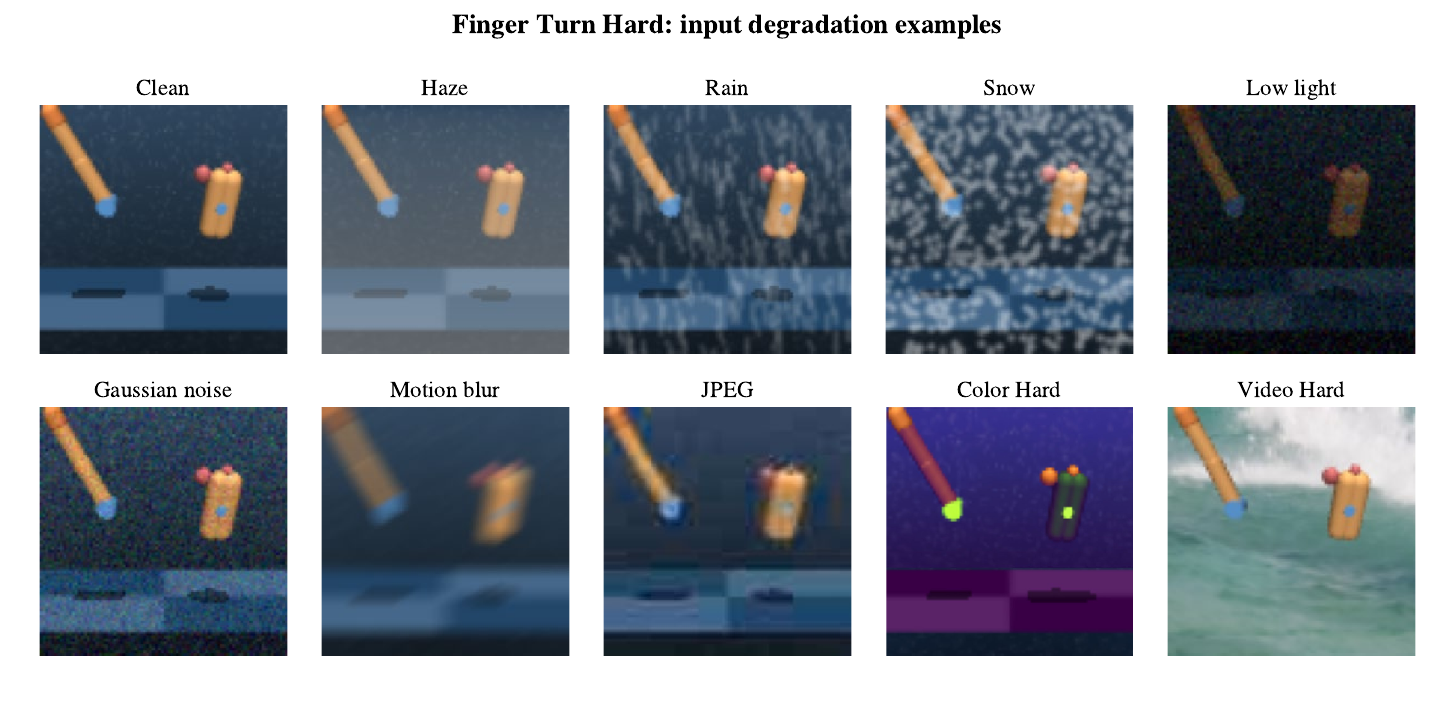}

\includegraphics[width=0.48\textwidth,trim=14 20 15 3,clip]{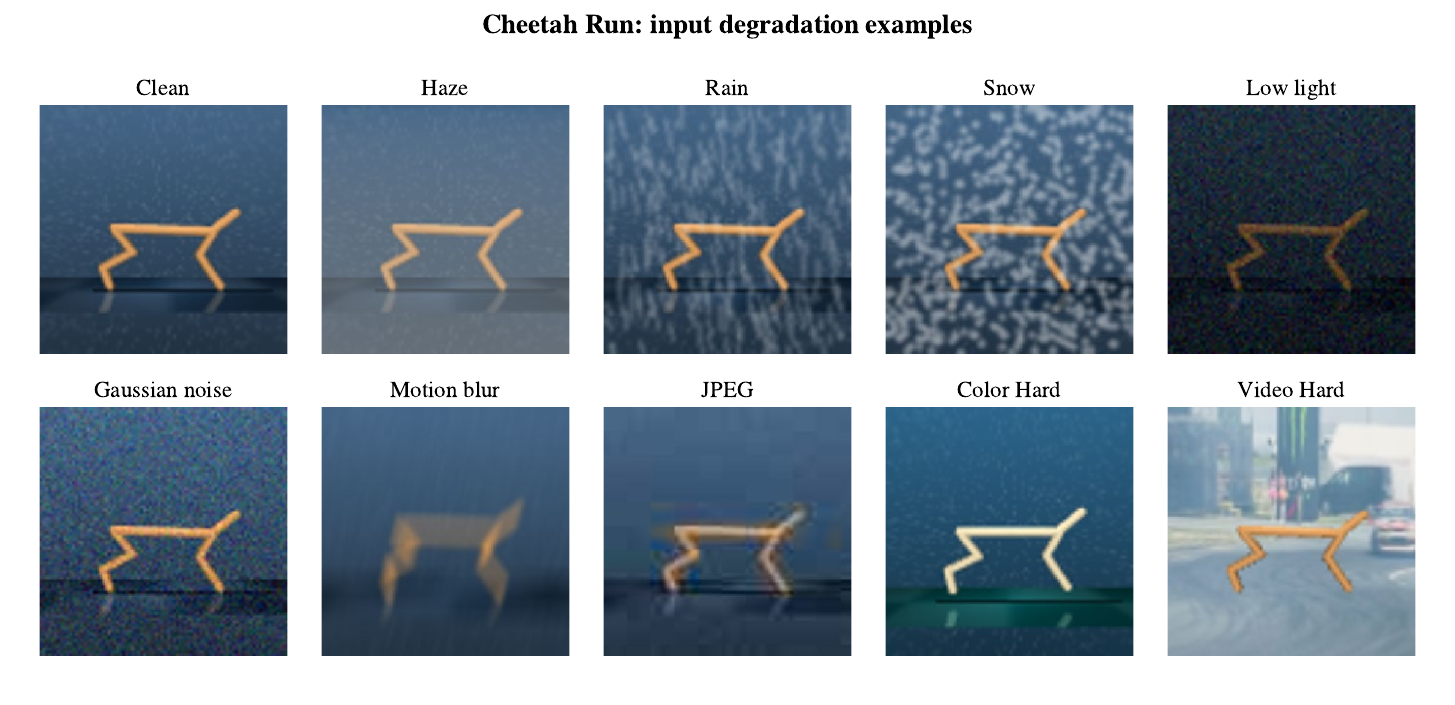}
\hfill
\includegraphics[width=0.48\textwidth,trim=14 20 15 3,clip]{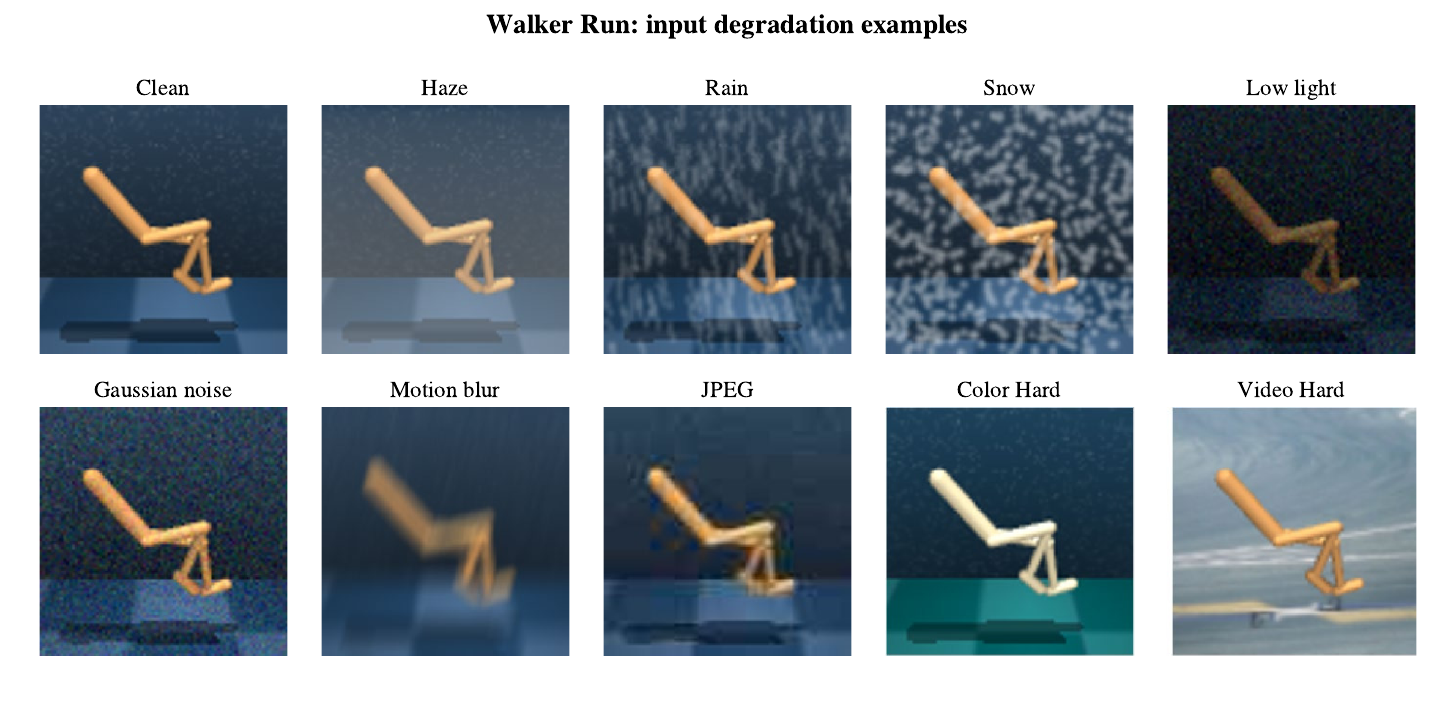}
\caption{Comprehensive degradation visualization across VDCS and DMC-GB benchmarks. (a) Cartpole Swingup. (b) Finger Turn Hard. (c) Cheetah Run. (d) Walker Run.}
\label{fig:vdcs_dmcgb_comprehensive}
\label{fig:vdcs_cartpole}
\label{fig:vdcs_finger}
\label{fig:vdcs_cheetah}
\label{fig:vdcs_walker}
\end{figure}

\subsection{Backbone Generalization Analysis}
\label{app:backbone_generalization}

We evaluate the same frozen ACO-MoE module with DrQ-v2~\cite{yarats2021mastering}
under VDCS Markov-temporal perturbations. Figure~\ref{fig:backbone_generalization}
reports \texttt{finger\_spin} and the 8-task average. On \texttt{finger\_spin},
DrQ-v2 scores 912 on clean observations and 861 with ACO-MoE, retaining 94.4\% of clean return.
Across all eight tasks, ACO-MoE retains 90.1\% of the clean return.

\begin{figure}[H]
    \centering
    \includegraphics[width=0.46\textwidth,trim=20 34 3 2,clip]{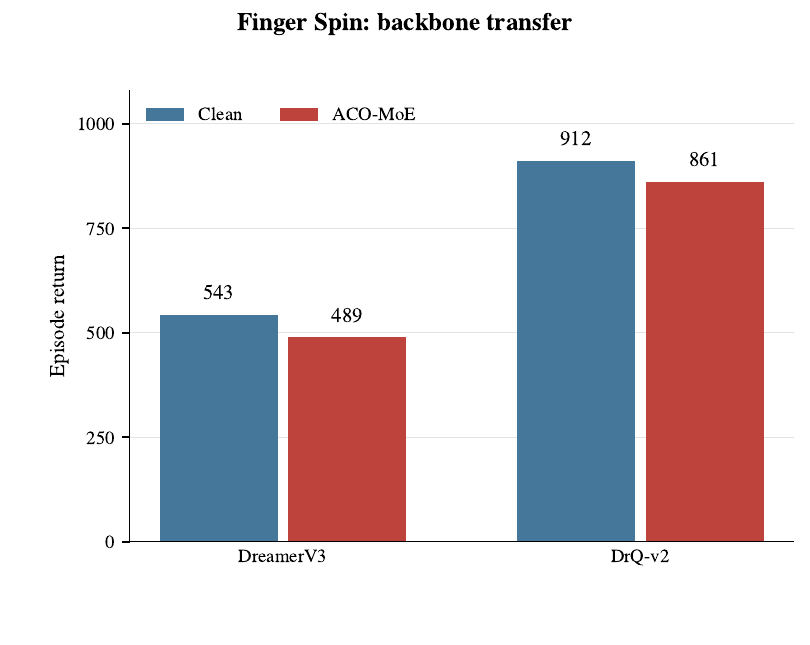}\hfill
    \includegraphics[width=0.49\textwidth,trim=24 34 3 2,clip]{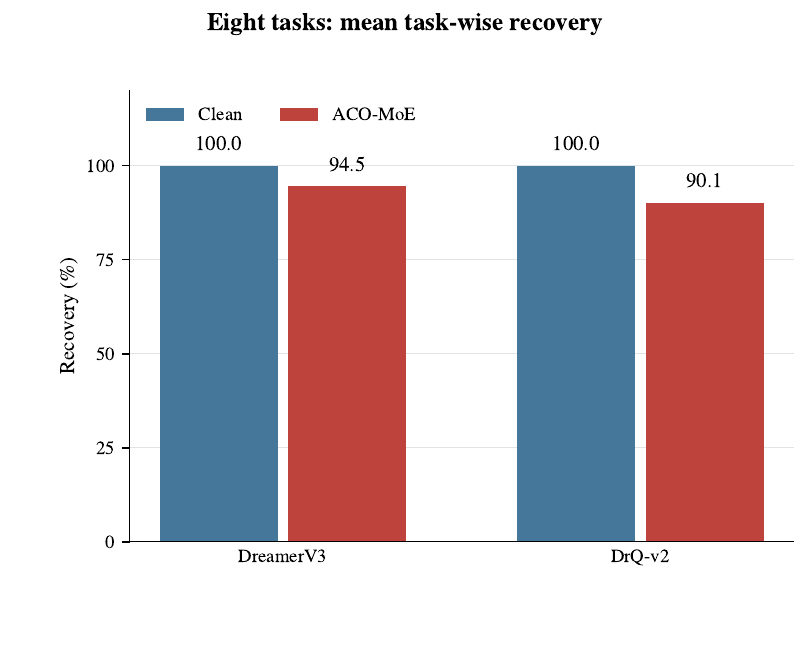}
    \caption{\textbf{Backbone generalization on VDCS Markov-temporal perturbations.}
    (a) \texttt{finger\_spin}: clean and ACO-MoE returns for DreamerV3 and
    DrQ-v2. (b) Average clean-normalized return across the eight tasks for each
    backbone and preprocessing setting.}
    \label{fig:backbone_generalization}
\end{figure}

\begin{samepage}
\subsection{Additional Ablation Studies}
\label{app:results_ablations}

Figure~\ref{fig:protocol_ablation} compares the frozen interface with
no-preprocessing and freeze-then-adapt variants; Tables~\ref{tab:video_hard_fixed}
and~\ref{tab:color_hard_fixed} test the same removals on DMC-GB, and
Table~\ref{tab:vdcs_static_single} reports fixed-corruption VDCS results.
On Video Hard, the single-expert, freeze-then-adapt, Mask-Only, RGB-Only, and
no-restoration variants score 11.1\%, 46.6\%, 10.5\%, 57.4\%, and 62.9\% below the full model, respectively.
Color Hard results are reported separately.

\end{samepage}

\begin{figure}[H]
    \centering
    \includegraphics[width=\textwidth,trim=13 27 10 3,clip]{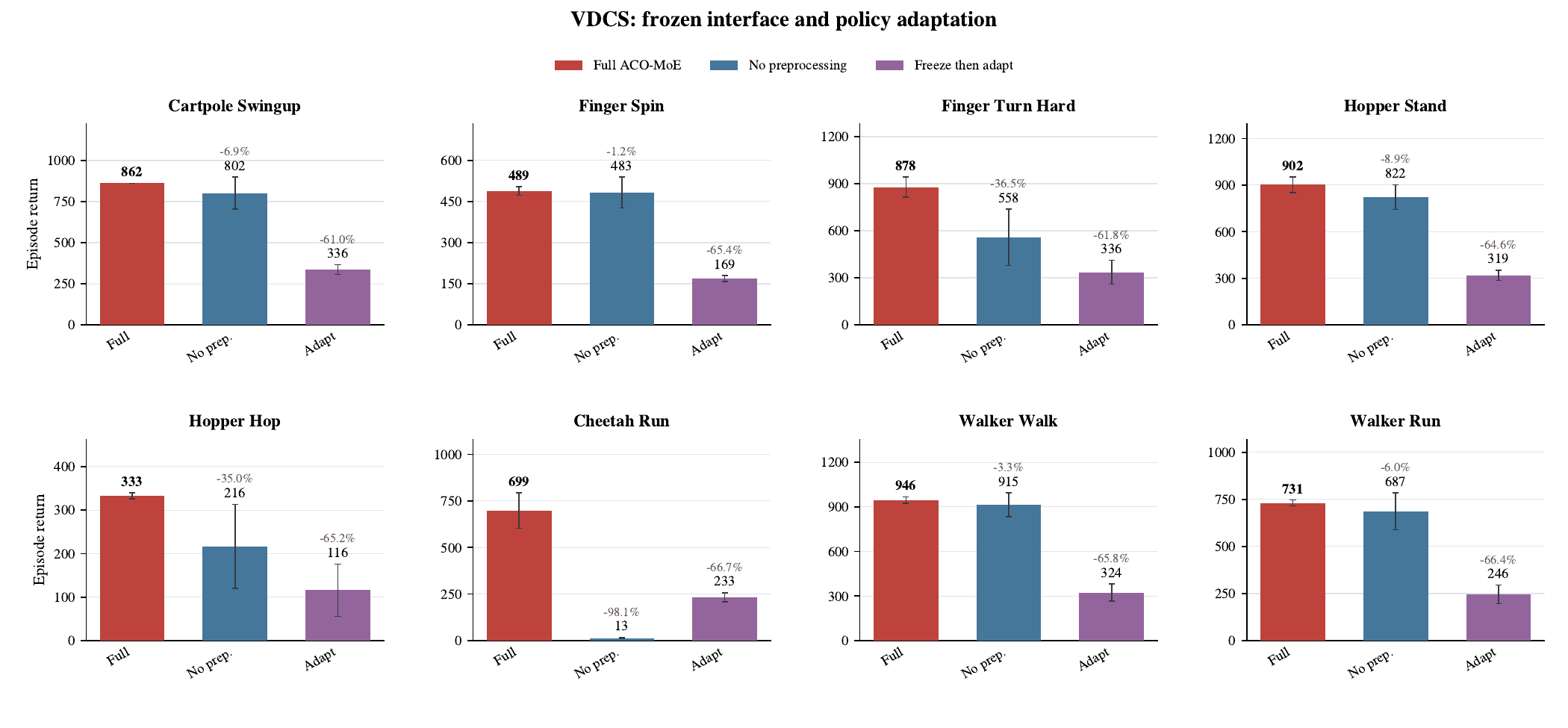}
    \caption{Ablation on VDCS Markov-temporal perturbations with DreamerV3.
We compare frozen ACO-MoE, DreamerV3 without preprocessing, and a
freeze-then-adapt variant, all starting from controllers trained on observations processed by the learned mask predictor. The full-model values match the main VDCS result table. Percentages denote performance relative to full ACO-MoE.}
    \label{fig:protocol_ablation}
\end{figure}

\begin{figure}[H]
    \centering
    \includegraphics[width=0.94\textwidth,trim=34 27 15 3,clip]{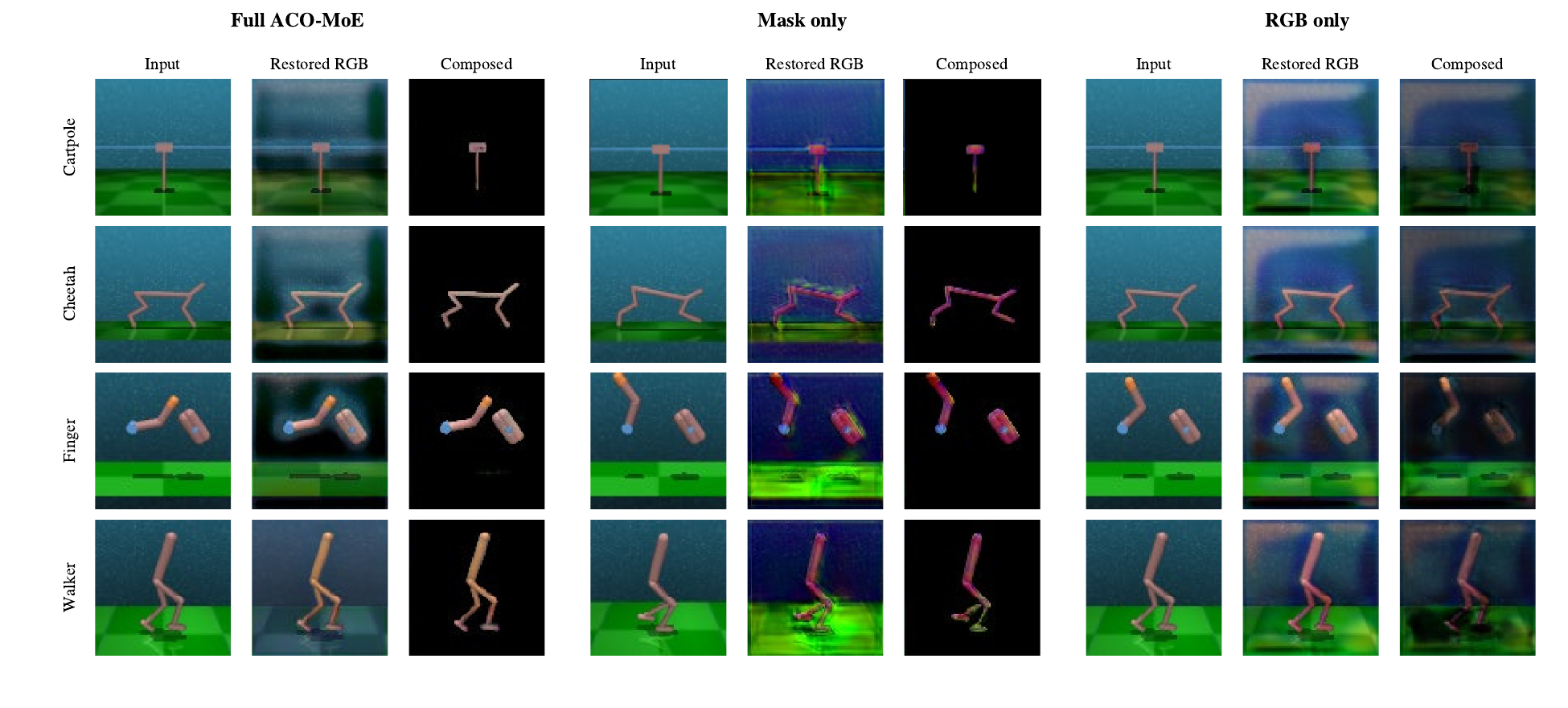}
    \caption{Qualitative comparison of RGB-Only, Mask-Only, and combined agent-centric observations under dynamic perturbations.}
    \label{fig:mask-rgb-ablation}
\end{figure}

\begin{table}[H]
\centering
\caption{Ablation on DMControl (\textit{video-hard}). Differences are relative to the full ACO-MoE reference in this table; the Average row compares mean returns. Benchmark results use Table~\ref{tab:video_hard_results}. \ranklegend{}.}
\label{tab:video_hard_fixed}

\small
\renewcommand{\arraystretch}{1.2}
\setlength{\tabcolsep}{2pt}
\begin{tabularx}{\textwidth}{@{}l*{6}{>{\centering\arraybackslash}X}@{}}
\toprule
\makecell[c]{DMControl \\ (video hard)} &
\makecell[c]{\textbf{ACO-MoE} \\ \textbf{(Ours)}} &
\makecell[c]{Single Expert} &
\makecell[c]{w/o frozen \\ policy} &
\makecell[c]{mask\_only} &
\makecell[c]{rgb\_only} &
\makecell[c]{w/o \\ restoration} \\
\midrule
CP-S & \bestcell{\tabstat{700}{78}} & \secondcell{\makecell[c]{\tabstat{651}{105}\\{\tabsecondary -49 (7.0\%)}}} & \makecell[c]{\tabstat{374}{96}\\{\tabsecondary -326 (46.6\%)}} & \thirdcell{\makecell[c]{\tabstat{504}{116}\\{\tabsecondary -196 (28.0\%)}}} & \makecell[c]{\tabstat{133}{95}\\{\tabsecondary -567 (81.0\%)}} & \makecell[c]{\tabstat{157}{68}\\{\tabsecondary -543 (77.6\%)}} \\
\addlinespace
F-S & \bestcell{\tabstat{486}{24}} & \thirdcell{\makecell[c]{\tabstat{429}{18}\\{\tabsecondary -57 (11.7\%)}}} & \makecell[c]{\tabstat{278}{68}\\{\tabsecondary -208 (42.8\%)}} & \secondcell{\makecell[c]{\tabstat{480}{30}\\{\tabsecondary -6 (1.2\%)}}} & \makecell[c]{\tabstat{273}{124}\\{\tabsecondary -213 (43.8\%)}} & \makecell[c]{\tabstat{158}{178}\\{\tabsecondary -328 (67.5\%)}} \\
\addlinespace
W-S & \secondcell{\tabstat{969}{21}} & \thirdcell{\makecell[c]{\tabstat{897}{21}\\{\tabsecondary -72 (7.4\%)}}} & \makecell[c]{\tabstat{503}{33}\\{\tabsecondary -466 (48.1\%)}} & \bestcell{\makecell[c]{\tabstat{983}{13}\\{\tabsecondary +14 (+1.4\%)}}} & \makecell[c]{\tabstat{657}{356}\\{\tabsecondary -312 (32.2\%)}} & \makecell[c]{\tabstat{647}{392}\\{\tabsecondary -322 (33.2\%)}} \\
\addlinespace
W-W & \secondcell{\tabstat{941}{67}} & \thirdcell{\makecell[c]{\tabstat{899}{22}\\{\tabsecondary -42 (4.5\%)}}} & \makecell[c]{\tabstat{477}{89}\\{\tabsecondary -464 (49.3\%)}} & \bestcell{\makecell[c]{\tabstat{971}{14}\\{\tabsecondary +30 (+3.2\%)}}} & \makecell[c]{\tabstat{404}{263}\\{\tabsecondary -537 (57.1\%)}} & \makecell[c]{\tabstat{393}{285}\\{\tabsecondary -548 (58.2\%)}} \\
\addlinespace
C-R & \bestcell{\tabstat{782}{26}} & \secondcell{\makecell[c]{\tabstat{572}{126}\\{\tabsecondary -210 (26.9\%)}}} & \makecell[c]{\tabstat{439}{77}\\{\tabsecondary -343 (43.9\%)}} & \thirdcell{\makecell[c]{\tabstat{533}{156}\\{\tabsecondary -249 (31.8\%)}}} & \makecell[c]{\tabstat{185}{62}\\{\tabsecondary -597 (76.3\%)}} & \makecell[c]{\tabstat{85}{43}\\{\tabsecondary -697 (89.1\%)}} \\
\midrule
\textbf{Average} &  & {-11.1\%} & {-46.6\%} & {-10.5\%} & {-57.4\%} & {-62.9\%} \\
\bottomrule
\end{tabularx}
\end{table}

\begin{table}[H]
\centering
\caption{Ablation on DMControl (\textit{color-hard}). Differences are relative to the full ACO-MoE reference in this table; the Average row compares mean returns. Benchmark results use Table~\ref{tab:color_hard_results}. \ranklegend{}.}
\label{tab:color_hard_fixed}

\small
\setlength{\tabcolsep}{2pt}
\begin{tabularx}{\textwidth}{@{}l*{6}{>{\centering\arraybackslash}X}@{}}
\toprule
\makecell[c]{DMControl \\ (color hard)} &
\makecell[c]{\textbf{ACO-MoE}\\ \textbf{(Ours)}} &
\makecell[c]{Single Expert} &
\makecell[c]{w/o frozen \\ policy} &
\makecell[c]{mask\_only} &
\makecell[c]{rgb\_only} &
\makecell[c]{w/o \\ restoration} \\
\midrule
CP-S & \bestcell{\tabstat{845}{3}} & \secondcell{\makecell[c]{\tabstat{827}{1}\\{\tabsecondary -18 (2.1\%)}}} & \makecell[c]{\tabstat{459}{96}\\{\tabsecondary -386 (45.7\%)}} & \thirdcell{\makecell[c]{\tabstat{725}{173}\\{\tabsecondary -120 (14.2\%)}}} & \makecell[c]{\tabstat{217}{171}\\{\tabsecondary -628 (74.3\%)}} & \makecell[c]{\tabstat{621}{295}\\{\tabsecondary -224 (26.5\%)}} \\
\addlinespace
F-S & \bestcell{\tabstat{500}{113}} & \secondcell{\makecell[c]{\tabstat{494}{11}\\{\tabsecondary -6 (1.2\%)}}} & \makecell[c]{\tabstat{285}{68}\\{\tabsecondary -215 (43.0\%)}} & \makecell[c]{\tabstat{399}{110}\\{\tabsecondary -101 (20.2\%)}} & \makecell[c]{\tabstat{395}{161}\\{\tabsecondary -105 (21.0\%)}} & \thirdcell{\makecell[c]{\tabstat{421}{78}\\{\tabsecondary -79 (15.8\%)}}} \\
\addlinespace
W-S & \secondcell{\tabstat{959}{47}} & \thirdcell{\makecell[c]{\tabstat{949}{19}\\{\tabsecondary -10 (1.0\%)}}} & \makecell[c]{\tabstat{496}{33}\\{\tabsecondary -463 (48.3\%)}} & \makecell[c]{\tabstat{903}{25}\\{\tabsecondary -56 (5.8\%)}} & \makecell[c]{\tabstat{638}{283}\\{\tabsecondary -321 (33.5\%)}} & \bestcell{\makecell[c]{\tabstat{984}{13}\\{\tabsecondary +25 (+2.6\%)}}} \\
\addlinespace
W-W & \secondcell{\tabstat{936}{29}} & \bestcell{\makecell[c]{\tabstat{954}{26}\\{\tabsecondary +18 (+1.9\%)}}} & \makecell[c]{\tabstat{482}{89}\\{\tabsecondary -454 (48.5\%)}} & \thirdcell{\makecell[c]{\tabstat{915}{57}\\{\tabsecondary -21 (2.2\%)}}} & \makecell[c]{\tabstat{319}{310}\\{\tabsecondary -617 (65.9\%)}} & \makecell[c]{\tabstat{722}{162}\\{\tabsecondary -214 (22.9\%)}} \\
\addlinespace
C-R & \bestcell{\tabstat{613}{121}} & \secondcell{\makecell[c]{\tabstat{597}{170}\\{\tabsecondary -16 (2.6\%)}}} & \makecell[c]{\tabstat{344}{77}\\{\tabsecondary -269 (43.9\%)}} & \thirdcell{\makecell[c]{\tabstat{593}{152}\\{\tabsecondary -20 (3.3\%)}}} & \makecell[c]{\tabstat{267}{162}\\{\tabsecondary -346 (56.4\%)}} & \makecell[c]{\tabstat{480}{147}\\{\tabsecondary -133 (21.7\%)}} \\
\midrule
\textbf{Average} &  & {-0.8\%} & {-46.4\%} & {-8.3\%} & {-52.3\%} & {-16.2\%} \\
\bottomrule
\end{tabularx}
\end{table}


\begin{table}[H]
\centering
\caption{Switching lowers every baseline relative to fixed corruptions while ACO-MoE maintains its return. Values average eight DMControl tasks. Static retention divides by the clean-input reference mean of 769. Markov retention averages the task-wise ratios to clean return. \ranklegend{}.}
\label{tab:vdcs_static_single}

\small
\setlength{\tabcolsep}{4.5pt}
\begin{tabularx}{\textwidth}{@{}l*{5}{>{\centering\arraybackslash}X}@{}}
\toprule
\textbf{Degradation}
  & \textbf{DreamerV3}~\cite{hafner2023mastering}
  & \textbf{SimGRL}~\cite{song2024simgrl}
  & \textbf{FTR}~\cite{wang2025ftr}
  & \textbf{ACO-MoE}
  & \textbf{Retained} \\
  & \textbf{(No Preproc.)}
  &
  &
  & \textbf{(Ours)}
  & \textbf{(\%)} \\
\midrule
Rain & \secondcell{\tabstat{613}{355}} & \tabstat{459}{385} & \thirdcell{\tabstat{516}{342}} & \bestcell{\tabstat{752}{221}} & 97.8 \\
Haze & \secondcell{\tabstat{610}{356}} & \thirdcell{\tabstat{458}{383}} & \tabstat{451}{395} & \bestcell{\tabstat{737}{222}} & 95.8 \\
Snow & \secondcell{\tabstat{567}{376}} & \tabstat{458}{388} & \thirdcell{\tabstat{463}{334}} & \bestcell{\tabstat{753}{217}} & 97.9 \\
\textit{Motion Blur} & \secondcell{\tabstat{557}{411}} & \thirdcell{\tabstat{295}{291}} & \textit{\tabstat{65}{42}} & \bestcell{\tabstat{569}{262}} & \textit{74.0} \\
Gaussian Noise & \secondcell{\tabstat{592}{376}} & \tabstat{462}{390} & \thirdcell{\tabstat{572}{363}} & \bestcell{\tabstat{750}{216}} & 97.5 \\
Low-Light & \secondcell{\tabstat{578}{375}} & \thirdcell{\tabstat{459}{386}} & \textit{\tabstat{225}{118}} & \bestcell{\tabstat{737}{218}} & 95.8 \\
JPEG Compr. & \secondcell{\tabstat{639}{346}} & \tabstat{369}{332} & \thirdcell{\tabstat{429}{354}} & \bestcell{\tabstat{733}{226}} & 95.3 \\
\midrule
\textbf{Avg.\ (Static)} & \secondcell{\tabstat{594}{367}} & \thirdcell{\tabstat{423}{357}} & \tabstat{376}{276} & \bestcell{\tabstat{719}{217}} & \textbf{93.5} \\
\midrule
\textbf{Avg.\ (Markov)} & \secondcell{\tabstat{562}{294}} & \thirdcell{\tabstat{412}{315}} & \tabstat{316}{226} & \bestcell{\tabstat{730}{204}} & \textbf{94.5} \\
\bottomrule
\end{tabularx}

\end{table}

\subsection{Comparison with Generic Image Restoration Methods}
\label{app:generic_restoration}

To isolate preprocessing from controller changes, we compare four restorers before the same frozen DreamerV3: PromptIR~\cite{potlapalli2023promptir}, Restormer~\cite{zamir2022restormer}, AdaIR~\cite{cui2025adair}, and NAFNet~\cite{chen2022simple}. Each restorer keeps its official architecture and is fine-tuned from its released weights on the same VDCS paired training set as ACO-MoE, with the clean frame as its target. The restorers are then frozen and restore full RGB frames without foreground composition or additional controller training.

Table~\ref{tab:generic_restoration} reports an average of 759.9 for ACO-MoE versus 546.6 for AdaIR, the strongest generic restorer, a gain of 213.3 points or 39.0\%. The comparison holds training pairs and the controller fixed. The generic restorers return full frames, while ACO-MoE returns the foreground observation expected by this controller.

The gains span several tasks: the strongest generic restorer scores 457 on \texttt{cartpole\_swingup}, 444 on \texttt{hopper\_stand}, and 388 on \texttt{cheetah\_run}, against 862, 902, and 699 for ACO-MoE. The gap narrows on \texttt{finger\_turn\_hard}, with 878 versus 834, and reverses on \texttt{walker\_stand}, with 970 versus 985. Table~\ref{tab:generic_restoration} gives the mean and standard deviation over evaluation episodes. The Average row summarizes across tasks. The $\Delta$ column gives the absolute and relative gains over the strongest restorer.

\begin{table}[H]
\centering
\caption{\textbf{ACO-MoE outperforms full-frame restorers on VDCS Markov switching.} All methods use the same DreamerV3~\cite{hafner2023mastering} controller. All four restorers are fine-tuned from their released weights on the same VDCS paired training set as ACO-MoE. \ranklegend{}. $^\dagger$W-S replaces W-R in main-paper Table~\ref{tab:vdcs_markov_temporal_results}.}
\label{tab:generic_restoration}
\small
\setlength{\tabcolsep}{4pt}
\begin{tabularx}{\textwidth}{@{}l*{5}{>{\centering\arraybackslash}X}r@{}}
\toprule
\multirow{2}{*}{\textbf{Task}} & \multicolumn{4}{c}{\textbf{Generic restoration}} & \multirow{2}{*}{\textbf{ACO-MoE}} & \multirow{2}{*}{$\Delta$} \\
\cmidrule(lr){2-5}
& PromptIR~\cite{potlapalli2023promptir} & Restormer~\cite{zamir2022restormer} & AdaIR~\cite{cui2025adair} & NAFNet~\cite{chen2022simple} & & \\
\midrule
CP-S & \secondcell{\tabstat{457}{185}} & \tabstat{393}{142} & \thirdcell{\tabstat{421}{216}} & \tabstat{226}{90} & \bestcell{\tabstat{862}{75}} & +405 / +89\% \\
F-S & \secondcell{\tabstat{441}{38}} & \tabstat{386}{108} & \thirdcell{\tabstat{432}{23}} & \tabstat{346}{39} & \bestcell{\tabstat{489}{102}} & +48 / +11\% \\
F-TH & \tabstat{779}{102} & \thirdcell{\tabstat{791}{114}} & \secondcell{\tabstat{834}{69}} & \tabstat{522}{313} & \bestcell{\tabstat{878}{65}} & +44 / +5\% \\
H-S & \thirdcell{\tabstat{361}{150}} & \tabstat{345}{88} & \secondcell{\tabstat{444}{126}} & \tabstat{240}{57} & \bestcell{\tabstat{902}{50}} & +458 / +103\% \\
H-H & \secondcell{\tabstat{130}{15}} & \tabstat{95}{29} & \thirdcell{\tabstat{111}{12}} & \tabstat{54}{24} & \bestcell{\tabstat{333}{46}} & +203 / +156\% \\
C-R & \secondcell{\tabstat{388}{75}} & \thirdcell{\tabstat{386}{102}} & \thirdcell{\tabstat{386}{55}} & \tabstat{270}{111} & \bestcell{\tabstat{699}{95}} & +311 / +80\% \\
W-W & \secondcell{\tabstat{857}{73}} & \thirdcell{\tabstat{826}{68}} & \tabstat{760}{53} & \tabstat{616}{54} & \bestcell{\tabstat{946}{87}} & +89 / +10\% \\
W-S$^\dagger$ & \tabstat{949}{62} & \secondcell{\tabstat{982}{15}} & \bestcell{\tabstat{985}{16}} & \tabstat{770}{101} & \thirdcell{\tabstat{970}{18}} & $-15$ / $-2$\% \\
\midrule
\rowcolor{gray!7}
\textbf{Average} & \thirdcell{\tabstat{545.2}{265.7}} & \tabstat{525.5}{283.6} & \secondcell{\tabstat{546.6}{268.2}} & \tabstat{380.5}{221.0} & \bestcell{\tabstat{759.9}{218.7}} & +213.3 / +39.0\% \\
\bottomrule
\end{tabularx}
\end{table}

\subsection{Zero-Shot Generalization to Unseen OOD Visual Perturbations}
\label{app:unseen_ood}

We evaluate the frozen adapter and controller on six held-out families across
the eight DMControl tasks. Each episode fixes one family and uses $\pm10\%$
intensity jitter. Figure~\ref{fig:unseen_deg} shows the perturbations and
Table~\ref{tab:zero_shot_unseen_ood} reports every task--family pair: family
averages range from 535.8 (defocus blur) to 744.2 (shot noise).

\begin{table}[H]
\centering
\caption{Zero-shot evaluation on unseen visual perturbations with ACO-MoE and DreamerV3~\cite{hafner2023mastering}. Both networks are fixed during evaluation. Results are mean $\pm$ std over 5 seeds, with 10 evaluation episodes per seed. The last row reports the mean over the 8 tasks for each unseen corruption family.}
\label{tab:zero_shot_unseen_ood}

\small
\setlength{\tabcolsep}{2pt}
\begin{tabular}{lcccccc}
\toprule
Task
& Defocus Blur
& Frost
& Occlusion Patch
& Saturation
& Shadow
& Shot Noise \\
\midrule
CP-S & \tabstat{857.68}{2.65} & \tabstat{860.36}{0.42} & \tabstat{852.37}{4.06} & \tabstat{861.51}{0.45} & \tabstat{860.61}{0.60} & \tabstat{859.60}{0.92} \\
F-S & \tabstat{487.33}{10.76} & \tabstat{504.53}{10.19} & \tabstat{465.34}{15.08} & \tabstat{503.26}{8.24} & \tabstat{497.08}{25.06} & \tabstat{496.09}{14.31} \\
F-TH & \tabstat{368.72}{444.15} & \tabstat{862.84}{131.93} & \tabstat{602.89}{357.54} & \tabstat{828.23}{147.54} & \tabstat{938.94}{23.60} & \tabstat{921.70}{38.19} \\
H-S & \tabstat{324.50}{61.02} & \tabstat{928.73}{17.52} & \tabstat{889.16}{75.71} & \tabstat{930.03}{21.30} & \tabstat{912.01}{23.80} & \tabstat{910.37}{21.31} \\
H-H & \tabstat{97.69}{54.01} & \tabstat{355.18}{18.09} & \tabstat{282.96}{10.30} & \tabstat{360.23}{19.24} & \tabstat{327.51}{37.60} & \tabstat{364.15}{20.33} \\
C-R & \tabstat{532.58}{111.11} & \tabstat{742.73}{67.93} & \tabstat{566.57}{88.08} & \tabstat{748.78}{111.11} & \tabstat{729.70}{76.45} & \tabstat{720.70}{59.70} \\
W-W & \tabstat{926.76}{37.17} & \tabstat{950.84}{21.93} & \tabstat{904.90}{37.06} & \tabstat{954.85}{16.38} & \tabstat{945.62}{45.54} & \tabstat{945.46}{38.25} \\
W-R & \tabstat{690.80}{39.96} & \tabstat{742.83}{13.56} & \tabstat{606.84}{43.38} & \tabstat{717.95}{32.53} & \tabstat{733.62}{22.69} & \tabstat{735.42}{13.57} \\
\midrule
Average over tasks & $535.8$ & $743.5$ & $646.4$ & $738.1$ & $743.1$ & $744.2$ \\
\bottomrule
\end{tabular}
\end{table}

\twocolumn
\section{Downstream Controller Objectives}
\label{app:controller_objectives}

\textbf{DreamerV3~\cite{hafner2023mastering}} learns a recurrent state-space model from pixels and trains an actor--critic in latent imagination. An encoder infers $z_t \sim q_\phi(z_t \mid h_t, x_t)$ and the dynamics update $h_t = f_\phi(h_{t-1}, z_{t-1}, a_{t-1})$. The world-model objective is
\begin{equation}
\begin{aligned}
\mathcal{L}_{\text{wm}}=\mathbb{E}\Big[
&\underbrace{-\log p_\phi(x_t\mid h_t,z_t)}_{\text{reconstruction}}
+\mathcal{L}_{\text{pred}}(r_t,c_t) \\
&+\underbrace{\beta_{\mathrm{dyn}}\mathrm{KL}\!\left[\operatorname{sg}(q_\phi)\,\|\,p_\phi\right]}_{\text{dynamics}}
+\underbrace{\beta_{\mathrm{rep}}\mathrm{KL}\!\left[q_\phi\,\|\,\operatorname{sg}(p_\phi)\right]}_{\text{representation}}
\Big],
\end{aligned}
\label{eq:dreamerv3_wm}
\end{equation}
where $\operatorname{sg}(\cdot)$ is stop-gradient, $c_t$ the continuation flag, and $\mathcal{L}_{\text{pred}}$ supervises reward and continuation. The reconstruction term encourages the latent to retain nuisance visual detail when $x_t$ is corrupted (Sec.~D). \textbf{TD-MPC2~\cite{hansen2024tdmpc2}} plans in a latent $z=h(x)$ learned by joint-embedding and reward/value prediction without reconstruction; it avoids reconstruction-driven contamination, but $h(\cdot)$ still encodes corrupted pixels. \textbf{$\pi_{0.5}$~\cite{intelligence2025pi05}} is a vision-language-action flow-matching policy (PaliGemma backbone with an action expert); on the Piper arm it consumes the global and wrist images ($224{\times}224$ after resize-with-pad), the 7-D joint state, and the instruction, and outputs 50-step chunks of delta joint targets and absolute gripper width.

\section{Real-Robot Deployment with ACO-MoE}
\label{app:real_robot}

\subsection{Data and Policy}
We evaluate three distinct tasks: placing a cube into the blue box, placing a cube into the white box, and placing a cup on the coaster. The Piper arm records two RGB cameras (global and wrist, $640{\times}480$, H.264) and 7-D joint states at 20\,Hz. We collect 193 successful teleoperated episodes in the real environment, with 117 cube$\to$box and 76 cup$\to$coaster episodes totaling 65,958 frames. Held-out data: 45 episodes (23 cube, 22 cup; 14,338 frames) from separate recording sessions. $\pi_{0.5}$ is fully fine-tuned on these demonstrations for 10k updates (batch 16, delta joint actions, quantile normalization) and frozen; neither the policy nor its data are modified by ACO-MoE. Fig.~\ref{fig:real_tasks} shows the three tasks from the third-person, global, and wrist views.

\begin{figure*}[t]
  \centering
  \includegraphics[width=\textwidth,page=2,trim=0 250 0 20,clip]{figures/fig_real_task_sequences.pdf}
  \caption{Three real-robot tasks at five phases: cube to blue box, cube to white box, and cup to coaster. Each task is shown from the third-person, global, and wrist views. The global and wrist cameras share the robot clock, and the third-person phone video is aligned to the global camera by cross-correlating frame motion.}
  \label{fig:real_tasks}
\end{figure*}

\subsection{Restorer Training}
Frames are decoded and resized exactly as the policy resizes them (PIL bilinear resize-with-pad), and only the $168{\times}224$ content region is restored; the letterbox bars are re-added afterwards. VDCS operators are rescaled from the 64\,px reference (lengths and radii by $s{=}\min(H,W)/64{=}2.625$, particle counts to preserve covered area). Each training sample draws a mode uniformly (or a clean frame with probability 0.1), a severity $\iota\sim\mathcal{U}[0.1,1]$, and a horizontal flip. Optimization: AdamW (lr $2{\times}10^{-4}$, weight decay $10^{-4}$, 1k warm-up, cosine decay to 2\%), 30k updates, batch 64 on two GPUs, bf16; loss weights $\lambda_{\mathrm{fft}}{=}0.05$, load balance 0.01, entropy 0.005; router temperature annealed $1{\to}0.5$; hard top-1 routing for the final 15\% of updates. The \emph{narrow-range} model in Fig.~\ref{fig:real_stress_full} differs only in its severity range, $[0.5,1.2]\,\bar\iota_k$. On a held-out wide-range validation set ($\iota\sim\mathcal{U}[0.1,1]$), the deployed model raises PSNR from 23.7 to 40.5\,dB on average (motion blur 23.1$\to$32.2, low light 13.0$\to$37.1, haze 22.2$\to$47.2\,dB) and keeps clean frames at 75.3\,dB.

\begin{figure*}[t]
  \centering
  \includegraphics[width=\textwidth]{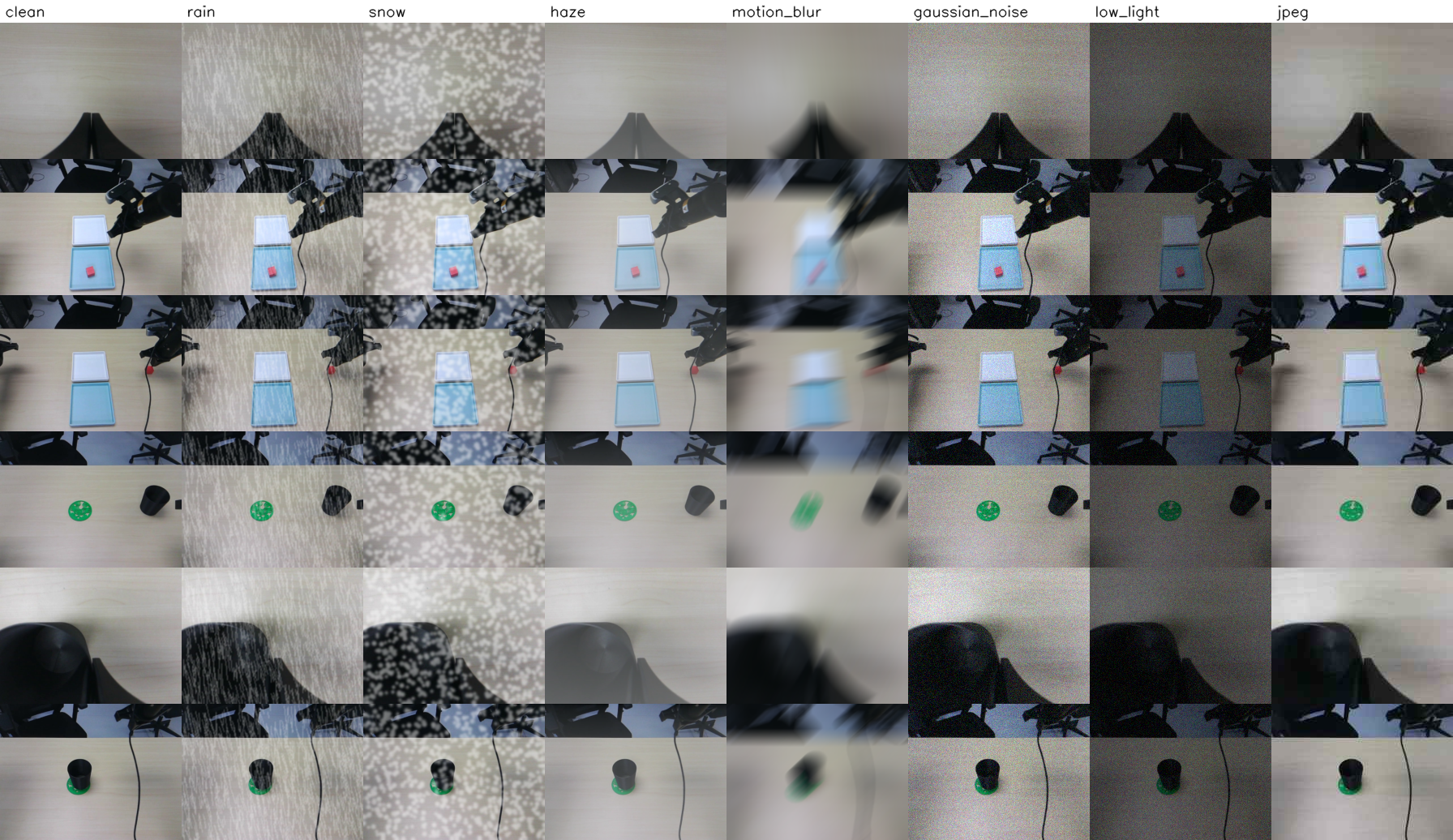}
  \caption{VDCS operators on held-out real-robot frames at the base severities ($\bar\iota_k\pm10\%$). Columns: clean, rain, snow, haze, motion blur, Gaussian noise, low light, JPEG. Rows: six randomly drawn held-out frames from the global and wrist cameras.}
  \label{fig:real_degradation_grid}
\end{figure*}

\subsection{Evaluation Protocol}
For each evaluated frame, the clean, corrupted, and restored observations are passed through the unchanged openpi preprocessing and $\pi_{0.5}$ with one shared flow-matching noise sample (10 denoising steps), so all conditions differ only in their images. Feeding the clean frames through our resize path reproduces the original pipeline's actions exactly (zero deviation). Metrics are computed on the 7 action dimensions in physical units: mean absolute joint error over the 50-step chunk (degrees) and gripper error (mm), both against the demonstration and against the clean-input prediction. The policy's own sampling spread, measured with a second noise sample on the clean frame, is $0.65^\circ$.

\begin{table}[t]
\centering
\caption{$\pi_{0.5}$~\cite{intelligence2025pi05} under VDCS base severities on 256 held-out frames (corrupted / restored). Dev.: deviation from the clean-input prediction. Err.: error against the demonstration. Clean input is its own reference, giving infinite PSNR and zero deviation.}
\label{tab:real_base_severity}
\footnotesize
\setlength{\tabcolsep}{1.25pt}
\begin{tabular}{@{}lcccc@{}}
\toprule
 & \textbf{PSNR (dB)} & \textbf{Dev.\ joint ($^\circ$)} & \textbf{Dev.\ grip (mm)} & \textbf{Err.\ joint ($^\circ$)} \\
\midrule
Clean & $\infty$ & 0 & 0 & 2.12 \\
\quad + ACO-MoE & 94.5 & 0.015 & 0.02 & 2.12 \\
\midrule
Rain & 23.1 / \textbf{41.9} & 0.58 / \textbf{0.17} & 1.58 / \textbf{0.37} & 2.17 / 2.16 \\
Snow & 18.5 / \textbf{41.9} & 0.64 / \textbf{0.11} & 1.74 / \textbf{0.30} & 2.23 / 2.12 \\
Haze & 20.3 / \textbf{47.7} & 0.14 / \textbf{0.03} & 0.30 / \textbf{0.07} & 2.11 / 2.13 \\
Motion blur & 24.4 / \textbf{33.3} & 1.36 / \textbf{0.56} & 2.70 / \textbf{1.25} & 2.56 / \textbf{2.28} \\
Gauss.\ noise & 26.3 / \textbf{40.9} & 0.63 / \textbf{0.23} & 0.71 / \textbf{0.36} & 2.10 / 2.20 \\
Low light & 10.2 / \textbf{36.5} & 0.77 / \textbf{0.29} & 1.23 / \textbf{0.49} & 2.14 / 2.21 \\
JPEG & 36.0 / \textbf{40.5} & 0.32 / \textbf{0.27} & 0.57 / \textbf{0.48} & 2.21 / 2.22 \\
\bottomrule
\end{tabular}
\end{table}

\subsection{Results}
Table~\ref{tab:real_base_severity} reports the base-severity results per mode. Restoration reduces the deviation from the clean-input prediction for every mode (by 58\% for motion blur, 62--83\% for rain, snow, haze, noise, and low light, and 15\% for JPEG), and reduces the error against the demonstrations for motion blur (2.56$^\circ\to$2.28$^\circ$). For the remaining modes this error changes by at most $0.11^\circ$; for Gaussian noise and low light it is slightly higher after restoration ($+0.10^\circ$ and $+0.07^\circ$), even though the restored actions are closer to the clean-input actions. Fig.~\ref{fig:real_stress_full} gives the full severity sweep, including the action deviation and image PSNR, for the narrow- and wide-range restorers.

\begin{figure*}[t]
  \centering
  \includegraphics[width=\textwidth]{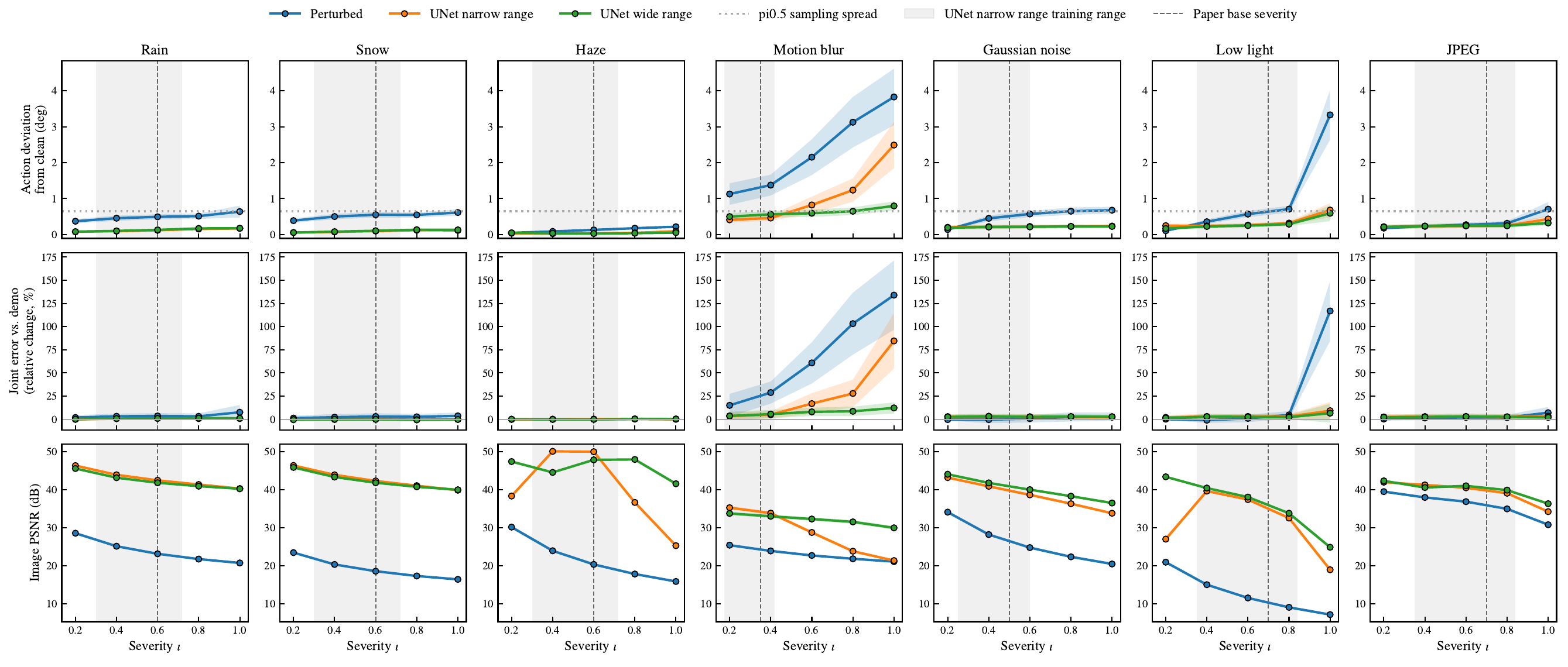}
  \caption{Severity sweep on 128 held-out real-robot frames. Rows: deviation of $\pi_{0.5}$'s actions from its clean-input prediction (dotted: sampling spread of the policy), relative increase of the joint error against the demonstration, and image PSNR. Blue: corrupted; orange: ACO-MoE trained near the base severities (shaded range); green: ACO-MoE trained on $\iota\in[0.1,1]$. Bands: 95\% CI over frames.}
  \label{fig:real_stress_full}
\end{figure*}

\begin{table}[t]
\centering
\caption{Real codec artifacts for $\pi_{0.5}$~\cite{intelligence2025pi05}. We re-encode clean held-out videos with streaming x264, a GOP of 30, and no B-frames. We decode 14,800 frame pairs at the policy resolution.}
\label{tab:real_codec}
\footnotesize
\begin{tabular}{@{}lcccc@{}}
\toprule
\textbf{x264 bitrate} & 1000k & 300k & 100k & 50k \\
\midrule
PSNR, decoded & 44.84 & 42.08 & 38.03 & 33.06 \\
PSNR, + ACO-MoE & 44.84 & 42.18 & 38.56 & 33.79 \\
\bottomrule
\end{tabular}
\end{table}

\paragraph{Real codec artifacts.} Table~\ref{tab:real_codec} shows that streaming compression of the real camera videos is mild after the policy's downsampling, and that the restorer, trained only on synthetic JPEG, changes these frames by less than 1\,dB.

\paragraph{Latency.} With BatchNorm folded into the convolutions, fp16, channels-last layout, and top-1 routing, restoring a two-camera batch at $168{\times}224$ takes 1.30\,ms on one RTX~5000 Ada (eager PyTorch, no compilation; $C{=}32$, $N_e{=}7$). Widths $C\in\{16,24,32,48\}$ with $N_e\in\{4,7\}$ all run in 1.3--2.0\,ms, so capacity can be chosen for quality rather than speed.

\subsection{Closed-Loop Evaluation on the Robot}
\label{sec:supp_closed_loop}

\paragraph{Protocol.} The frozen $\pi_{0.5}$ controls the Piper arm in real time with real-time chunking, actions at 20\,Hz, a 200\,Hz control loop, and 60\% of the maximum drive speed. Both camera images pass through ACO-MoE as described in Sec.~\ref{sec:exp_real} ($C{=}32$, $N_e{=}7$, top-1 routing, fp16) or bypass it. The VDCS operators corrupt the two live streams at the policy resolution with the same type and severity on both cameras and seed 0. The settings are clean input, low light at $\iota{=}0.9$, motion blur at $\iota{=}0.6$, and Markov switching among motion blur ($\bar\iota{=}0.5$), Gaussian noise ($\bar\iota{=}0.4$), and low light ($\bar\iota{=}0.7$) with stay probability 0.95 and a severity random walk within 10\% of the base value, and a physically dim room with Gaussian noise at $\iota{=}0.5$ added to both streams. For the physical low-light condition, we dim the room itself. The global camera frames before the injected noise have a mean intensity of 15 to 19 on the 8-bit scale, against about 128 under normal lighting. The operator judges success. The clean, motion-blur, and dim-room settings comprise 20 trials, run in two batches of 10 on different days with identical settings, except that the 20 cup trials in the dim room form a single batch. The first restored cube batch in the dim room used noise at $\iota{=}0.51$. Low light at $\iota{=}0.9$ and Markov switching comprise the 10 trials of the first batch.

\begin{table}[hpbt]
\centering
\caption{Closed-loop success of $\pi_{0.5}$~\cite{intelligence2025pi05} on the Piper arm. The operator judges cube placement in the blue box and cup placement on the coaster. Off and On denote evaluation without and with ACO-MoE. $^{\dagger}$ marks task-weighted training. Both live camera streams receive the injected perturbations. Dim-room trials combine physical low light with injected Gaussian noise. $p$ is the one-sided Fisher exact test.}
\label{tab:real_closed_loop}
\footnotesize
\setlength{\tabcolsep}{4pt}
\begin{tabular}{@{}llccc@{}}
\toprule
\textbf{Task} & \textbf{Perturbation} & \textbf{Off} & \textbf{On} & $p$ \\
\midrule
\multirow{5}{*}{Blue box} & Clean & 20/20 & 20/20 & 1.00 \\
 & Low light, $\iota{=}0.9$ & 10/10 & 10/10 & 1.00 \\
 & Motion blur, $\iota{=}0.6$ & 1/20 & \textbf{19/20} & $2.9{\times}10^{-9}$ \\
 & Markov switching & 10/10 & 10/10 & 1.00 \\
 & Dim room + noise, $\iota{=}0.5$ & 0/20 & \textbf{18/20} & $1.7{\times}10^{-9}$ \\
\midrule
\multirow{5}{*}{Cup} & Clean & 20/20 & 20/20 & 1.00 \\
 & Low light, $\iota{=}0.9$ & 8/10 & 10/10 & 0.24 \\
 & Motion blur, $\iota{=}0.6$ & 2/20 & \textbf{16/20} & $8.3{\times}10^{-6}$ \\
 & Motion blur, $\iota{=}0.8$$^{\dagger}$ & 0/10 & \textbf{19/20} & $3.7{\times}10^{-7}$ \\
 & Dim room + noise, $\iota{=}0.5$ & 0/20 & \textbf{15/20} & $3.9{\times}10^{-7}$ \\
\bottomrule
\end{tabular}
\end{table}

\begin{figure*}[t]
  \centering
  \includegraphics[width=0.95\textwidth,trim=3 1 3 7,clip]{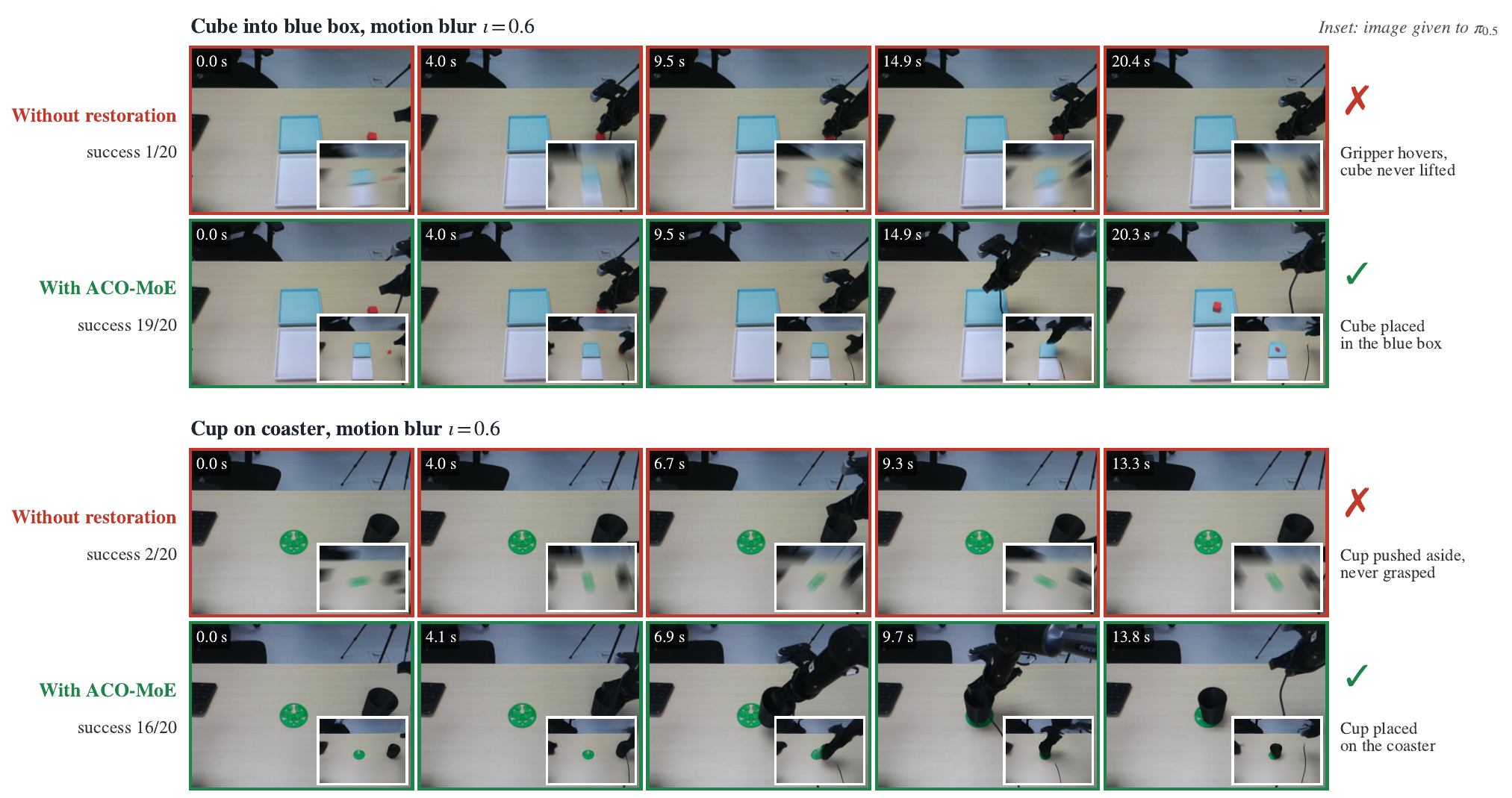}
  \caption{First trial of each first-batch run under motion blur at $\iota{=}0.6$, with snapshots at the same policy steps. The large frames show the global camera before the perturbation, and the insets show the image given to $\pi_{0.5}$, which is blurred without restoration and restored by ACO-MoE otherwise. Success counts cover both batches. Without restoration the gripper reaches the object but does not grasp it, whereas with ACO-MoE the policy completes both tasks.}
  \label{fig:real_rollouts}
\end{figure*}

\begin{figure*}[t]
  \centering
  \includegraphics[width=\textwidth,trim=3 1 3 7,clip]{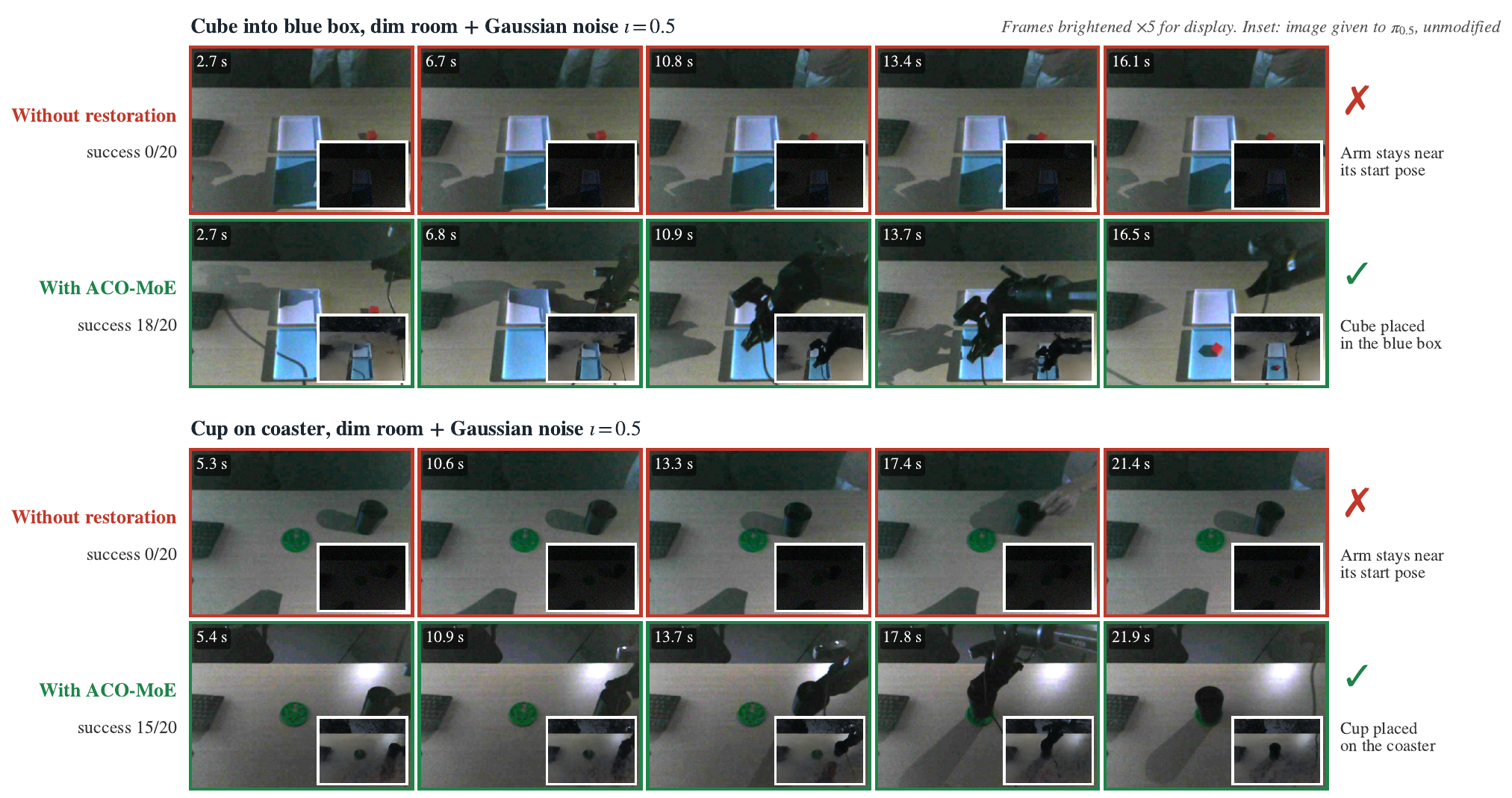}
  \caption{First trial of each second-batch run in the physically dim room with Gaussian noise at $\iota{=}0.5$, with snapshots at the same policy steps. The large frames show the global camera before the injected noise, brightened five times for display, and the insets show the unmodified image given to $\pi_{0.5}$. Without restoration the arm does not enter the global view in these frames, whereas with ACO-MoE the policy completes both tasks.}
  \label{fig:real_dim_rollouts}
\end{figure*}

\paragraph{Results.} Motion blur and the dim room are the decisive settings in Table~\ref{tab:real_closed_loop}. Under motion blur at $\iota{=}0.6$, the policy without restoration succeeds in 1 of 20 cube trials and 2 of 20 cup trials, and in the trials of Fig.~\ref{fig:real_rollouts} the gripper reaches the object without grasping it. With ACO-MoE, success reaches 19/20 and 16/20 ($p{=}2.9{\times}10^{-9}$ and $p{=}8.3{\times}10^{-6}$), or 35/40 against 3/40 over both tasks ($p{=}1.3{\times}10^{-13}$, two-sided). In the dim room, success rises from 0/20 to 18/20 on the cube task and from 0/20 to 15/20 on the cup task ($p{=}1.7{\times}10^{-9}$ and $p{=}3.9{\times}10^{-7}$), or 33/40 against 0/40 over both tasks ($p{=}1.2{\times}10^{-15}$, two-sided), and Fig.~\ref{fig:real_dim_rollouts} shows one trial of each. For settings run in both batches, every two-sided Fisher comparison between the batches gives $p\ge0.47$. Clean input stays at 20/20 on both tasks with and without restoration, which places the success rate with restoration above 86\% at one-sided 95\% confidence. Low light at $\iota{=}0.9$ and Markov switching already reach 10/10 without restoration on the cube task and stay at 10/10 with it, and low-light success on the cup task rises from 8/10 to 10/10 with $p{=}0.24$.

\paragraph{Failure modes without restoration.} The logged joint states show how the unrestored policy fails. Under motion blur at $\iota{=}0.6$, the first three joints sweep ranges of 38$^\circ$ to 84$^\circ$ and the gripper opens and closes 4 to 13 times per run, which indicates repeated grasp attempts. Under motion blur at $\iota{=}0.8$, the arm still moves, but the gripper never opens beyond 1\,mm, so the policy never attempts a grasp. In the dim room, the arm hardly approaches the objects. On the cup task, the first three joints stay within 7.5$^\circ$, and on the cube task only the second joint moves substantially, by up to 36$^\circ$, while the other two stay within 11$^\circ$.

\paragraph{Routing and latency.} The logged routing decisions follow the perturbation. Under the synthetic perturbations, 11687 of 11688 camera-level decisions of the full-frame restorer select one expert per perturbation type, namely expert 0 for clean input, expert 2 for motion blur, expert 3 for low light, and expert 6 for Gaussian noise, including every switch of the Markov run. The router selects the low-light expert on all 5828 decisions from the dim room with added Gaussian noise. On the robot computer, the restorer takes 2.4 to 2.6\,ms per step, and the whole adapter, which resizes, perturbs, restores, and pads both images, takes 5.5 to 9.7\,ms, well within the 50\,ms action period.

\paragraph{Task-weighted version.} The task-weighted restorer in main-paper Fig.~\ref{fig:real_mask} uses $C{=}16$, $N_e{=}7$, and the same 30k-update training, but weights the $\ell_1$ term by $1{+}4\,m_t^{\mathrm{ps}}$. SAM~3.1 masks the clean training frames using prompts proposed and revised by the VLM after mask review. The global camera uses \emph{robot arm} and \emph{robot gripper}, the wrist camera uses \emph{black object}, which covers the jaws and the grasped cup, and both cameras add the objects named in the instruction, namely \emph{red cube} with \emph{blue box} or \emph{white box}, and \emph{black cup} with \emph{green coaster}. Each mask is dilated by 3\,px. In closed loop, it handles the cup task with motion blur at $\iota{=}0.8$ on both live streams and succeeds in 19 of 20 trials over two batches, against 0 of 10 without restoration. All 2486 camera-level decisions select one expert. Two-camera restoration takes 2.0\,ms per robot step. Over the 249 logged steps, the PSNR against the camera frame before the software blur rises inside the dilated task mask from 15.1 to 19.9\,dB on the global camera and from 18.9 to 26.9\,dB on the wrist camera, and in the background from 21.7 to 25.3\,dB and from 26.3 to 35.7\,dB. Offline replay reproduces the logged outputs within one intensity level. On the same global inputs, an equal-size unweighted restorer reaches 18.8\,dB inside the mask and 24.6\,dB outside it, so the weighting adds 1.1\,dB in the task region and 0.7\,dB in the background. The 95\% bootstrap intervals of these gains are $[0.8,1.1]$ and $[0.4,0.6]$\,dB in the first batch and $[1.2,1.5]$ and $[0.8,0.9]$\,dB in the second. On the wrist camera, whose mask covers 26\% to 27\% of the frame, both gains are about 2\,dB. The larger full-frame model gains another 0.8\,dB inside the mask and differs by at most 0.1\,dB in the background. The equal-size unweighted restorer succeeds in 8 of 10 first-batch trials.

\section{VLM Prompts for SAM~3.1 Pseudo-Labels}
\label{app:vlm_prompts}

\paragraph{Models.}
The VLM is the frozen, closed-source Claude Opus~5.5, with model identifier \texttt{claude-opus-5-5}, queried in one independent conversation per task. SAM~3.1~\cite{carion2026sam3} uses the image model of the released \texttt{sam3.1} checkpoint. It segments each clean frame at the input size of the controller, which is $84{\times}84$ on DMControl and $168{\times}168$ on RoboSuite. Every prompt of the task is applied at a confidence threshold of 0.3, and the pseudo mask is the union of all returned instances.

\ifdefined\ifwithsupp\newpage\fi
\paragraph{Procedure.}
The procedure has three stages, and no stage shows the VLM a reference mask, an overlap score, or any other quality measure.
\begin{enumerate}
\item \emph{Proposal.} The VLM receives the instruction below, the task ID, and three clean frames, and returns a list of prompts.
\item \emph{Review.} For up to three rounds, the VLM sees the SAM~3.1 result on the same three frames, drawn as the frame, the mask overlay, and the masked frame, and either accepts the prompts or revises them.
\item \emph{Audit.} SAM~3.1 labels every frame with the current prompts. For up to two rounds, the VLM inspects the three frames whose mask area is most unusual, namely the two smallest masks and the largest one, and either accepts or revises. The answer of the last round is final.
\end{enumerate}
The prompts in Table~\ref{tab:vlm_prompts} are copied verbatim from the final answers and are then frozen for labeling.

\onecolumn
\noindent
\begin{minipage}[t]{\dimexpr(\linewidth-\columnsep)/2\relax}
\vspace{0pt}
\paragraph{Instruction.}
{\normalsize\raggedright\setlength{\baselineskip}{13pt}\setlength{\parskip}{4pt}\noindent
You are preparing text prompts for an open-vocabulary image segmenter (SAM~3.1). The segmenter takes short noun phrases and returns a mask for every image region that matches a phrase. The attached images are frames from the simulated control task ``\{task\_id\}''.

\noindent We want a task-oriented foreground mask: it keeps the controlled agent and every object whose state matters for solving this task, for example a target, a marker, or a manipulated object. Everything else, such as the floor, sky, walls, grid lines, shadows, and reflections, will be removed. Look at the frames and return the shortest list of noun phrases that together cover this foreground.

\noindent Rules: (1)~each phrase has at most three words and describes visible appearance, \ie color and shape, since the segmenter matches what it sees, not what the object is called; (2)~use a single phrase for parts that look the same, and add another phrase only for a task-relevant part with a clearly different appearance; (3)~do not name the floor, sky, walls, grid, shadows, or reflections.

\noindent Return only JSON with the fields \texttt{task}, \texttt{foreground}, and \texttt{prompts}.
\par}
\end{minipage}\hfill
\begin{minipage}[t]{\dimexpr(\linewidth-\columnsep)/2\relax}
\vspace{0pt}
\begingroup
\centering
\captionof{table}{Final VLM prompts used by SAM~3.1~\cite{carion2026sam3} to label every clean training frame in DMControl and RoboSuite. Both benchmarks use the same labeling procedure. The pseudo mask of a frame is the union over all prompts of its task.}
\label{tab:vlm_prompts}
\renewcommand{\arraystretch}{1.45}
\setlength{\tabcolsep}{4pt}
\begin{tabularx}{\linewidth}{@{}lX@{}}
\toprule
\textbf{Task} & \textbf{Final prompts} \\
\midrule
CP-S & orange block, orange pole \\
F-S & orange finger, orange stick, orange cylinder, orange sausage, orange object, blue dot, blue ball \\
F-TH & orange shapes, orange arm, orange cylinders, red sphere, blue circle \\
H-S & orange object \\
H-H & orange stick figure, orange robot \\
C-R & orange stick figure \\
W-W & orange rods, orange stick figure, orange cylinder, orange sticks, orange robot \\
W-R & orange stick figure, orange cylinders, orange tubes \\
RoboSuite \texttt{Door} & white robot arm, thin black panel, brown horizontal bar \\
RoboSuite \texttt{Lift} & white robot arm, small red cube \\
\bottomrule
\end{tabularx}
\endgroup
\end{minipage}

\section{Additional VDCS Baseline}
\label{app:drq_vdcs}
\begin{table}[H]
\centering
\caption{DrQ-v2 returns stay below 70 on all eight VDCS tasks under Markov switching. Values report mean and standard deviation.}
\label{tab:drq_vdcs}
\small
\begin{tabularx}{\linewidth}{@{}l*{9}{>{\centering\arraybackslash}X}@{}}
\toprule
\textbf{Task} & CP-S & F-S & F-TH & H-S & H-H & C-R & W-W & W-R & Average \\
\midrule
\textbf{DrQ-v2}~\cite{yarats2021mastering} & \tabstat{49}{79} & \tabstat{2}{1} & \tabstat{68}{38} & \tabstat{4}{3} & \tabstat{3}{2} & \tabstat{6}{3} & \tabstat{45}{24} & \tabstat{25}{13} & \tabstat{25.3}{25.2} \\
\bottomrule
\end{tabularx}
\end{table}

\fi

\end{document}